\documentclass{article}

\newcommand{\affmark}[1]{%
  \begingroup
  \renewcommand{\thefootnote}{\arabic{footnote}}%
  \footnotemark[#1]%
  \endgroup
}
\newcommand{\afftext}[2]{%
  \begingroup
  \renewcommand{\thefootnote}{\arabic{footnote}}%
  \footnotetext[#1]{#2}%
  \endgroup
}

\usepackage[english]{babel}
\usepackage[a4paper,top=2cm,bottom=2cm,left=1.5cm,right=1.5cm,marginparwidth=1.75cm]{geometry}

\usepackage[round]{natbib}
\usepackage{amsmath,amsthm,amssymb,bm,color}
\IfFileExists{bbm.sty}{\usepackage{bbm}}{}
\usepackage{mathrsfs}
\usepackage{enumitem}
\usepackage{graphicx}
\usepackage{booktabs}
\usepackage{float}
\usepackage{placeins}
\usepackage{titlesec}
\titleformat{\section}
  {\normalfont\large\bfseries}
  {\thesection}{1em}{\MakeUppercase}
\usepackage{hyperref}
\usepackage[capitalise]{cleveref}
\hypersetup{
  colorlinks=true,
  citecolor=blue,
  linkcolor=magenta,
  pdftitle={Minimax-optimal Scaling Laws in Ridge Regression via Transfer Learning},
}

\newcommand{\pt}{\texttt{sr}}
\newcommand{\ft}{\texttt{tr}}

\renewcommand{\P}{\mathbb{P}}
\newcommand{\E}{\mathbb{E}}

\newcommand{\cN}{\mathcal{N}}

\newcommand{\R}{\mathbb{R}}

\newcommand{\lamonmu}{\tau}
\newcommand{\lamonmupowerft}{\gamma_{\tau_{\ft}}}
\newcommand{\lamonmupowerpt}{\gamma_{\tau_{\pt}}}

\newcommand{\crosstracescale}{\Gamma_T}
\newcommand{\shrinkagebias}{B_{\rm sh}}

\newcommand{\gammatrans}{\Gamma_{\rm TL}}
\newcommand{\eps}{\varepsilon}

\newcommand{\diag}{\text{diag}}

\def\bzero{{\boldsymbol 0}}

\DeclareMathOperator*{\argmin}{arg\,min}

\newtheorem{theorem}{Theorem}
\newtheorem*{theorem*}{Theorem}
\newtheorem{lemma}{Lemma}

\newtheorem{assumption}{Assumption}

\newtheorem{proposition}{Proposition}

\newtheorem{corollary}{Corollary}

\theoremstyle{definition}

\newtheoremstyle{myremark} 
    {\topsep}                    
    {\topsep}                    
    {\it}                        
    {}                           
    {\bf}                        
    {.}                          
    {.5em}                       
    {}  

\theoremstyle{myremark}
\newtheorem{remark}{Remark}[section]

\DeclareSymbolFont{rsfs}{U}{rsfs}{m}{n}
\DeclareSymbolFontAlphabet{\mathscrsfs}{rsfs}

\def\bA{{\boldsymbol A}}
\def\bB{{\boldsymbol B}}
\def\bC{{\boldsymbol C}}

\def\bH{{\boldsymbol H}}

\def\bI{{\boldsymbol I}}

\def\bP{{\boldsymbol P}}
\def\bQ{{\boldsymbol Q}}

\def\bS{{\boldsymbol S}}

\def\bV{{\boldsymbol V}}
\def\bW{{\boldsymbol W}}
\def\bX{{\boldsymbol X}}

\def\be{{\boldsymbol e}}

\def\bx{{\boldsymbol x}}
\def\by{{\boldsymbol y}}
\def\bz{{\boldsymbol z}}

\def\bbeta{{\boldsymbol \beta}}

\def\bSigma{{\boldsymbol \Sigma}}
\def\bTheta{{\boldsymbol \Theta}}

\def\cR{\mathcal{R}}

\def\Tr{{\rm Tr}}

\def\cV{{\mathcal V}}

\def\cV{{\mathcal V}}

\def\naturals{{\mathbb N}}

\def\naturals{{\mathbb N}}

\def\cD{{\mathcal D}}

\def\bA{{\boldsymbol A}}

\def\bTheta{{\boldsymbol \Theta}}

\def\cV{{\mathcal V}}
\def\bP{{\boldsymbol P}}
\def\diag{{\rm diag}}
\def\bS{{\boldsymbol S}}

\def\bC{{\boldsymbol C}}

\def\cB{\mathcal{B}}

\def\cB{\mathcal{B}}

\def\op{{\rm op}}

\newcommand{\figorbox}[3]{%
  \IfFileExists{figs/#1}%
    {\includegraphics[width=#2]{figs/#1}}%
    {\fbox{\parbox[c][#3][c]{0.92\linewidth}{\centering
      \textcolor{red}{\texttt{\detokenize{figs/#1}} not found}}}}%
}

\title{{\fontsize{22}{27}\selectfont Minimax-optimal Scaling Laws in Ridge Regression \\via Transfer Learning}}
\author{
Yalda Shabanzadeh \affmark{1}
\\
\and
Diyuan Wu \affmark{1}
\\
\and 
Marco Mondelli \affmark{1}
}
\date{}

\begin{document}
\maketitle
\begingroup
\renewcommand{\thefootnote}{}
\footnotetext{Correspondence to: \texttt{diyuan.wu@ist.ac.at}}
\endgroup
\afftext{1}{Institute of Science and Technology Austria (ISTA). Emails: \texttt{\{yalda.shabanzadeh, diyuan.wu, marco.mondelli\}@ist.ac.at}}

\begin{abstract}
Pretrained models are often reused across downstream tasks, motivating the question of whether this transfer learning procedure can provide a substantial performance improvement, namely, in terms of the scaling law. We address this question in the prototypical setting of high-dimensional regression, where \emph{(i)} a pretrained \emph{source} model is trained via ridge regression, and \emph{(ii)} a fine-tuned \emph{target} model is regularized via the $L_2$ distance from the source estimator. First, we establish a deterministic equivalent for the risk of this two-step procedure under both model and covariance shift. Next, by specializing the deterministic equivalent to a shared regression function and commuting power-law covariance spectra, we identify a regime where the resulting scaling law achieves the minimax-optimal exponent. Notably, in this regime, the optimal source regularization depends on source data only, indicating its universality across all targets. Furthermore, we show that the same optimal rate is achieved by performing the fine-tuning step via one-pass stochastic gradient descent. We support our theoretical analysis via experiments on both synthetic settings and language models.
\end{abstract}

\section{INTRODUCTION}
Empirical studies have demonstrated that reusing pretrained
representations improves downstream performance in computer vision
\citep{donahue2014decaf,yosinski2014transferable}
and language tasks
\citep{howard2018universal,devlin2019bert}. Methods such as L2-SP~\citep{li2018explicit},
adapters~\citep{houlsby2019parameter},
and LoRA~\citep{hu2022lora} adapt models through
regularization toward pretrained weights, the addition of task-specific
modules, and low-rank parameter updates, respectively.
Besides the choice of a specific fine-tuning method, downstream performance
also depends on the scale of the source and target model. More precisely, empirical scaling laws describe how transfer performance varies
with the scaling of model size and the amounts of pretraining and fine-tuning data
\citep{hernandez2021scaling,zhang2024scaling}.
Notably, better pretraining performance on the source task does not necessarily improve results downstream:
longer pretraining can improve source performance while worsening
downstream performance \citep{springer2025overtrained};
conversely, stronger pretraining weight decay can improve downstream
performance even when it increases the pretraining validation loss
\citep{han2026weight}.
Understanding these effects requires studying when pretraining
improves downstream learning, especially in the scaling law, and how the source model should
be trained to achieve these gains.

From a theoretical viewpoint, several works have studied regularization with respect to a source estimator, referred to as L2-SP in \citep{li2018explicit}.
In this regard, \citet{dar2024common} analyze minimum-norm pretraining with isotropic source covariates followed by ridge fine-tuning
under model shift, identifying conditions under which transfer
outperforms ridge regression only on target data.
\citet{kim2025transfer} study minimum-norm pretraining followed
by the target interpolator closest to the source estimator,
deriving non-asymptotic risk bounds under covariate
and model shift in benign-overfitting regimes.
In functional linear regression, \citet{lin2024functional}
establish minimax rates for transfer learning where both the pretraining and fine-tuning objective is $L_2$-regularized
under bounded model shift and covariance compatibility
conditions, quantifying its benefits over learning only on target data.
In nonparametric regression under model shift,
\citet{lin2024smoothness} develop a method for
selecting regularization parameters in pretraining and
fine-tuning 
that achieves minimax rates. However, these works do not explicitly characterize when transfer learning
achieves a strictly better rate than learning from a single dataset (only the source, or only the target). \cite{wu2022power} establish sharp risk bounds
under covariate shift when both pretraining and fine-tuning use one-pass stochastic gradient descent (SGD), and
construct specific examples where this two-stage procedure improves sample complexity. This motivates us to derive a systematic characterization of how covariance spectra and sample sizes determine when transfer improves the decay of the risk. 
A related question concerns how the source model should
be trained to achieve these gains.
For ridge regression,
\citet{hedges2026source} shows that, under proportional asymptotics with isotropic data, the optimal source regularization is
independent of the target sample size but it still depends on task
alignment.
This leaves open whether a single source estimator, trained
without target information, can remain optimal across
a family of target distributions in regimes where transfer
improves the scaling exponent over either dataset alone.

In this paper, we study these questions through a tractable
model of high-dimensional linear regression.
The source and target tasks may differ in their regression
coefficients, covariate distributions, and noise levels,
with sample sizes growing at different polynomial rates.
We obtain the source estimator from ridge regression
and fine-tune it on target data using a squared
$\ell_2$ penalty centered at the source estimator.
We summarize our main contributions as follows:

\vspace{-.5em}

\begin{itemize}[itemsep=1pt,leftmargin=*]

\item \textbf{Scaling laws for transfer learning.} We derive a deterministic equivalent for the target excess risk of ridge regression, together with a non-asymptotic  approximation
    error (Theorem~\ref{thm:total-risk}).
    Specializing to power-law covariance spectra and a common regressor, we obtain
    scaling laws for transfer learning
    (Theorem~\ref{thm:ridge-transfer-scaling}). Comparing them to the joint minimax rate which uses both datasets
    (Theorem~\ref{thm:minimax}), we show that the optimally tuned deterministic equivalent achieves the joint minimax rate whenever the latter is
    strictly faster than the minimax rates of either
    dataset alone (Corollary \ref{cor:minimax-optimal-ridge-transfer}).

\item \textbf{Universality of optimal source training.} Our analysis unveils that the optimal source regularization depends
    only on source information, allowing a single source estimator
    to achieve these gains across various target
    distributions without target-dependent tuning of the source. Furthermore, the same source estimator remains optimal when fine-tuning is performed via one-pass SGD (Lemma \ref{lem:sgd-scaling-main}), indicating universality across fine-tuning methods.

    \item \textbf{Numerical evidence.} We evaluate the accuracy of the predictions based on deterministic equivalents in synthetic data and PCA features from MNIST and CIFAR-10.
We further conduct language-modeling experiments, 
giving evidence that similar transfer
phenomena may arise beyond linear models.

\end{itemize}


\section{Related Work}
\label{sec:related-work}

\textbf{Empirical works.} 
Early empirical studies examined how pretrained representations
are adapted and how transfer benefits vary across downstream
tasks \citep{yosinski2014transferable,kornblith2019better}.
Common approaches include using frozen features, fine-tuning
all pretrained parameters, and adapting a small number of
trainable parameters.
L2-SP regularizes the target model toward pretrained weights
\citep{li2018explicit}, while adapters and LoRA introduce
task-specific modules and low-rank weight updates
\citep{houlsby2019parameter,hu2022lora}.
The adaptation procedure also affects out-of-distribution
performance: \citet{kumar2022finetuning} show that full
fine-tuning can distort pretrained features, and that
training a linear head on frozen features before fine-tuning
can improve robustness to distribution shift.
Empirical scaling laws further quantify the benefits of
pretraining in terms of an equivalent amount of additional
target data \citep{hernandez2021scaling}, and describe how
downstream performance depends on model size, pretraining
and fine-tuning data sizes, and the adaptation method
\citep{zhang2024scaling}.
\citet{isik2025scaling} show that source-target distribution
alignment and fine-tuning data size affect downstream scaling:
increasing pretraining data can reduce downstream cross-entropy
without consistently improving translation quality.
These findings motivate theoretical analyses of how task
alignments, datasets size, and fine-tuning procedures
affect the benefit of transfer learning.

\textbf{Theoretical works.} Beyond the two-stage regression analyses discussed in the
introduction, theoretical studies have considered other forms of transfer learning. One related line of work is continual linear regression, which can be considered as a generalization of transfer learning to multiple sequential tasks. \citet{zhao2024statistical} identify matrix-valued regularizers
that match the estimation rates of learning jointly from all tasks, under shared regression coefficients and
commuting covariances.
\cite{karpel2026optimal} characterize the generalization error and the optimal  regularization
strength for isotropic data in the proportional asymptotics regime.
For classification, classical domain-adaptation bounds relate
target error to source error, distribution discrepancy,
and the joint error of the best shared classifier
\citep{bendavid2010theory}.
\citet{hanneke2019value} characterize minimax rates,
quantifying how transfer difficulty and label noise determine
the relative value of source and target samples.
\citet{tripuraneni2020taskdiversity,collins2024multitask}
study models in which tasks share an unknown feature
representation and 
show that learning this representation from sufficiently
diverse training tasks can reduce the sample complexity
of learning a new task.
In high-dimensional regression,
\citet{li2022transfer} establish minimax rates under sparse
or small source--target coefficient differences and develop
methods that adapt to unknown informative sources.
\citet{he2024transfusion} develop transfer
estimators that improve estimation under heterogeneous
covariances when source--target coefficient differences
are sufficiently small. Complementing these results, we study how covariance spectra and relative sample sizes determine the scaling benefits of transfer, identifying regimes in which ridge pretraining followed by ridge or SGD fine-tuning achieves minimax scaling and outperforms both target-only and source-only learning. 

\section{Problem SETUP}
\label{sec:setup}

Given $n\in \mathbb N$, we use the shorthand $[n]:=\{1, \ldots, n\}$. We consider two independent training datasets, corresponding to a source
task ($\pt$) and a target task ($\ft$).  For the source and target tasks, respectively, we observe $n_{\pt}$ and $n_{\ft}$ independent samples
$(\bx_{\pt}^{(j)},y_{\pt}^{(j)})_{j\in[n_{\pt}]}$ and
$(\bx_{\ft}^{(j)},y_{\ft}^{(j)})_{j\in[n_{\ft}]}$, where
$\bx_{\pt}^{(j)}$ and $\bx_{\ft}^{(j)} \in \R^d$ are drawn from distributions
$P_{\pt}$ and $P_{\ft}$ such that
$\E_{P_{\pt}}[\bx]=\E_{P_{\ft}}[\bx]=\mathbf 0$, 
$\E_{P_{\pt}}[\bx\bx^\top]=\bSigma_{\pt}$, and 
$\E_{P_{\ft}}[\bx\bx^\top]=\bSigma_{\ft}$.
We assume the following concentration property on the designs.

\begin{assumption}
\label{asm:distr}
There exist constants $C,c>0$ such that, for every $z \in \{\pt, \ft\}$, every PSD matrix $\bA \in \R^{d \times d}$, and every $t>0$, a sample $\bx\sim P_z$ satisfies
\begin{equation}
    \P(|\bx^\top\bA\bx-\Tr(\bSigma_z\bA)|> t\|\bSigma_z^{1/2}\bA\bSigma_z^{1/2}\|_F)\leq Ce^{-ct}.
\end{equation}
\end{assumption}

Assumption \ref{asm:distr} is standard in the literature \citep{misiakiewicz2024non,wu2026improved} and covers e.g.\ independent sub-Gaussian entries.
We denote $\bX_{z}=[\bx_{z}^{(1)},\ldots,\bx_{z}^{(n_{z})}]^\top\in\R^{n_{z}\times d}$ the data matrix and $\by_{z}=(y_{z}^{(j)})_{j\in[n_{z}]}\in\R^{n_{z}}$ the label for source ($z = \pt$) and target ($z= \ft$). The labels are generated by
\begin{equation}
\begin{aligned}
\by_{z} &= \bX_{z}\bbeta_{z}+\boldsymbol{\eps}_{z},
&\qquad
\boldsymbol{\eps}_{z}
&\sim \cN(\bzero,\sigma_{z}^2\bI_{n_z}).
\end{aligned}
\label{eq:two-tasks}
\end{equation}

The noise vectors are independent of the data and of each other.
In particular, the source and target tasks may differ in their parameters,
covariances, and noise levels.
For an estimator $\hat\bbeta$, we consider the excess prediction risk
on the target task,
\begin{equation}
\begin{aligned}
  \cR(\hat\bbeta)
  &:=\E_{\bx, \eps}\left[
    \big(\bx^\top\hat\bbeta-\bx^\top\bbeta_{\ft}\big)^2
  \right]
  \\&=\E_{\eps}\left[
    (\hat\bbeta-\bbeta_{\ft})^\top
    \bSigma_{\ft}
    (\hat\bbeta-\bbeta_{\ft})
  \right],
\label{eq:risk-def}
\end{aligned}
\end{equation}
where the expectation also averages
over label noise. 

In this paper, we study transfer learning via the following two-step procedure. We first obtain a source estimator via ridge regression on the source dataset,
\begin{equation}
  \hat\bbeta_{\pt}
  =\argmin_{\bbeta}
  \left\{
    \|\bX_{\pt}\bbeta-\by_{\pt}\|_2^2
    +\lambda_{\pt}\|\bbeta\|_2^2
  \right\},
\label{eq:pretrain}
\end{equation}
and then obtain the target model by fitting a regression model on the target dataset with ridge penalty
centered at the source estimator,
\begin{equation}
  \hat\bbeta_{\ft}
  =\argmin_{\bbeta}
  \left\{
    \|\bX_{\ft}\bbeta-\by_{\ft}\|_2^2
    +\lambda_{\ft}\|\bbeta-\hat\bbeta_{\pt}\|_2^2
  \right\}.
\label{eq:finetune}
\end{equation}
The above two-step procedure has been used by  \cite{kim2025transfer} to study properties of transfer learning. The two estimators admit the closed forms
\begin{equation}
\begin{aligned}
\hat\bbeta_{\pt}
&=
(\bX_{\pt}^\top\bX_{\pt}+\lambda_{\pt}\bI_d)^{-1}
\bX_{\pt}^\top\by_{\pt},\\
\hat\bbeta_{\ft}&=
(\bX_{\ft}^\top\bX_{\ft}+\lambda_{\ft}\bI_d)^{-1}
\left(
\bX_{\ft}^\top\by_{\ft}
+\lambda_{\ft}\hat\bbeta_{\pt}
\right).
\end{aligned}
\label{eq:closed-form}
\end{equation}

\begin{remark}[Infinite-dimensional formulation]
For convenience, we use matrix notation throughout.
In the infinite-dimensional setting, covariance matrices
$\bSigma_z$ are interpreted as positive semidefinite trace-class
operators on a separable Hilbert space, and transposes should be understood as adjoints. Our arguments apply to both settings under the stated assumptions.
\end{remark}



\section{Deterministic equivalent for target risk}
\label{sec:deteq}

We start by establishing a deterministic equivalent for the risk \eqref{eq:risk-def} of $\hat\bbeta_{\ft}$. First, let the effective regularization parameters $\lamonmu_\pt,\lamonmu_\ft>0$ be the unique solution of 
\begin{equation}
\begin{aligned}
  n_z-\frac{\lambda_z}{\lamonmu_z}
  =\Tr\left[\bSigma_z(\bSigma_z+\lamonmu_z\bI_d)^{-1}\right],\quad z\in\{\pt,\ft\}.
\label{eq:effective-ridge-main}
\end{aligned}
\end{equation}
Furthermore, for $z\in\{\pt,\ft\}$, let us also define \begin{equation}
\begin{aligned}
\bW_z:=\bSigma_z(\bSigma_z+\lamonmu_z\bI_d)^{-2},
\end{aligned}
\label{eq:W-def-main}
\end{equation}
and the two scalar prefactors
\begin{equation}
\begin{aligned}
  c_{\ft}&:=\frac{\lamonmu_{\ft}^2\,n_{\ft}}
  {n_{\ft}-\Tr[\bSigma_{\ft}^2(\bSigma_{\ft}+\lamonmu_{\ft}\bI_d)^{-2}]},
  \\
  c_{\pt}&:=\frac{1}
  {n_{\pt}-\Tr[\bSigma_{\pt}^2(\bSigma_{\pt}+\lamonmu_{\pt}\bI_d)^{-2}]}.
\end{aligned}
\label{eq:c-def-main}
\end{equation}



\begin{theorem}[Deterministic equivalent for target risk in transfer learning]
\label{thm:total-risk}
Let Assumption~\ref{asm:distr} hold, and suppose there exists a constant $K>0$ such that for $z \in \{\pt, \ft\},$ \begin{equation*}
    \begin{split}
        \lambda_{z}\rho_{z}
\geq
\|\bSigma_{z}\|_{\mathrm{op}}n_{z}^{-K},
\qquad
\rho_{z}^{5/2}\log^{3/2}(n_{z})
\leq K\sqrt{n_{z}},
    \end{split}
\end{equation*} where $\rho_z$ is defined in \eqref{eq:rho_lambda}. Define
$$
\varepsilon_\pt
:=
\frac{\rho_\pt^6\log^{5/2}(n_\pt)}{\sqrt{n_\pt}},
\qquad
\varepsilon_\ft
:=
\frac{\rho_\ft^6\log^{5/2}(n_\ft)}{\sqrt{n_\ft}},
$$
and
\begin{equation}
\begin{aligned}
\mathcal E_{\pt}
&:=
c_{\ft}\lamonmu_{\pt}^2
\left(
\bbeta_{\pt}^\top
(\bSigma_{\pt}+\lamonmu_{\pt}\bI_d)^{-1}
\bbeta_{\pt}
\right)\\&\qquad \cdot
\left\|
\bW_{\ft}^{1/2}
(\bSigma_{\pt}+\lamonmu_{\pt}\bI_d)^{-1}
\bW_{\ft}^{1/2}
\right\|_{\mathrm{op}}.
\label{eq:bias-error-scale-theorem}
\end{aligned}
\end{equation}
Then for any $D>0,$ with probability at least
$1-(n_{\pt}^{-D}+n_{\ft}^{-D})$, 
\begin{equation}
|\cR-\overline{\cR}|
\leq
C_{D,K}
\left[
(\varepsilon_{\pt}+\varepsilon_{\ft})
\overline{\cR}
+
\varepsilon_{\pt}\mathcal E_{\pt}
\right],
\label{eq:main-decomposition}
\end{equation}
where
\begin{align} 
 \overline{\cR} &:= \underbrace{\overline{\cB}_1+\overline{\cB}_2}_{\textstyle\overline{\cB}}
  +\overline{\cV}_{\pt}+\overline{\cV}_{\ft},\\ 
  \overline{\cB}_1
&:=
 c_{\ft}
 \left\|
 \bW_{\ft}^{1/2}
 (\bSigma_{\pt}
  (\bSigma_{\pt}+\lamonmu_{\pt}\bI_d)^{-1}
  \bbeta_{\pt}-\bbeta_{\ft})
 \right\|_2^2,
  \label{eq:main-B1}\\[1mm]
  \overline{\cB}_2
  &:=c_{\ft}c_{\pt}\lamonmu_{\pt}^2
  \Tr\left(\bW_{\ft}\bW_{\pt}\right)
  \bbeta_{\pt}^\top\bW_{\pt}\bbeta_{\pt},
  \label{eq:main-B2}\\[1mm]
  \overline{\cV}_{\pt}
  &:=\sigma_{\pt}^2\,c_{\ft}c_{\pt}
  \Tr\left(\bW_{\ft}\bW_{\pt}\right),
  \label{eq:main-var-pt}\\[1mm]
  \overline{\cV}_{\ft}
  &:=\sigma_{\ft}^2\,
  \frac{\Tr\left[\bSigma_{\ft}^2
  (\bSigma_{\ft}+\lamonmu_{\ft}\bI_d)^{-2}\right]}
  {n_{\ft}-\Tr\left[\bSigma_{\ft}^2
  (\bSigma_{\ft}+\lamonmu_{\ft}\bI_d)^{-2}\right]}.
  \label{eq:main-var-ft}
\end{align}
\end{theorem}


The proof of Theorem \ref{thm:total-risk} is deferred to Appendix \ref{app:deteq}, and it uses the deterministic equivalents for various functionals proved in  \citep[Theorem 6]{wu2026improved}. We note that, while \cite{misiakiewicz2024non} establish multiplicative error bounds, the error bound in \eqref{eq:main-decomposition} is  $O(\max(\eps_\pt, \eps_\ft))$ and it is additive, as in \citep{wu2026improved, wu2026data}, due to asymmetric terms in the risk expression. We expect that a multiplicative error bound can be established
in certain structured settings, such as the power-law model
considered in Section~\ref{sec:scaling}.

\section{SCALING LAWS for Transfer Learning}
\label{sec:scaling}

We start by describing in Section \ref{subsec:power} the power-law model under which scaling laws are derived. Then, in Section \ref{sec:minimax-sub}, we recall minimax results from \cite{wu2026data} using both source and target data; in Section \ref{sec:scalinglaw-sub}, we  establish scaling laws for the ridge estimator in \eqref{eq:finetune} obtained via transfer learning; finally, in Section \ref{sec:sgd-sub} we give scaling laws for one-pass SGD initialized from the source estimator in \eqref{eq:pretrain}.

\subsection{Power-law model}\label{subsec:power}

Throughout, we consider $d=\infty$ and assume that \emph{(i)} 
$\bSigma_{\pt}, \bSigma_{\ft}$ commute, so that w.l.o.g.\ we can take them to be diagonal, \emph{(ii)} the source--target tasks share the same regression weights $\bbeta_\pt = \bbeta_\ft = \bbeta_*,$ and \emph{(iii)} the noise variances $\sigma^2_{\eps_\pt}, \sigma^2_{\eps_\pt} = \Theta(1).$

More precisely, for \emph{(i)}, we consider power-law spectra
\begin{equation}
  (\bSigma_{\pt})_{k,k}\asymp k^{-a},
  \quad
  (\bSigma_{\ft})_{k,k}\asymp k^{-b},
  \quad a,b>1.
\label{eq:power-law-spectra}
\end{equation}

For \emph{(ii)}, we study two types of signals $\bbeta_*.$ For the minimax model in Section \ref{sec:minimax-sub}, we assume that the signal $\bbeta_*$ belongs to the augmented parameter
class $\bTheta(R)$, defined for fixed $r_{\pt},r_{\ft}>0$
and a radius $R>0$ independent of $n_{\ft},n_{\pt}$ by
\begin{equation}
  \bTheta(R):=
  \left\{\bbeta:
 \|\bSigma_{\pt}^{1/2-r_{\pt}}\bbeta\|_2\le R,\,\,
   \|\bSigma_{\ft}^{1/2-r_{\ft}}\bbeta\|_2\le R
  \right\}.
  \label{eq:minimax-class}
\end{equation}

The conditions in \eqref{eq:minimax-class} correspond to a source condition \citep{caponnetto2007optimal}. Defining 
\begin{equation}
  \delta:=b-a,
  \,\,
  s:=(br_{\ft})\vee
  \left(ar_{\pt}+\frac{\delta}{2}\right),
  \,\,
  \gamma_{\rm c}
  :=1-\frac{\delta}{1+2s},
\label{eq:minimax-parameters}
\end{equation}  \eqref{eq:minimax-class} is, up to multiplicative constants in the radius, equivalent to the ellipsoid $\{\bbeta: \sum_k k^{2s -b} (\bbeta)_k^2 \leq R^2\}$, see \citep[Appendix C]{wu2026data}.

When studying the scaling law of the ridge estimator (Section \ref{sec:scalinglaw-sub}) and one-pass SGD (Section \ref{sec:sgd-sub}), we consider $\bbeta_*$ satisfying 
\begin{equation}
  (\bSigma_{\pt}^{1/2}\bbeta_*)_k^2
  \asymp k^{-1-2ar_{\pt}},
  \;
  (\bSigma_{\ft}^{1/2}\bbeta_*)_k^2
  \asymp k^{-1-2br_{\ft}},
  \label{eq:source-conditions}
\end{equation}
in the regime where $br_{\ft}=ar_{\pt}+\delta/2.$ This also implies that $ s = br_{\ft}=ar_{\pt}+\delta/2$ by definition. Equivalently, we can write
$(\bbeta_*)_k^2\asymp k^{-1 +b-2s}.$ We further restrict to the regime where $2s > b$ so that $\bbeta_*$ is square-integrable.

Note that $\bbeta_*$ lies outside of $\bTheta(R).$ 
Nevertheless, the scaling law on $\bbeta_*$ reflects the worst case in the ellipsoid. To see this, for any $\bbeta \in \bTheta(R),$ letting
$\bz=\bSigma_{\ft}^{1/2}\bbeta$ and $\bz_{*,k}=\bSigma_{\ft}^{1/2}\bbeta_*,$  we have
\[
 \sum_{k=M}^{2M-1}\bz_{*,k}^2
 \asymp M^{-2s}
 \asymp
 \sup_{\bbeta\in\bTheta(R)}
 \sum_{k=M}^{2M-1}\bz_k^2.
\]
Thus, on each spectral band $M\le k<2M$, $\bz_*$ has the same
energy order as the largest allowed by the ellipsoid.
The bias term in the test error sums the contribution of each band and is governed
by the dominant band, up to logarithmic factors.
Thus, the worst-case bias is obtained by picking a signal in the ellipsoid that concentrates its energy on
the worst band, yielding the same bias exponent as $\bbeta_*.$
As the variance does not depend on $\bbeta_*,$ this gives
the same polynomial scaling exponent for the test error.

We set $n:=n_{\ft}$ and take
\begin{equation}
  n_{\pt}\asymp n^{\gamma_n},
  \qquad
  \lambda_{\pt}\asymp n^{-\gamma_{\lambda_{\pt}}},
  \qquad
  \lambda_{\ft}\asymp n^{-\gamma_{\lambda_{\ft}}}.
\label{eq:standing-sample-ridge-scaling}
\end{equation}
We consider $\gamma_n > 0$ and ridge penalties satisfying
$\gamma_{\lambda_{\pt}}>-\gamma_n$ and
$\gamma_{\lambda_{\ft}}>-1$, as well as the target-only
($\lambda_{\pt}=\infty$) and source-only
($\lambda_{\ft}=\infty$) cases.
This condition on ridge penalties simply ensures that  $n_z^{-1} \lambda_z = O(1)$ for $z \in \{\pt,\ft\}$, see also the remark on stable regularization in \cite{wu2026improved}. 

\subsection{Minimax rates}
\label{sec:minimax-sub}

 For a given dataset or pair of datasets
$\cD$, define the minimax risk as
\begin{equation}
 \cR^*(n_{\ft},n_{\pt}) =  \inf_{\hat\bbeta=\hat\bbeta(\cD)}
  \sup_{\bbeta\in\bTheta(R)}
  \E_{\bbeta}\|\bSigma_{\ft}^{1/2}(\hat\bbeta-\bbeta)\|_2^2,
\label{eq:minimax-risk-definition}
\end{equation} where $\E_{\beta}$ denotes the expectation w.r.t the data distribution generated by the signal $\beta.$ 

\begin{theorem}[Minimax scaling laws, Theorem 3 in \cite{wu2026data}] 
\label{thm:minimax} 
Under the linear model in \eqref{eq:two-tasks}, the power-law spectra
in \eqref{eq:power-law-spectra}, and the parameter class
\(\bTheta(R)\) in \eqref{eq:minimax-class},
the minimax risks for individual datasets satisfy
\begin{equation}
  \cR^*(n,0)\asymp n^{-\Gamma_{\ft}},
  \quad
  \cR^*(0,n_{\pt})
  \asymp
  \begin{cases}
    n^{-\Gamma_{\pt}}, & \delta<1,\\[1mm]
    n^{-\gamma_n}\log n, & \delta=1,\\[1mm]
    n^{-\gamma_n}, & \delta>1,
  \end{cases}
\label{eq:minimax-single-dataset-rates}
\end{equation}
and the minimax risk using both datasets satisfies
\begin{equation}
  \cR^*(n,n_{\pt})
  \asymp
  \begin{cases}
    \cR^*(n,0),
    & \gamma_n\leq\gamma_{\rm c},\\[1mm]
    \cR^*(0,n_{\pt}),
    & \delta\leq1,\ \gamma_n>\gamma_{\rm c},\\[1mm]
    \cR^*(0,n_{\pt}),
    & \delta>1,\ \gamma_n\geq1,\\[1mm]
    n^{-(\delta-1+\gamma_n)/\delta},
    & \delta>1,\ \gamma_{\rm c}<\gamma_n<1.
  \end{cases}
\label{eq:minimax-rates}
\end{equation}
The implicit constants may depend on constants $R,s, a,b$  and noise levels $\sigma^2_\pt, \sigma^2_\ft.$
\end{theorem}

Theorem~\ref{thm:minimax} shows that, when
$\delta>1$ and $\gamma_c<\gamma_n<1$, using both datasets
yields a strictly faster minimax rate than using either dataset
alone. Furthermore, \cite{wu2026data} also show that this improved minimax rate can be achieved by a ridge estimator. However, both the minimax rate of Theorem~\ref{thm:minimax} and the ridge estimator of \cite{wu2026data} require access to both source and target data at the same time. In contrast, in the transfer learning model we are considering, the target estimator does not have access to the source data per se, but rather only to the pretrained estimator. We will next show that, despite this restriction, minimax rates remain achievable both by ridge regression and one-pass SGD. 

\subsection{Scaling laws for ridge estimator}
\label{sec:scalinglaw-sub}
We now study the estimator in \eqref{eq:finetune}, and derive the scaling law for the deterministic equivalent of its risk.
To express the resulting rates, we introduce the effective cutoff
exponents
\begin{equation}
  u:=\frac{(\gamma_{\lambda_{\pt}}+\gamma_n)
            \wedge(a\gamma_n)}{a},
  \quad
  v:=\frac{(1+\gamma_{\lambda_{\ft}})\wedge b}{b},
\label{eq:u-v-def}
\end{equation}
and 
\begin{equation}
  \Gamma_{\ft}
  :=\frac{2s}{1+2s},
  \qquad
  \Gamma_{\pt}
  :=\frac{2s\gamma_n}{1+2s-\delta}.
\label{eq:minimax-exponents}
\end{equation}

\begin{theorem}[Scaling laws for ridge estimator]
\label{thm:ridge-transfer-scaling}
Under the assumptions of Theorem \ref{thm:total-risk}, for any fixed $u,v,$  define the following quantities 
\begin{equation}
\begin{aligned}
  \Gamma_{\cV_{\ft}}&=1-v,
  \qquad
  &\Gamma_{\cB_1}&=2bv+\Gamma_{\rm sh},
  \\
  \Gamma_{\cV_{\pt}}&=\gamma_n+2bv+\Gamma_T,
  \qquad
  &\Gamma_{\cB_2}&=\Gamma_{\cV_{\pt}}+2au+\Gamma_S,
\end{aligned}
\label{eq:four-exponents}
\end{equation} where $\Gamma_S=-(a+b-2br_{\ft})_+u,$ \begin{equation*} 
\begin{split}
    \Gamma_T=
  \begin{cases}
    -(a{+}b{+}1)u-(\delta{+}1)_+(v-u), & u\le v,\\[1mm]
    -(a{+}b{+}1)v-(1{-}\delta)_+(u-v), & v\le u,
  \end{cases}
\end{split}
\end{equation*}
and 
\begin{equation*}
  \Gamma_{\rm sh}\hspace{-.2em}=\hspace{-.2em}
  \begin{cases}
\hspace{-.1em}      2au\hspace{-.2em}-\hspace{-.2em}2bu\big(1{+}\tfrac ab{-}r_{\ft}\big)_+\hspace{-.5em}-\hspace{-.2em}2b(v{-}u)(1{-}r_{\ft})_+, \,\, u\hspace{-.1em}\le\hspace{-.1em} v,\\
 \hspace{-.2em}     2au\hspace{-.2em}-\hspace{-.2em}2bv\big(1{+}\tfrac ab{-}r_{\ft}\big)_+\hspace{-.5em}-\hspace{-.2em}2b(u{-}v)\big(\tfrac ab{-}r_{\ft}\big)_+,\,\,
     v\hspace{-.1em}\le\hspace{-.1em} u.
  \end{cases}
\end{equation*}
Then, the deterministic equivalents of the bias and variance terms satisfy \begin{equation*}
      \begin{split}
          &\overline{\cB}_1 =  \widetilde\Theta\!\left(n^{-\Gamma_{\cB_1}}\right), \quad \overline{\cB}_2 = \widetilde\Theta\!\left(n^{-\Gamma_{\cB_2}}\right),\\
          & \overline{\cV}_\pt = \widetilde\Theta\!\left(n^{-\Gamma_{\cV_\pt}}\right), \quad \overline{\cV}_\ft = \widetilde\Theta\!\left(n^{-\Gamma_{\cV_\ft}}\right),
      \end{split}
  \end{equation*} and therefore $\overline{\cR}=\widetilde\Theta\!\left(n^{-\gammatrans}\right)$, with
\begin{equation}
  \gammatrans=\min \{\Gamma_{\cV_{\ft}}, \Gamma_{\cB_1}, \Gamma_{\cV_{\pt}}, \Gamma_{\cB_2}\}.
\label{eq:Gamma-trans-main}
\end{equation}
\end{theorem}

Theorem  \ref{thm:ridge-transfer-scaling} establishes the scaling law of the deterministic equivalent $\overline{\cR}$, up to a poly-logarithmic factor (included in the notation $\widetilde\Theta\!\left(\cdot\right)$), and its proof is deferred to Appendix \ref{app:scaling-ridge}. 

We next choose ridge parameters for which $\overline{\cR}$ achieves the same exponent in the scaling law as the minimax risk using both datasets given by \eqref{eq:minimax-rates}.

\begin{corollary}[Minimax-optimality of ridge deterministic equivalent]
\label{cor:minimax-optimal-ridge-transfer}
Under the assumptions of \Cref{thm:total-risk}, in the regime $\delta>1, \gamma_c < \gamma_n <1$, by
picking \begin{equation}
 \lambda_{\pt}\asymp n_{\pt}^{1-a},
 \qquad
 \lambda_{\ft}
 \asymp n^{1-b(1-\gamma_n)/\delta},
 \label{eq:optimal-lambda-choice}
\end{equation}
or equivalently,
$u=\gamma_n$ and $v=(1-\gamma_n)/\delta$, we have that $\Gamma_{\rm TL} = \frac{\delta - 1 +\gamma_n}{\delta}$.  
\end{corollary}

Corollary~\ref{cor:minimax-optimal-ridge-transfer} shows that
the deterministic equivalent of the risk for the ridge estimator obtained via transfer learning matches the minimax rate using jointly source and target data. We note that the regularization is  chosen by optimizing the deterministic equivalent
$\overline{\cR}$ from Theorem~\ref{thm:total-risk}. For this choice, 
we expect
$|\cR-\overline{\cR}|=o(\overline{\cR})$
with high probability, which would imply the same scaling law
for the actual test risk $\cR$.
This expectation is motivated by the analogous approximation guarantee for mixed ridge regression in
\citet[Lemma~1]{wu2026data}
and is supported by our numerical results
(\Cref{fig:de-regimes}). This gives evidence that the two-step procedure composed of pretraining and fine-tuning is optimal and improves on using either dataset in isolation.
 We defer the full characterization of optimal rate and optimal regularization to Corollary \ref{cor:minimax-optimal-ridge-transfer-full} in Appendix \ref{app:optimal-ridge-rate}, from which Corollary~\ref{cor:minimax-optimal-ridge-transfer} directly follows.

We further note that the optimal source regularization depends only on source sample size and spectral decay, which makes the same tuning of the source estimator optimal across different targets.
This supports pretraining a single source model and reusing
it across downstream tasks, with only the target
regularization chosen according to the specific fine-tuning data.

We finally remark that Corollary \ref{cor:minimax-optimal-ridge-transfer} shows the existence of a target-independent optimal source estimator rather than its uniqueness. In particular, the optimal ridge pair in \eqref{eq:optimal-lambda-choice} may not be unique, and it is possible that there exist other optimal pairs in which $\lambda_\pt$ depends on the target.

\subsection{Scaling laws for one-pass SGD}

\label{sec:sgd-sub}

We now replace ridge fine-tuning with one-pass SGD.
We first obtain the source estimator
$\widehat\bbeta_{\pt}$ by ridge regression
\eqref{eq:pretrain}, and initialize
$\widehat\bbeta_0=\widehat\bbeta_{\pt}$.
Using each target sample once, SGD returns
$\widehat\bbeta_{\ft}=\widehat\bbeta_n$, where
\begin{equation}
 \widehat\bbeta_t
 =
 \widehat\bbeta_{t-1}
 -
 \eta_t\bx_t
 (\bx_t^\top\widehat\bbeta_{t-1}-y_t),
 \qquad t=1,\ldots,n.
 \label{eq:sgd-rd-iteration}
\end{equation}

For the target covariates, we additionally impose the
following fourth-moment condition.

\begin{assumption}
\label{ass:sgd-fourth-moment-main}
There exists a constant $C\ge1$ such that,
for every bounded positive semidefinite matrix  $\bA$,
\begin{equation}
 \E\big[
 (\bx^\top\bA\bx)\bx\bx^\top
 \big] \leq
 C\Tr(\bSigma_{\ft}\bA)\bSigma_{\ft},
 \quad \bx\sim P_{\ft}.
 \label{eq:sgd-fourth-moment-main}
\end{equation}
\end{assumption}

\begin{lemma}[Scaling laws for one-pass SGD]
\label{lem:sgd-scaling-main}
In the setting of Theorem \ref{thm:ridge-transfer-scaling} and under Assumption \ref{ass:sgd-fourth-moment-main}, consider the regime $\delta>1, \gamma_c < \gamma_n < 1$.
Let
\begin{equation}
 \lambda_{\pt}=\kappa n_{\pt}^{1-a},
 \qquad
 \eta_t=\eta_0\,2^{-\lfloor(t-1)/N\rfloor},
 \label{eq:sgd-optimal-tuning-main}
\end{equation}
where $\kappa>0$ is independent of the target task,
and
\begin{equation}
 N=\left\lfloor\frac{n}{\log_2 n}\right\rfloor,
 \qquad
 \eta_0=\frac1N
 \left(\frac{N}{n_{\pt}}\right)^{b/\delta}.
 \label{eq:sgd-optimal-step-main}
\end{equation}
Then, the last SGD iterate satisfies $$ \E\big[
 \|\widehat\bbeta_{\ft}-\bbeta_*\|_{\bSigma_{\ft}}^2
 \big]
 =
 \widetilde\Theta(n^{-\Gamma_{\rm TL}}),\quad \Gamma_{\rm TL}
 =
 \frac{\delta-1+\gamma_n}{\delta}.$$
\end{lemma}

Lemma~\ref{lem:sgd-scaling-main} shows that one-pass SGD
matches the exponent of optimally regularized ridge, which is in turn minimax optimal. We note that both methods use the same rate-optimal source regularization, which does not depend on the target task.
Thus, the universality of the pretrained source estimator
across target tasks extends to both ridge
and SGD fine-tuning.

The proof is given in Appendix \ref{sec:gd-sgd-transfer}.
For the upper bound, we condition on the source estimator
and apply the SGD bound of \citet{wu2022power}.
We then evaluate the deterministic equivalent w.r.t the source estimator.
At the selected ridge
$\lambda_{\pt}\asymp n_{\pt}^{1-a}$,
we show that the related deterministic equivalents have multiplicative approximation error. For the lower bound, we compare the second-moment recursion of the target-noise component with its population counterpart. 

\begin{remark}[Saturation effect of ridge and comparison with one-pass SGD]
In linear regression under power-law spectra, ridge regression
suffers from saturation, which is instead avoided by one-pass SGD 
\citep{wu2025risk}.
In our transfer setting, ridge and SGD fine-tuning attain
the same polynomial scaling exponent in the regime
$\delta>1$ and $\gamma_{\rm c}<\gamma_n<1$ since the test error is dominated by
the variance terms, while the signal-dependent
bias is of lower order.
Consequently, ridge saturation does not affect the leading
scaling exponent in this regime.

\end{remark}

\begin{figure*}[t]
    \centering
    \includegraphics[width=\textwidth]{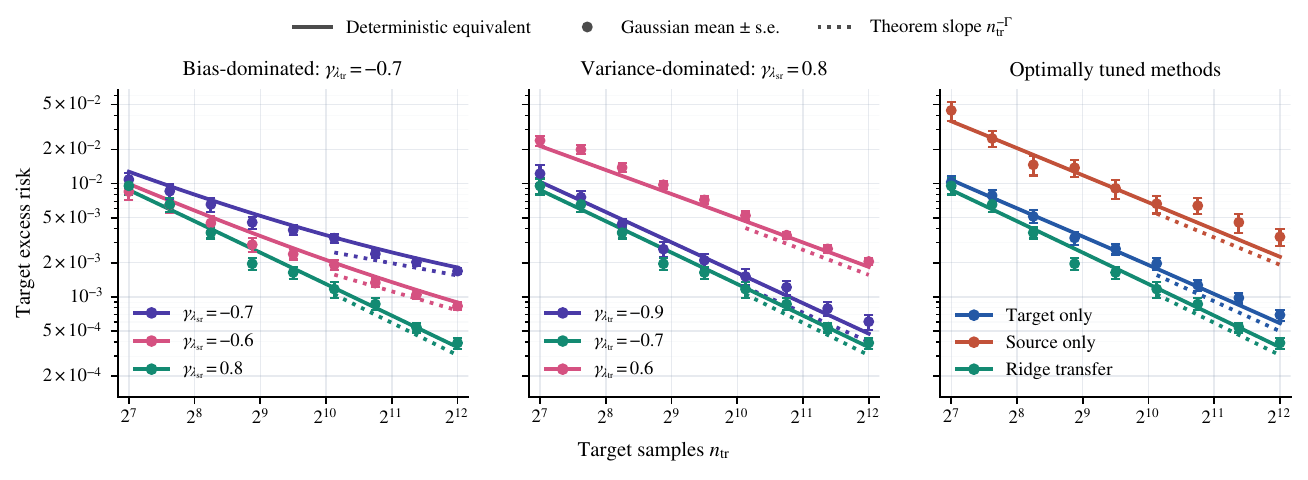}
    \caption{\textbf{Accuracy of theoretical predictions across various regimes} for $a=2$, $b=6$, $r_{\ft}=0.6$, $d=50{,}000$. Solid lines show the deterministic equivalent, markers show Gaussian simulations (mean $\pm$ s.e., $50$ seeds), and dotted segments show the predicted slopes. \textbf{Left:} Bias is the bottleneck in the risk, and the optimal scaling law is reached by choosing the correct $\lambda_{\pt}$. \textbf{Middle:} Variance is the bottleneck in the risk, and the optimal $\lambda_{\ft}$ balances $\overline{\cV}_{\ft}=\overline{\cV}_{\pt}$. \textbf{Right:} Comparison between optimally tuned target-only, source-only, and transfer learning risks.}
    \label{fig:de-regimes}
\vspace{-.5em}
\end{figure*}

\begin{figure}[t]
    \centering
    \includegraphics[width=0.5\textwidth]{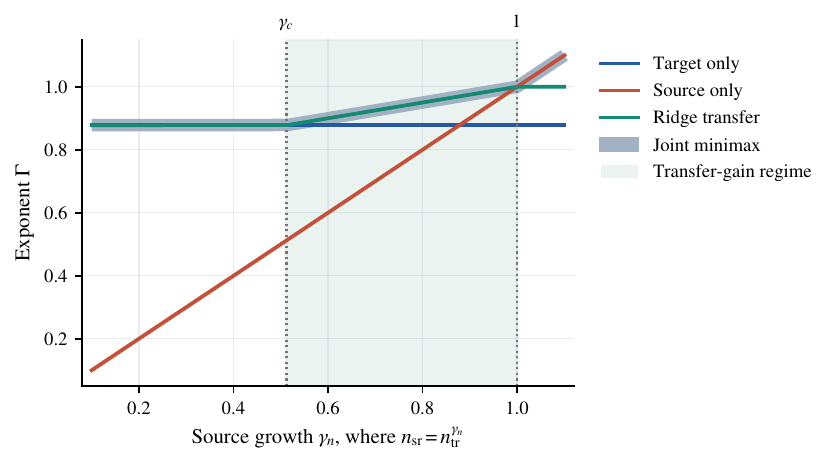}
    \caption{\textbf{Scaling law exponents of optimally regularized ridge and minimax rates, as a function of pretraining data size.} Solid lines show the exponents from \Cref{thm:ridge-transfer-scaling}. The region $\gamma_c<\gamma_n<1$ is where transfer learning attains the minimax rate obtained by employing both datasets, and it improves over estimators using either target-only or source-only data. The gray band is the minimax rate, i.e., the best achievable worst-case exponent any estimator can reach using both source and target data.}
    \label{fig:synthetic-rate-transition}
\end{figure}

\begin{figure*}[t]
    \centering
    \includegraphics[width=\linewidth]{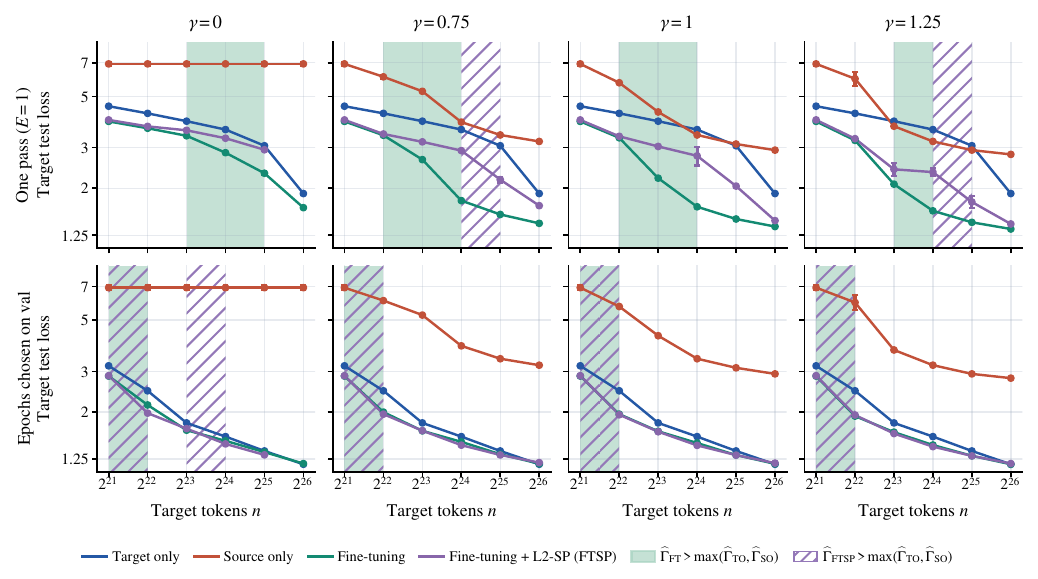}
    \caption{\textbf{Transfer learning in language models.} Target test loss as a function of target sample size. Different columns correspond to different values of the source sample size ($\gamma=0$ in the left-most plot and then $\gamma=0.75$, $\gamma=1$, and $\gamma=1.25$). 
    Different rows correspond to different number of passes on the data (single pass in the first row, and multiple passes with number of epochs chosen from validation in the second row). 
    Access to more source data produces a significant advantage of FT under short target training (one-pass), while the advantage becomes weaker with additional target optimization (multi-pass). We denote the local exponent by  $\widehat\Gamma=-\frac{\ln L(2n) - \ln L(n)}{\ln 2}$, see \eqref{eq:gamma-hat-llm} for details. Green shading shows the ranges where $\hat\Gamma_{FT}>\max(\widehat\Gamma_{TO},\widehat\Gamma_{SO})$, and purple hatching marks where $\widehat\Gamma_{FTSP}>\max(\widehat\Gamma_{TO},\widehat\Gamma_{SO})$. FTSP refers to our implementation of L2-SP \citep{li2018explicit}, see Appendix \ref{app:llm-exp} for details.
    }
    \label{fig:gpt2-regimes-all}
\end{figure*}

\section{Numerical Results}

We first showcase our theoretical predictions for synthetic data, and then examine whether a similar behavior occurs in the training of language models.

\paragraph{Experiments on synthetic data.}
We consider the power-law model of \Cref{sec:scaling}, with $(\bSigma_{\pt})_{k,k}\asymp k^{-a}$ and $(\bSigma_{\ft})_{k,k}\asymp k^{-b}$. We use $a=2$, $b=6$, $r_{\ft}=0.6$, $d=50{,}000$, with $n_{\pt}=n_{\ft}^{\gamma_n}$. The ridge parameters follow the scaling  in \eqref{eq:standing-sample-ridge-scaling}.

\Cref{fig:de-regimes} shows bias-limited, variance-limited, and optimally tuned ridge regression, and how the regularization choice affects the scaling law of the target risk. The empirical results are closely tracked by deterministic-equivalent predictions, and they recover the corresponding theoretical decay rates. In the left panel, the bias is limiting the scaling law, and choosing the correct source regularization $\lambda_{\pt}$ mitigates the issue. In the middle panel, the variance is limiting the scaling law, and choosing the correct target regularization $\lambda_{\ft}$ mitigates the issue. In the right panel, all estimators (target-only, source-only, and ridge transfer) are optimally tuned, and transfer learning improves the scaling law as compared to using a single dataset.  



While \Cref{fig:de-regimes} fixes $\gamma_n$ and shows the effect of regularization on performance, \Cref{fig:synthetic-rate-transition} varies $\gamma_n$ and examines how the optimal exponent $\Gamma$ of ridge regression changes as the size of the source dataset grows. The exponents are obtained from  \Cref{thm:ridge-transfer-scaling} and highlight the regime $\gamma_{\rm c}<\gamma_n<1$ in which combining source and target data improves the decay rate of the risk. 
Additionally, in Appendix \ref{app:de-empirical}, we show that the empirical risk is 
tracked by deterministic equivalents for PCA features from MNIST \citep{LeCun1998GradientbasedLA} and CIFAR-10 \citep{Krizhevsky09learningml}.

\paragraph{Experiments on language models.} We train 
GPT-2 small \citep{radford2019language}, using FineWeb-Edu \citep{penedo2024fineweb} as the source distribution and fine-tune it using Python \citep{codeparrot_clean} as the target distribution. We design these experiments to see whether fine-tuning (FT) from the source trained on $n_{\pt}\asymp n^{\gamma_n}$ tokens to a target with $n$ tokens can make the target loss decrease faster in $n$ than both target-only (TO) and source-only (SO) training. We also check whether this loss decrease depends on $\gamma_n$.


\Cref{fig:gpt2-regimes-all} shows that the improvement in the transfer learning performance depends on both the size of source data and the optimization method used for target data. With a fixed source budget, FT achieves lower loss at small target budgets, but TO gradually closes the gap. When the source budget grows with the target budget, FT can decrease faster over the early scaling range under short target training; this rate advantage becomes smaller after additional target optimization. Thus, beyond the linear setting, we observe a similar dependence of transfer scaling on how the source budget grows with the target budget. Overall,  we observe a higher FT rate under target training than under TO and SO separately, but this advantage depends on target optimization. Further details on the experimental setup are in Appendix \ref{app:llm-exp}.

\section{Conclusions}
\label{sec:discussion}

We study transfer learning in high-dimensional regression, where the source model is obtained by ridge regression and the target model is regularized toward the source estimator. Under rather general conditions, we establish a deterministic equivalent for the risk of the estimator implementing ridge regression with target data. Specializing to power-law covariance spectra and a shared regression function, we derive explicit scaling laws and compare them with the minimax rate of using both datasets. Notably, whenever combining the datasets yields a strictly better minimax rate than using either dataset alone, transfer learning via ridge regression attains this optimal exponent. Moreover, fine-tuning via one-pass SGD with a suitable geometric learning rate schedule attains the same exponent in the scaling law. Both ridge regression and SGD employ an optimal source estimator which does not depend on the target task and it is, therefore, universal across target distributions. Experiments on language models corroborate the generality of our findings beyond linear models.

{\bf Future directions.} A natural direction is to extend our analysis beyond regression
to classification, building on existing minimax theory \citep{cai2021transfer}.
In particular, it would be interesting to determine whether
the improved scaling laws and the universality of the source
estimator persist in this setting.
Another direction is to incorporate model size explicitly
through random feature models \citep{defilippis2024dimension}.
Allowing the number of random features to grow jointly with
the source and target sample sizes would enable us to
characterize how model capacity affects the benefits of transfer,
and to identify the model size required to retain the improved
scaling exponent.


\clearpage

\subsection*{AI use statement}

We used generative AI tools to polish the writing, implement straightforward parts of the code, fill in proof details after we had identified the key methods and arguments, check proofs for correctness, and improve the coverage of related work. We did not use generative AI tools to develop the paper's core ideas and define the problem setting. We take full responsibility for the final content of this work, including all text, claims, code, and other artifacts produced with the assistance of generative AI.


\bibliography{references}

@article{caponnetto2007optimal,
  title={Optimal rates for the regularized least-squares algorithm},
  author={Caponnetto, Andrea and De Vito, Ernesto},
  journal={Foundations of Computational Mathematics},
  volume={7},
  number={3},
  pages={331--368},
  year={2007},
  publisher={Springer},
  doi={10.1007/s10208-006-0196-8}
}

@inproceedings{
    wu2026improved,
    title={Improved Scaling Laws via Weak-to-Strong Generalization in Random Feature Ridge Regression},
    author={Diyuan Wu and Lehan Chen and Theodor Misiakiewicz and Marco Mondelli},
    booktitle={Forty-third International Conference on Machine Learning},
    year={2026},
    url={https://openreview.net/forum?id=yE3zbIOW3e}
}

@inproceedings{tripuraneni2020taskdiversity,
 author = {Tripuraneni, Nilesh and Jordan, Michael I. and Jin, Chi},
 booktitle = {Advances in Neural Information Processing Systems},
 editor = {H. Larochelle and M. Ranzato and R. Hadsell and M.F. Balcan and H. Lin},
 pages = {7852--7862},
 publisher = {Curran Associates, Inc.},
 title = {On the Theory of Transfer Learning: The Importance of Task Diversity},
 url = {https://proceedings.neurips.cc/paper_files/paper/2020/file/59587bffec1c7846f3e34230141556ae-Paper.pdf},
 volume = {33},
 year = {2020}
}

@InProceedings{collins2024multitask,
  title = 	 {Provable Multi-Task Representation Learning by Two-Layer {R}e{LU} Neural Networks},
  author =       {Collins, Liam and Hassani, Hamed and Soltanolkotabi, Mahdi and Mokhtari, Aryan and Shakkottai, Sanjay},
  booktitle = 	 {Proceedings of the 41st International Conference on Machine Learning},
  pages = 	 {9292--9345},
  year = 	 {2024},
  editor = 	 {Salakhutdinov, Ruslan and Kolter, Zico and Heller, Katherine and Weller, Adrian and Oliver, Nuria and Scarlett, Jonathan and Berkenkamp, Felix},
  volume = 	 {235},
  series = 	 {Proceedings of Machine Learning Research},
  month = 	 {21--27 Jul},
  publisher =    {PMLR},
  url = 	 {https://proceedings.mlr.press/v235/collins24a.html},
}

@article{hernandez2021scaling,
  title={Scaling laws for transfer},
  author={Hernandez, Danny and Kaplan, Jared and Henighan, Tom and McCandlish, Sam},
  journal={arXiv preprint arXiv:2102.01293},
  year={2021}
}

@inproceedings{
zhang2024scaling,
title={When Scaling Meets {LLM} Finetuning: The Effect of Data, Model and Finetuning Method},
author={Biao Zhang and Zhongtao Liu and Colin Cherry and Orhan Firat},
booktitle={The Twelfth International Conference on Learning Representations},
year={2024},
url={https://openreview.net/forum?id=5HCnKDeTws}
}

@InProceedings{han2026weight,
  title =          {Weight Decay Improves Language Model Plasticity},
  author =       {Han, Tessa and Bordt, Sebastian and Zhang, Hanlin and Kakade, Sham M.},
  booktitle =          {Proceedings of the 43rd International Conference on Machine Learning},
  pages =          {39571--39602},
  year =          {2026},
  editor =          {Zhang, Tong and Dudik, Miroslav and Jaggi, Martin and Agarwal, Alekh and Li, Sharon and Schuurmans, Dale and Zhu, Jerry and Berkenkamp, Felix and Dong, Hanze and Bietti, Alberto},
  volume =          {306},
  series =          {Proceedings of Machine Learning Research},
  month =          {06--11 Jul},
  publisher =    {PMLR},
  url =          {https://proceedings.mlr.press/v306/han26b.html},
}

@article{misiakiewicz2024non,
  title={A non-asymptotic theory of kernel ridge regression: deterministic equivalents, test error, and gcv estimator},
  author={Misiakiewicz, Theodor and Saeed, Basil},
  journal={arXiv preprint arXiv:2403.08938},
  year={2024}
}

@inproceedings{
defilippis2024dimension,
title={Dimension-free deterministic equivalents and scaling laws for random feature regression},
author={Leonardo Defilippis and Bruno Loureiro and Theodor Misiakiewicz},
booktitle={The Thirty-eighth Annual Conference on Neural Information Processing Systems},
year={2024},
url={https://openreview.net/forum?id=FBLJIfW64D}
}

@InProceedings{wu2025risk,
title = {Risk Comparisons in Linear Regression: Implicit Regularization Dominates Explicit Regularization (Extended Abstract)},
author = {Wu, Jingfeng and Bartlett, Peter L. and Kakade, Sham M. and Lee, Jason D. and Yu, Bin},
booktitle = {Proceedings of Thirty Ninth Conference on Learning Theory},
pages = {6849--6851},
year = {2026},
editor = {Hanneke, Steve and Lattimore, Tor},
volume = {336},
series = {Proceedings of Machine Learning Research},
month = {29 Jun--03 Jul},
publisher = {PMLR},
url = {https://proceedings.mlr.press/v336/wu26a.html},
}

@article{wu2022power,
  title={The power and limitation of pretraining-finetuning for linear regression under covariate shift},
  author={Wu, Jingfeng and Zou, Difan and Braverman, Vladimir and Gu, Quanquan and Kakade, Sham},
  journal={Advances in Neural Information Processing Systems},
  volume={35},
  pages={33041--33053},
  year={2022}
}

@article{li2022transfer,
  title={Transfer learning for high-dimensional linear regression: Prediction, estimation and minimax optimality},
  author={Li, Sai and Cai, T Tony and Li, Hongzhe},
  journal={Journal of the Royal Statistical Society Series B: Statistical Methodology},
  volume={84},
  number={1},
  pages={149--173},
  year={2022},
  publisher={Oxford University Press}
}

@InProceedings{he2024transfusion,
  title = 	 {TransFusion: Covariate-Shift Robust Transfer Learning for High-Dimensional Regression},
  author =       {He, Zelin and Sun, Ying and Li, Runze},
  booktitle = 	 {Proceedings of The 27th International Conference on Artificial Intelligence and Statistics},
  pages = 	 {703--711},
  year = 	 {2024},
  editor = 	 {Dasgupta, Sanjoy and Mandt, Stephan and Li, Yingzhen},
  volume = 	 {238},
  series = 	 {Proceedings of Machine Learning Research},
  month = 	 {02--04 May},
  publisher =    {PMLR},
  url = 	 {https://proceedings.mlr.press/v238/he24a.html},
}

@inproceedings{
    kim2025transfer,
    title={Transfer Learning for Benign Overfitting in High-Dimensional Linear Regression},
    author={Yeichan Kim and Ilmun Kim and Seyoung Park},
    booktitle={The Thirty-ninth Annual Conference on Neural Information Processing Systems},
    year={2025},
    url={https://openreview.net/forum?id=arXNS7T90z}
}

@article{
hedges2026source,
title={Source-Optimal Training is Transfer-Suboptimal},
author={C. Evans Hedges},
journal={Transactions on Machine Learning Research},
issn={2835-8856},
year={2026},
url={https://openreview.net/forum?id=CMlpokFXfA},
note={}
}

@InProceedings{li2018explicit,
  title = 	 {Explicit Inductive Bias for Transfer Learning with Convolutional Networks},
  author =       {Li, Xuhong and Grandvalet, Yves and Davoine, Franck},
  booktitle = 	 {Proceedings of the 35th International Conference on Machine Learning},
  pages = 	 {2825--2834},
  year = 	 {2018},
  editor = 	 {Dy, Jennifer and Krause, Andreas},
  volume = 	 {80},
  series = 	 {Proceedings of Machine Learning Research},
  month = 	 {10--15 Jul},
  publisher =    {PMLR},
  url = 	 {https://proceedings.mlr.press/v80/li18a.html},
}

@inproceedings{
lin2024functional,
title={On Hypothesis Transfer Learning of Functional Linear Models},
author={Haotian Lin and Matthew Reimherr},
booktitle={Forty-first International Conference on Machine Learning},
year={2024},
url={https://openreview.net/forum?id=mGsF8Q0fGZ}
}

@inproceedings{
lin2024smoothness,
title={Smoothness Adaptive Hypothesis Transfer Learning},
author={Haotian Lin and Matthew Reimherr},
booktitle={Forty-first International Conference on Machine Learning},
year={2024},
url={https://openreview.net/forum?id=v0VUsQI5yw}
}

@article{dar2024common,
  title={The common intuition to transfer learning can win or lose: Case studies for linear regression},
  author={Dar, Yehuda and LeJeune, Daniel and Baraniuk, Richard G},
  journal={SIAM Journal on Mathematics of Data Science},
  volume={6},
  number={2},
  pages={454--480},
  year={2024},
  publisher={SIAM}
}

@inproceedings{hanneke2019value,
 author = {Hanneke, Steve and Kpotufe, Samory},
 booktitle = {Advances in Neural Information Processing Systems},
 editor = {H. Wallach and H. Larochelle and A. Beygelzimer and F. d\textquotesingle Alch\'{e}-Buc and E. Fox and R. Garnett},
 pages = {},
 publisher = {Curran Associates, Inc.},
 title = {On the Value of Target Data in Transfer Learning},
 url = {https://proceedings.neurips.cc/paper_files/paper/2019/file/b91f4f4d36fa98a94ac5584af95594a0-Paper.pdf},
 volume = {32},
 year = {2019}
}

@InProceedings{karpel2026optimal,
  title = 	 {Optimal L2 Regularization in High-dimensional Continual Linear Regression},
  author =       {Karpel, Gilad and Moroshko, Edward and Levinstein, Ran and Meir, Ron and Soudry, Daniel and Evron, Itay},
  booktitle = 	 {Proceedings of The 37th International Conference on Algorithmic Learning Theory},
  pages = 	 {1--62},
  year = 	 {2026},
  editor = 	 {Telgarsky, Matus and Ullman, Jonathan},
  volume = 	 {313},
  series = 	 {Proceedings of Machine Learning Research},
  month = 	 {23--26 Feb},
  publisher =    {PMLR},
  url = 	 {https://proceedings.mlr.press/v313/karpel26a.html},

}

@techreport{radford2019language,
  title={Language Models are Unsupervised Multitask Learners},
  author={Radford, Alec and Wu, Jeffrey and Child, Rewon and Luan, David and Amodei, Dario and Sutskever, Ilya},
  institution={OpenAI},
  year={2019}
}

@misc{karpathy2022nanogpt,
author = {Andrej Karpathy},
title = {{nanoGPT}},
year = {2022},
publisher = {GitHub},
journal = {GitHub repository},
howpublished = {\url{https://github.com/karpathy/nanoGPT}}
}

@article{penedo2024fineweb,
  title={The {FineWeb} datasets: Decanting the web for the finest text data at scale},
  author={Penedo, Guilherme and Kydl{\'\i}{\v{c}}ek, Hynek and Ben Allal, Loubna and Lozhkov, Anton and Mitchell, Margaret and Raffel, Colin and Von Werra, Leandro and Wolf, Thomas},
  journal={Advances in Neural Information Processing Systems},
  volume={37},
  pages={30811--30849},
  year={2024}
}

@inproceedings{
loshchilov2018decoupled,
title={Decoupled Weight Decay Regularization},
author={Ilya Loshchilov and Frank Hutter},
booktitle={International Conference on Learning Representations},
year={2019},
url={https://openreview.net/forum?id=Bkg6RiCqY7},
}

@article{LeCun1998GradientbasedLA,
  author={Lecun, Y. and Bottou, L. and Bengio, Y. and Haffner, P.},
  journal={Proceedings of the IEEE}, 
  title={Gradient-based learning applied to document recognition}, 
  year={1998},
  volume={86},
  number={11},
  pages={2278-2324},
  doi={10.1109/5.726791}}

@techreport{Krizhevsky09learningml,
  author = {Alex Krizhevsky},
  title = {Learning multiple layers of features from tiny images},
  institution = {University of Toronto},
  year = {2009}
}

@inproceedings{yosinski2014transferable,
 author = {Yosinski, Jason and Clune, Jeff and Bengio, Yoshua and Lipson, Hod},
 booktitle = {Advances in Neural Information Processing Systems},
 editor = {Z. Ghahramani and M. Welling and C. Cortes and N. Lawrence and K. Weinberger},
 pages = {},
 publisher = {Curran Associates, Inc.},
 title = {How transferable are features in deep neural networks?},
 url = {https://proceedings.neurips.cc/paper_files/paper/2014/file/532a2f85b6977104bc93f8580abbb330-Paper.pdf},
 volume = {27},
 year = {2014}
}

@inproceedings{kornblith2019better,
  title={Do Better {ImageNet} Models Transfer Better?},
  author={Kornblith, Simon and Shlens, Jonathon and Le, Quoc V},
  booktitle={Proceedings of the IEEE/CVF conference on computer vision and pattern recognition},
  pages={2661--2671},
  year={2019}
}

@InProceedings{zhao2024statistical,
  title = 	 {A Statistical Theory of Regularization-Based Continual Learning},
  author =       {Zhao, Xuyang and Wang, Huiyuan and Huang, Weiran and Lin, Wei},
  booktitle = 	 {Proceedings of the 41st International Conference on Machine Learning},
  pages = 	 {61021--61039},
  year = 	 {2024},
  editor = 	 {Salakhutdinov, Ruslan and Kolter, Zico and Heller, Katherine and Weller, Adrian and Oliver, Nuria and Scarlett, Jonathan and Berkenkamp, Felix},
  volume = 	 {235},
  series = 	 {Proceedings of Machine Learning Research},
  month = 	 {21--27 Jul},
  publisher =    {PMLR},
  url = 	 {https://proceedings.mlr.press/v235/zhao24n.html},
}

@InProceedings{donahue2014decaf,
  title = 	 {DeCAF: A Deep Convolutional Activation Feature for Generic Visual Recognition},
  author = 	 {Donahue, Jeff and Jia, Yangqing and Vinyals, Oriol and Hoffman, Judy and Zhang, Ning and Tzeng, Eric and Darrell, Trevor},
  booktitle = 	 {Proceedings of the 31st International Conference on Machine Learning},
  pages = 	 {647--655},
  year = 	 {2014},
  editor = 	 {Xing, Eric P. and Jebara, Tony},
  volume = 	 {32},
  number =       {1},
  series = 	 {Proceedings of Machine Learning Research},
  address = 	 {Beijing, China},
  month = 	 {22--24 Jun},
  publisher =    {PMLR},
  url = 	 {https://proceedings.mlr.press/v32/donahue14.html},
}

@inproceedings{howard2018universal,
    title = "Universal Language Model Fine-tuning for Text Classification",
    author = "Howard, Jeremy  and
      Ruder, Sebastian",
    editor = "Gurevych, Iryna  and
      Miyao, Yusuke",
    booktitle = "Proceedings of the 56th Annual Meeting of the Association for Computational Linguistics (Volume 1: Long Papers)",
    month = jul,
    year = "2018",
    address = "Melbourne, Australia",
    publisher = "Association for Computational Linguistics",
    url = "https://aclanthology.org/P18-1031/",
    doi = "10.18653/v1/P18-1031",
    pages = "328--339",
}

@inproceedings{devlin2019bert,
    title = "{BERT}: Pre-training of Deep Bidirectional Transformers for Language Understanding",
    author = "Devlin, Jacob  and
      Chang, Ming-Wei  and
      Lee, Kenton  and
      Toutanova, Kristina",
    editor = "Burstein, Jill  and
      Doran, Christy  and
      Solorio, Thamar",
    booktitle = "Proceedings of the 2019 Conference of the North {A}merican Chapter of the Association for Computational Linguistics: Human Language Technologies, Volume 1 (Long and Short Papers)",
    month = jun,
    year = "2019",
    address = "Minneapolis, Minnesota",
    publisher = "Association for Computational Linguistics",
    url = "https://aclanthology.org/N19-1423/",
    doi = "10.18653/v1/N19-1423",
    pages = "4171--4186",
}

@InProceedings{houlsby2019parameter,
  title = 	 {Parameter-Efficient Transfer Learning for {NLP}},
  author =       {Houlsby, Neil and Giurgiu, Andrei and Jastrzebski, Stanislaw and Morrone, Bruna and De Laroussilhe, Quentin and Gesmundo, Andrea and Attariyan, Mona and Gelly, Sylvain},
  booktitle = 	 {Proceedings of the 36th International Conference on Machine Learning},
  pages = 	 {2790--2799},
  year = 	 {2019},
  editor = 	 {Chaudhuri, Kamalika and Salakhutdinov, Ruslan},
  volume = 	 {97},
  series = 	 {Proceedings of Machine Learning Research},
  month = 	 {09--15 Jun},
  publisher =    {PMLR},
  url = 	 {https://proceedings.mlr.press/v97/houlsby19a.html},
}

@inproceedings{
    hu2022lora,
    title={Lo{RA}: Low-Rank Adaptation of Large Language Models},
    author={Edward J Hu and Yelong Shen and Phillip Wallis and Zeyuan Allen-Zhu and Yuanzhi Li and Shean Wang and Lu Wang and Weizhu Chen},
    booktitle={International Conference on Learning Representations},
    year={2022},
    url={https://openreview.net/forum?id=nZeVKeeFYf9}
}

@inproceedings{
    springer2025overtrained,
    title={Overtrained Language Models Are Harder to Fine-Tune},
    author={Jacob Mitchell Springer and Sachin Goyal and Kaiyue Wen and Tanishq Kumar and Xiang Yue and Sadhika Malladi and Graham Neubig and Aditi Raghunathan},
    booktitle={Forty-second International Conference on Machine Learning},
    year={2025},
    url={https://openreview.net/forum?id=YW6edSufht}
}

@inproceedings{
    kumar2022finetuning,
    title={Fine-Tuning can Distort Pretrained Features and Underperform Out-of-Distribution},
    author={Ananya Kumar and Aditi Raghunathan and Robbie Matthew Jones and Tengyu Ma and Percy Liang},
    booktitle={International Conference on Learning Representations},
    year={2022},
    url={https://openreview.net/forum?id=UYneFzXSJWh}
}

@inproceedings{isik2025scaling,
  title={Scaling laws for downstream task performance in machine translation},
  author={Isik, Berivan and Ponomareva, Natalia and Hazimeh, Hussein and Paparas, Dimitris and Vassilvitskii, Sergei and Koyejo, Sanmi},
  booktitle={International Conference on Learning Representations},
  year={2025}
}

@article{bendavid2010theory,
  title={A theory of learning from different domains},
  author={Ben-David, Shai and Blitzer, John and Crammer, Koby and Kulesza, Alex and Pereira, Fernando and Vaughan, Jennifer Wortman},
  journal={Machine Learning},
  volume={79},
  number={1},
  pages={151--175},
  year={2010},
  publisher={Springer}
}

@article{cai2021transfer,
  title={Transfer Learning for Nonparametric Classification: Minimax Rate and Adaptive Classifier},
  author={Cai, T. Tony and Wei, Hongji},
  journal={The Annals of Statistics},
  volume={49},
  number={1},
  pages={100--128},
  year={2021},
}

@article{wu2026data,
  title={When do data mixtures improve scaling laws? Insights from high-dimensional regression},
  author={Wu, Diyuan and Chen, Lehan and Misiakiewicz, Theodor and Mondelli, Marco},
  journal={arXiv preprint arXiv:2609.38011},
  year={2026}
}

@misc{codeparrot_clean,
  author={{CodeParrot}},
  title={{CodeParrot-Clean} Dataset},
  year={2021},
  publisher={Hugging Face},
  howpublished={\url{https://huggingface.co/datasets/codeparrot/codeparrot-clean}},
}


\clearpage
\appendix
\section*{Appendix}
\addcontentsline{toc}{section}{Appendix}







\section{Proof of Theorem \ref{thm:total-risk}}
\label{app:deteq}
We first introduce a few quantities that appear in the approximation error bounds. Following \citep[Appendix B]{wu2026improved}, given a covariance matrix $\bSigma \in \R^{d \times d}$ with eigenvalues $\xi_1 \geq \xi_2 \geq \xi_3 \geq \cdots$, we define the \emph{intrinsic dimension} $r_{\bSigma} (k)$ at level $k\in \naturals$ of $\bSigma$ to be the intrinsic dimension of the covariance matrix $\bSigma_{\geq k} = \diag ( \xi_k^2, \xi_{k+1}^2, \ldots)$, which is
        \[
        r_{\bSigma} (k) := \frac{\Tr( \bSigma_{\geq k})}{\| \bSigma_{\geq k} \|_{\rm op}} = \frac{\sum_{j =k}^d \xi^2_j}{\xi^2_{k}}.
        \] 

Given any $\lambda > 0,$ define 
\begin{equation}
\label{eq:rho_lambda}
    \begin{split}
        \rho_\lambda(\bSigma,n) := 1 + \frac{n \cdot \xi^2_{ \lfloor \eta_*\cdot n \rfloor}}{\lambda} \left\{  1 + \frac{r_{\bSigma} ( \lfloor \eta_* \cdot n \rfloor) \vee n}{n} \log \left( r_{\bSigma} ( \lfloor \eta_* \cdot n \rfloor) \vee n \right) \right\},
    \end{split}
\end{equation}
where $\eta_* \in (0,1/2)$ is a constant that  only depends on the constants in Assumption \ref{asm:distr}. For fixed $\lambda$ and for $z \in \{\pt, \ft\}$, we define $\rho_z(\lambda): = \rho_\lambda(\bSigma_z, n_z)$. We drop the dependence on $\lambda$ when there is no confusion.

Define the source and target resolvents by
\begin{equation}
\bH_{\pt}
:=
(\bX_{\pt}^{\top}\bX_{\pt}+\lambda_{\pt}\bI_d)^{-1},
\qquad
\bH_{\ft}
:=
(\bX_{\ft}^{\top}\bX_{\ft}+\lambda_{\ft}\bI_d)^{-1}.
\label{eq:resolvent-def}
\end{equation}

We use throughout the effective regularization parameters
$\lamonmu_{\pt},\lamonmu_{\ft}$, the matrices
$\bW_{\pt},\bW_{\ft}$, and the prefactors
$c_{\pt},c_{\ft}$ defined in \eqref{eq:effective-ridge-main},
\eqref{eq:W-def-main}, and \eqref{eq:c-def-main}.

The proof of \Cref{thm:total-risk} has two steps.
First, we decompose the risk exactly into a bias term and two variance terms. We then control each term by its deterministic counterpart.

\begin{lemma}[Exact risk decomposition]
\label{lemm:exact-risk}
The conditional excess risk satisfies
\begin{equation}
\cR
=
\cB+\cV_{\pt}+\cV_{\ft},
\label{eq:exact-risk-decomposition}
\end{equation}
where
\begin{equation}
\cB
:=
\lambda_{\ft}^2
\left(
\bbeta_{\pt}-\bbeta_{\ft}
-\lambda_{\pt}\bH_{\pt}\bbeta_{\pt}
\right)^\top
\bH_{\ft}\bSigma_{\ft}\bH_{\ft}
\left(
\bbeta_{\pt}-\bbeta_{\ft}
-\lambda_{\pt}\bH_{\pt}\bbeta_{\pt}
\right),
\label{eq:exact-bias}
\end{equation}
\begin{equation}
\cV_{\pt}
:=
\lambda_{\ft}^2\sigma_{\pt}^2
\Tr\left(
\bH_{\ft}\bSigma_{\ft}\bH_{\ft}
\bH_{\pt}\bX_{\pt}^{\top}\bX_{\pt}\bH_{\pt}
\right),
\label{eq:exact-pretraining-variance}
\end{equation}
and
\begin{equation}
\cV_{\ft}
:=
\sigma_{\ft}^2
\Tr\left(
\bH_{\ft}\bSigma_{\ft}\bH_{\ft}
\bX_{\ft}^{\top}\bX_{\ft}
\right).
\label{eq:exact-finetuning-variance}
\end{equation}
\end{lemma}

\begin{proof}
Using $\bH_{\ft}\bX_{\ft}^{\top}\bX_{\ft}
=\bI_d-\lambda_{\ft}\bH_{\ft}$, the error can be written as
\begin{equation}
\hat{\bbeta}_{\ft}-\bbeta_{\ft}
=
\lambda_{\ft}\bH_{\ft}
(\hat{\bbeta}_{\pt}-\bbeta_{\ft})
+
\bH_{\ft}\bX_{\ft}^{\top}\varepsilon_{\ft}.
\label{eq:ft-error-first}
\end{equation}

Similarly, since
$\bH_{\pt}\bX_{\pt}^{\top}\bX_{\pt}
=\bI_d-\lambda_{\pt}\bH_{\pt}$,
\begin{equation}
\hat{\bbeta}_{\pt}
=
\bbeta_{\pt}
-\lambda_{\pt}\bH_{\pt}\bbeta_{\pt}
+
\bH_{\pt}\bX_{\pt}^{\top}\varepsilon_{\pt}.
\label{eq:pretraining-noise-decomposition}
\end{equation}
Substituting this into \eqref{eq:ft-error-first} gives
\begin{equation}
\begin{aligned}
\hat{\bbeta}_{\ft}-\bbeta_{\ft}
={}&
\lambda_{\ft}\bH_{\ft}
\left(
\bbeta_{\pt}-\bbeta_{\ft}
-\lambda_{\pt}\bH_{\pt}\bbeta_{\pt}
\right)
+
\lambda_{\ft}\bH_{\ft}\bH_{\pt}
\bX_{\pt}^{\top}\epsilon_{\pt}
+
\bH_{\ft}\bX_{\ft}^{\top}\epsilon_{\ft}.
\end{aligned}
\label{eq:error-decomposition}
\end{equation}
Conditional on $\bX_{\pt}$ and $\bX_{\ft}$, the first term is deterministic, while the two noise terms are independent and zero mean. Hence, all cross terms vanish in the conditional risk. The deterministic term is exactly $\cB$ by \eqref{eq:exact-bias}. For the source-noise term, using
$\E[\epsilon_{\pt}\epsilon_{\pt}^{\top}]
=\sigma_{\pt}^2\bI$ and
$\E[\epsilon^\top A\epsilon]
=\Tr(A\E[\epsilon\epsilon^\top])$,
$$
\begin{aligned}
&\lambda_{\ft}^2
\E\left[
\epsilon_{\pt}^{\top}
\bX_{\pt}\bH_{\pt}\bH_{\ft}
\bSigma_{\ft}
\bH_{\ft}\bH_{\pt}\bX_{\pt}^{\top}
\epsilon_{\pt}
\mid \bX_{\pt},\bX_{\ft}
\right]=
\lambda_{\ft}^2\sigma_{\pt}^2
\Tr\left(
\bH_{\ft}\bSigma_{\ft}\bH_{\ft}
\bH_{\pt}\bX_{\pt}^{\top}\bX_{\pt}\bH_{\pt}
\right)
=
\cV_{\pt},
\end{aligned}
$$
where we used the cyclicity of the trace.
Likewise, since
$\E[\epsilon_{\ft}\epsilon_{\ft}^{\top}]
=\sigma_{\ft}^2\bI$,
$$
\begin{aligned}
&\E\left[
\epsilon_{\ft}^{\top}
\bX_{\ft}\bH_{\ft}
\bSigma_{\ft}
\bH_{\ft}\bX_{\ft}^{\top}
\epsilon_{\ft}
\mid \bX_{\pt},\bX_{\ft}
\right]
=
\sigma_{\ft}^2
\Tr\left(
\bH_{\ft}\bSigma_{\ft}\bH_{\ft}
\bX_{\ft}^{\top}\bX_{\ft}
\right)
=
\cV_{\ft}.
\end{aligned}
$$
Combining the three terms gives \eqref{eq:exact-risk-decomposition}.
\end{proof}

\begin{proposition}[Fine-tuning variance]
\label{prop:finetuning-variance}
Let $\rho_{\ft}:=\rho_{\lambda_{\ft}}(n_{\ft})$ be defined
as in \cite[Appendix~B]{wu2026improved}.
Fix $D,K>0$ and assume
$$
\lambda_{\ft}\rho_{\ft}
\geq
\|\bSigma_{\ft}\|_{\mathrm{op}}n_{\ft}^{-K},
\qquad
\rho_{\ft}^{5/2}\log^{3/2}(n_{\ft})
\leq K\sqrt{n_{\ft}}.
$$
Define
\begin{equation}
\overline{\cV}_{\ft}
:=
\sigma_{\ft}^2
\frac{
\Tr[\bSigma_{\ft}^2
(\bSigma_{\ft}+\lamonmu_{\ft}\bI_d)^{-2}]
}{
n_{\ft}
-
\Tr[\bSigma_{\ft}^2
(\bSigma_{\ft}+\lamonmu_{\ft}\bI_d)^{-2}]
}.
\label{eq:finetuning-variance-deteq}
\end{equation}
Then, for sufficiently large $n_{\ft}$, with probability at least
$1-n_{\ft}^{-D}$,
$$
|\cV_{\ft}-\overline{\cV}_{\ft}|
\leq
C_{D,K}
\frac{\rho_{\ft}^{6}\log^{3/2}(n_{\ft})}{\sqrt{n_{\ft}}}
\overline{\cV}_{\ft}.
$$
\end{proposition}

\begin{proof}
By \eqref{eq:exact-finetuning-variance},
$$
\cV_{\ft}
=
\sigma_{\ft}^2
\Tr\left(
\bH_{\ft}\bSigma_{\ft}\bH_{\ft}
\bX_{\ft}^\top\bX_{\ft}
\right).
$$
The trace above is exactly the
functional in
\cite[Theorem~5, equation~(37)]{wu2026improved}
with $\bA=\bI_d$. Therefore, that equation gives
$$
\sigma_{\ft}^2
\frac{
\Tr\left[
\bSigma_{\ft}^2
(\bSigma_{\ft}+\lamonmu_{\ft}\bI_d)^{-2}
\right]
}{
n_{\ft}
-
\Tr\left[
\bSigma_{\ft}^2
(\bSigma_{\ft}+\lamonmu_{\ft}\bI_d)^{-2}
\right]
}=\overline{\cV}_{\ft}.
$$
More precisely, the finite-sample bound in the same theorem yields
$$
|\cV_{\ft}-\overline{\cV}_{\ft}|
\leq
C_{D,K}
\frac{\rho_{\ft}^{6}\log^{3/2}(n_{\ft})}{\sqrt{n_{\ft}}}
\overline{\cV}_{\ft}
$$
with probability at least $1-n_{\ft}^{-D}$.
\end{proof}

\begin{proposition}[Pretraining variance]
\label{prop:pretraining-variance}
Assume the two design matrices are independent and satisfy the regularization conditions of \Cref{prop:finetuning-variance}. Set
$$
\varepsilon_{\ft}
:=
\frac{\rho_{\ft}^{6}\log^{5/2}(n_{\ft})}{\sqrt{n_{\ft}}},
\qquad
\varepsilon_{\pt}
:=
\frac{\rho_{\pt}^{6}\log^{3/2}(n_{\pt})}{\sqrt{n_{\pt}}},
$$
and assume $\varepsilon_{\pt},\varepsilon_{\ft}\leq1$.
Define
\begin{equation}
\overline{\cV}_{\pt}
:=
\sigma_{\pt}^2c_{\ft}c_{\pt}
\Tr(\bW_{\ft}\bW_{\pt}).
\label{eq:pretraining-variance-deteq}
\end{equation}
Then, for sufficiently large sample sizes, with probability at least
$1-n_{\pt}^{-D}-n_{\ft}^{-D}$,
$$
|\cV_{\pt}-\overline{\cV}_{\pt}|
\leq
C_{D,K}
(\varepsilon_{\pt}+\varepsilon_{\ft})
\overline{\cV}_{\pt}.
$$
\end{proposition}

\begin{proof}
By \eqref{eq:exact-pretraining-variance},
$$
\cV_{\pt}
=
\lambda_{\ft}^2\sigma_{\pt}^2
\Tr\left(
\bH_{\ft}\bSigma_{\ft}\bH_{\ft}
\bH_{\pt}\bX_{\pt}^{\top}\bX_{\pt}\bH_{\pt}
\right).
$$
We remove the two random designs one at a time. First, we condition on $\bX_{\pt}$, so that
$\bH_{\pt}\bX_{\pt}^{\top}\bX_{\pt}\bH_{\pt}$
is fixed and positive semidefinite. The trace in $\cV_{\pt}$ is the
functional in
\cite[Theorem~5, equation~(36)]{wu2026improved} with
$$
\bA_{\ft}
=
\bSigma_{\ft}^{-1/2}
\bH_{\pt}\bX_{\pt}^{\top}\bX_{\pt}\bH_{\pt}
\bSigma_{\ft}^{-1/2},
\qquad
\bB_{\ft}=\bI_d.
$$
For $\bB_{\ft}=\bI_d$, the deterministic term in \cite[equation~(36)]{wu2026improved} is
$$
\frac{\lamonmu_{\ft}^2}{\lambda_{\ft}^2}
\frac{
n_{\ft}
\Tr\left[
\bA_{\ft}\bSigma_{\ft}^2
(\bSigma_{\ft}+\lamonmu_{\ft}\bI_d)^{-2}
\right]
}{
n_{\ft}
-
\Tr\left[
\bSigma_{\ft}^2
(\bSigma_{\ft}+\lamonmu_{\ft}\bI_d)^{-2}
\right]
}.
$$
Multiplying by $\lambda_{\ft}^2\sigma_{\pt}^2$ and substituting
$\bA_{\ft}$ gives $\sigma_{\pt}^2c_{\ft}
\Tr\left(
\bW_{\ft}
\bH_{\pt}\bX_{\pt}^{\top}\bX_{\pt}\bH_{\pt}
\right).
$
Hence, with probability at least $1-n_{\ft}^{-D}$,
\begin{equation}
|\cV_{\pt}-\sigma_{\pt}^2c_{\ft}
\Tr(
\bW_{\ft}\bH_{\pt}\bX_{\pt}^{\top}\bX_{\pt}\bH_{\pt})|
\leq
C_{D,K}\varepsilon_{\ft}\sigma_{\pt}^2c_{\ft}
\Tr(
\bW_{\ft}\bH_{\pt}\bX_{\pt}^{\top}\bX_{\pt}\bH_{\pt}).
\label{eq:pretraining-first-stage}
\end{equation}

Next, note that
$\Tr(\bW_{\ft}\bH_{\pt}\bX_{\pt}^{\top}\bX_{\pt}\bH_{\pt})$
is the functional in
\cite[Theorem~5, equation~(37)]{wu2026improved}
with
$\bA_{\pt}
=
\bSigma_{\pt}^{-1/2}
\bW_{\ft}
\bSigma_{\pt}^{-1/2}$.
Its deterministic term is
$$
\begin{aligned}
&\sigma_{\pt}^2c_{\ft}
\frac{
\Tr\left[
\bW_{\ft}
\bSigma_{\pt}
(\bSigma_{\pt}+\lamonmu_{\pt}\bI_d)^{-2}
\right]
}{
n_{\pt}
-
\Tr\left[
\bSigma_{\pt}^2
(\bSigma_{\pt}+\lamonmu_{\pt}\bI_d)^{-2}
\right]
}=
\sigma_{\pt}^2c_{\ft}c_{\pt}
\Tr(\bW_{\ft}\bW_{\pt})
=
\overline{\cV}_{\pt}.
\end{aligned}
$$
Therefore, with probability at least $1-n_{\pt}^{-D}$,
\begin{equation}
|\sigma_{\pt}^2c_{\ft}
\Tr\left(
\bW_{\ft}
\bH_{\pt}\bX_{\pt}^{\top}\bX_{\pt}\bH_{\pt}
\right)-\overline{\cV}_{\pt}|
\leq
C_{D,K}\varepsilon_{\pt}\overline{\cV}_{\pt}.
\label{eq:pretraining-second-stage}
\end{equation}

Since $\varepsilon_{\pt}\leq1$, 
\eqref{eq:pretraining-second-stage} implies
$$
\sigma_{\pt}^2c_{\ft}
\Tr\left(
\bW_{\ft}
\bH_{\pt}\bX_{\pt}^{\top}\bX_{\pt}\bH_{\pt}
\right)
\leq
C_{D,K}\overline{\cV}_{\pt}.
$$
Thus, by the triangle inequality,
\eqref{eq:pretraining-first-stage} and
\eqref{eq:pretraining-second-stage} give
$$
|\cV_{\pt}-\overline{\cV}_{\pt}|
\leq
C_{D,K}
(\varepsilon_{\pt}+\varepsilon_{\ft})
\overline{\cV}_{\pt}.
$$

Since the first bound holds conditionally on $\bX_{\pt}$ with failure probability at most $n_{\ft}^{-D}$, averaging over $\bX_{\pt}$ gives the same unconditional failure probability. A union bound with the source event then gives probability at least $1-n_{\pt}^{-D}-n_{\ft}^{-D}$.
\end{proof}

\begin{proposition}[Bias]
\label{prop:bias}
Assume
$$
\varepsilon_{\pt}
:=
\frac{\rho_{\pt}^{6}\log^{5/2}(n_{\pt})}{\sqrt{n_{\pt}}},
\qquad
\varepsilon_{\ft}
:=
\frac{\rho_{\ft}^{6}\log^{5/2}(n_{\ft})}{\sqrt{n_{\ft}}}.
$$
Under the assumptions of
\Cref{prop:finetuning-variance}, let
$\bbeta_{\pt},\bbeta_{\ft}\in\R^d$ be fixed.
Define
\begin{equation}
\begin{aligned}
\overline{\cB}
:={}&
c_{\ft}
\left[
\bSigma_{\pt}
(\bSigma_{\pt}+\lamonmu_{\pt}\bI_d)^{-1}
\bbeta_{\pt}
-\bbeta_{\ft}
\right]^\top
\bW_{\ft}
\left[
\bSigma_{\pt}
(\bSigma_{\pt}+\lamonmu_{\pt}\bI_d)^{-1}
\bbeta_{\pt}
-\bbeta_{\ft}
\right]
\\
&+
c_{\ft}c_{\pt}\lamonmu_{\pt}^2
\Tr(\bW_{\ft}\bW_{\pt})
\bbeta_{\pt}^\top\bW_{\pt}\bbeta_{\pt},
\end{aligned}
\label{eq:bias-deteq}
\end{equation}
and
\begin{equation}
\mathcal E_{\pt}
:=
c_{\ft}\lamonmu_{\pt}^2
\left(
\bbeta_{\pt}^\top
(\bSigma_{\pt}+\lamonmu_{\pt}\bI_d)^{-1}
\bbeta_{\pt}
\right)
\left\|
\bW_{\ft}^{1/2}
(\bSigma_{\pt}+\lamonmu_{\pt}\bI_d)^{-1}
\bW_{\ft}^{1/2}
\right\|_{\mathrm{op}}.
\label{eq:bias-error-scale}
\end{equation}
Then, for sufficiently large sample sizes, with probability at least
$1-n_{\pt}^{-D}-n_{\ft}^{-D}$,
\begin{equation}
\left|\cB-\overline{\cB}\right|
\leq
C_{D,K}
\left[
(\varepsilon_{\pt}+\varepsilon_{\ft})\overline{\cB}
+
\varepsilon_{\pt}\mathcal E_{\pt}
\right].
\label{eq:bias-approximation}
\end{equation}
\end{proposition}

\begin{proof}
From \eqref{eq:exact-bias}, we have
$$
\cB
=
\lambda_{\ft}^2
\left(
\bbeta_{\pt}-\bbeta_{\ft}
-\lambda_{\pt}\bH_{\pt}\bbeta_{\pt}
\right)^\top
\bH_{\ft}\bSigma_{\ft}\bH_{\ft}
\left(
\bbeta_{\pt}-\bbeta_{\ft}
-\lambda_{\pt}\bH_{\pt}\bbeta_{\pt}
\right).
$$

We first remove the target design and condition on $\bX_{\pt}$, so that
$\bbeta_{\pt}-\bbeta_{\ft}
-\lambda_{\pt}\bH_{\pt}\bbeta_{\pt}$
is fixed. By cyclic property of trace, the quadratic form above
is the functional in
\cite[Theorem~5, equation~(36)]{wu2026improved} with
$$
\bA_{\ft}
=
\bSigma_{\ft}^{-1/2}
\left(
\bbeta_{\pt}-\bbeta_{\ft}
-\lambda_{\pt}\bH_{\pt}\bbeta_{\pt}
\right)
\left(
\bbeta_{\pt}-\bbeta_{\ft}
-\lambda_{\pt}\bH_{\pt}\bbeta_{\pt}
\right)^\top
\bSigma_{\ft}^{-1/2},
\qquad
\bB_{\ft}=\bI_d.
$$
The deterministic term in \cite[equation~(36)]{wu2026improved}, after multiplying by
$\lambda_{\ft}^2$, is
$$
c_{\ft}
\left(
\bbeta_{\pt}-\bbeta_{\ft}
-\lambda_{\pt}\bH_{\pt}\bbeta_{\pt}
\right)^\top
\bW_{\ft}
\left(
\bbeta_{\pt}-\bbeta_{\ft}
-\lambda_{\pt}\bH_{\pt}\bbeta_{\pt}
\right).
$$
Hence, with probability at least $1-n_{\ft}^{-D}$,
\begin{equation}
\begin{aligned}
\Big|
\cB
&-
c_{\ft}
\left(
\bbeta_{\pt}-\bbeta_{\ft}
-\lambda_{\pt}\bH_{\pt}\bbeta_{\pt}
\right)^\top
\bW_{\ft}
\left(
\bbeta_{\pt}-\bbeta_{\ft}
-\lambda_{\pt}\bH_{\pt}\bbeta_{\pt}
\right)
\Big|
\\
&\leq
C_{D,K}\varepsilon_{\ft}
c_{\ft}
\left(
\bbeta_{\pt}-\bbeta_{\ft}
-\lambda_{\pt}\bH_{\pt}\bbeta_{\pt}
\right)^\top
\bW_{\ft}
\left(
\bbeta_{\pt}-\bbeta_{\ft}
-\lambda_{\pt}\bH_{\pt}\bbeta_{\pt}
\right).
\end{aligned}
\label{eq:bias-target-step}
\end{equation}

We next remove the source design. It remains to analyze
$$
c_{\ft}
\left(
\bbeta_{\pt}-\bbeta_{\ft}
-\lambda_{\pt}\bH_{\pt}\bbeta_{\pt}
\right)^\top
\bW_{\ft}
\left(
\bbeta_{\pt}-\bbeta_{\ft}
-\lambda_{\pt}\bH_{\pt}\bbeta_{\pt}
\right).
$$ Expanding the remaining quadratic form gives
$$
\begin{aligned}
c_{\ft}\Big[
&(\bbeta_{\pt}-\bbeta_{\ft})^\top
\bW_{\ft}
(\bbeta_{\pt}-\bbeta_{\ft})
-2\lambda_{\pt}
(\bbeta_{\pt}-\bbeta_{\ft})^\top
\bW_{\ft}\bH_{\pt}\bbeta_{\pt}
+\lambda_{\pt}^2
\bbeta_{\pt}^\top
\bH_{\pt}\bW_{\ft}\bH_{\pt}\bbeta_{\pt}
\Big].
\end{aligned}
$$
The first term is already deterministic, so only the cross and
quadratic terms remain. For the cross term, we apply
\cite[Theorem~6, equation~(38)]{wu2026improved}
with
$$
\bA
=
\bSigma_{\pt}^{-1/2}
\bbeta_{\pt}
(\bbeta_{\pt}-\bbeta_{\ft})^\top
\bW_{\ft}
\bSigma_{\pt}^{-1/2}.
$$
Its deterministic term is
$$
2c_{\ft}\lamonmu_{\pt}
(\bbeta_{\pt}-\bbeta_{\ft})^\top
\bW_{\ft}
(\bSigma_{\pt}+\lamonmu_{\pt}\bI_d)^{-1}
\bbeta_{\pt}.
$$

Substituting the above choice of $\bA$ into the error term in
\cite[Theorem~6, equation~(38)]{wu2026improved} yields a quadratic form $(\bSigma_{\pt}+\lamonmu_{\pt}\bI_d)^{-1}$. To reduce it to
$(\bbeta_{\pt}-\bbeta_{\ft})^\top\bW_{\ft}(\bbeta_{\pt}-\bbeta_{\ft})$,
we use
$$
\bW_{\ft}
(\bSigma_{\pt}+\lamonmu_{\pt}\bI_d)^{-1}
\bW_{\ft}
\preceq
\left\|
\bW_{\ft}^{1/2}
(\bSigma_{\pt}+\lamonmu_{\pt}\bI_d)^{-1}
\bW_{\ft}^{1/2}
\right\|_{\mathrm{op}}
\bW_{\ft},
$$
which gives
\begin{equation}
\begin{aligned}
&\Big|
2c_{\ft}\lambda_{\pt}
(\bbeta_{\pt}-\bbeta_{\ft})^\top
\bW_{\ft}\bH_{\pt}\bbeta_{\pt}
-
2c_{\ft}\lamonmu_{\pt}
(\bbeta_{\pt}-\bbeta_{\ft})^\top
\bW_{\ft}
(\bSigma_{\pt}+\lamonmu_{\pt}\bI_d)^{-1}
\bbeta_{\pt}
\Big|
\\
&\hspace{16em}\leq
C_{D,K}\varepsilon_{\pt}
\sqrt{
\mathcal E_{\pt}\,
c_{\ft}
(\bbeta_{\pt}-\bbeta_{\ft})^\top
\bW_{\ft}
(\bbeta_{\pt}-\bbeta_{\ft})
}.
\end{aligned}
\label{eq:bias-source-cross}
\end{equation}
For the quadratic term, we apply
\cite[Theorem~6, equation~(40)]{wu2026improved}
with
$$
\bA
=
\bSigma_{\pt}^{-1/2}
\bbeta_{\pt}\bbeta_{\pt}^\top
\bSigma_{\pt}^{-1/2},
\qquad
\bB
=
\bSigma_{\pt}^{-1/2}
\bW_{\ft}
\bSigma_{\pt}^{-1/2}.
$$
Its deterministic term is
$$
\begin{aligned}
c_{\ft}\lamonmu_{\pt}^2
\bbeta_{\pt}^\top
(\bSigma_{\pt}+\lamonmu_{\pt}\bI_d)^{-1}
\bW_{\ft}
(\bSigma_{\pt}+\lamonmu_{\pt}\bI_d)^{-1}
\bbeta_{\pt}+c_{\ft}c_{\pt}\lamonmu_{\pt}^2
\Tr(\bW_{\ft}\bW_{\pt})
\bbeta_{\pt}^\top\bW_{\pt}\bbeta_{\pt}.
\end{aligned}
$$
Therefore,
\begin{equation}
\begin{aligned}
\Big|
&c_{\ft}\lambda_{\pt}^2
\bbeta_{\pt}^\top
\bH_{\pt}\bW_{\ft}\bH_{\pt}\bbeta_{\pt}
-
c_{\ft}\lamonmu_{\pt}^2
\bbeta_{\pt}^\top
(\bSigma_{\pt}+\lamonmu_{\pt}\bI_d)^{-1}
\bW_{\ft}
(\bSigma_{\pt}+\lamonmu_{\pt}\bI_d)^{-1}
\bbeta_{\pt}
-
c_{\ft}c_{\pt}\lamonmu_{\pt}^2
\Tr(\bW_{\ft}\bW_{\pt})
\bbeta_{\pt}^\top\bW_{\pt}\bbeta_{\pt}
\Big|
\\
&\hspace{16em}\leq
C_{D,K}\varepsilon_{\pt}
\left(
\mathcal E_{\pt}
+
c_{\ft}c_{\pt}\lamonmu_{\pt}^2
\Tr(\bW_{\ft}\bW_{\pt})
\bbeta_{\pt}^\top\bW_{\pt}\bbeta_{\pt}
\right).
\end{aligned}
\label{eq:bias-source-quadratic}
\end{equation}

We now identify the deterministic terms produced by these two
calculations. Combining them with the deterministic first term in the
expansion gives
$$
\begin{aligned}
&c_{\ft}
(\bbeta_{\pt}-\bbeta_{\ft})^\top
\bW_{\ft}
(\bbeta_{\pt}-\bbeta_{\ft})
-2c_{\ft}\lamonmu_{\pt}
(\bbeta_{\pt}-\bbeta_{\ft})^\top
\bW_{\ft}
(\bSigma_{\pt}+\lamonmu_{\pt}\bI_d)^{-1}
\bbeta_{\pt}
\\
&\quad
+c_{\ft}\lamonmu_{\pt}^2
\bbeta_{\pt}^\top
(\bSigma_{\pt}+\lamonmu_{\pt}\bI_d)^{-1}
\bW_{\ft}
(\bSigma_{\pt}+\lamonmu_{\pt}\bI_d)^{-1}
\bbeta_{\pt}
+c_{\ft}c_{\pt}\lamonmu_{\pt}^2
\Tr(\bW_{\ft}\bW_{\pt})
\bbeta_{\pt}^\top\bW_{\pt}\bbeta_{\pt}=\overline{\cB}.
\end{aligned}
$$
The first three terms form the quadratic
$$
\begin{aligned}
c_{\ft}
\left(
\bbeta_{\pt}-\bbeta_{\ft}
-\lamonmu_{\pt}
(\bSigma_{\pt}+\lamonmu_{\pt}\bI_d)^{-1}
\bbeta_{\pt}
\right)^\top
\bW_{\ft}
\left(
\bbeta_{\pt}-\bbeta_{\ft}
-\lamonmu_{\pt}
(\bSigma_{\pt}+\lamonmu_{\pt}\bI_d)^{-1}
\bbeta_{\pt}
\right).
\end{aligned}
$$
Using the resolvent identity
$$
\bI_d-\lamonmu_{\pt}
(\bSigma_{\pt}+\lamonmu_{\pt}\bI_d)^{-1}
=
\bSigma_{\pt}
(\bSigma_{\pt}+\lamonmu_{\pt}\bI_d)^{-1},
$$
this equals
$$
\begin{aligned}
c_{\ft}
\left[
\bSigma_{\pt}
(\bSigma_{\pt}+\lamonmu_{\pt}\bI_d)^{-1}
\bbeta_{\pt}
-\bbeta_{\ft}
\right]^\top
\bW_{\ft}
\left[
\bSigma_{\pt}
(\bSigma_{\pt}+\lamonmu_{\pt}\bI_d)^{-1}
\bbeta_{\pt}
-\bbeta_{\ft}
\right]=\overline{\cB}_1.
\end{aligned}
$$
Together with the remaining trace term ($c_{\ft}c_{\pt}\lamonmu_{\pt}^2
\Tr(\bW_{\ft}\bW_{\pt})
\bbeta_{\pt}^\top\bW_{\pt}\bbeta_{\pt}=\overline{\cB}_2$), this is exactly
$\overline{\cB}$ in \eqref{eq:bias-deteq}.

It remains to control the errors. By the definition of
$\mathcal E_{\pt}$,
$$
c_{\ft}\lamonmu_{\pt}^2
\bbeta_{\pt}^\top
(\bSigma_{\pt}+\lamonmu_{\pt}\bI_d)^{-1}
\bW_{\ft}
(\bSigma_{\pt}+\lamonmu_{\pt}\bI_d)^{-1}
\bbeta_{\pt}
\leq
\mathcal E_{\pt}.
$$
Moreover,
$$
c_{\ft}
(\bbeta_{\pt}-\bbeta_{\ft})^\top
\bW_{\ft}
(\bbeta_{\pt}-\bbeta_{\ft})
\leq
2\overline{\cB}+2\mathcal E_{\pt}.
$$
Hence Young's inequality gives
$$
\sqrt{
\mathcal E_{\pt}\,
c_{\ft}
(\bbeta_{\pt}-\bbeta_{\ft})^\top
\bW_{\ft}
(\bbeta_{\pt}-\bbeta_{\ft})
}
\leq
C(\overline{\cB}+\mathcal E_{\pt}).
$$
Since the trace correction in
\eqref{eq:bias-source-quadratic} is a nonnegative term in
$\overline{\cB}$, combining
\eqref{eq:bias-source-cross} and
\eqref{eq:bias-source-quadratic} gives
$$
\begin{aligned}
\Big|
&c_{\ft}
\left(
\bbeta_{\pt}-\bbeta_{\ft}
-\lambda_{\pt}\bH_{\pt}\bbeta_{\pt}
\right)^\top
\bW_{\ft}
\left(
\bbeta_{\pt}-\bbeta_{\ft}
-\lambda_{\pt}\bH_{\pt}\bbeta_{\pt}
\right)
-\overline{\cB}
\Big|
\leq
C_{D,K}\varepsilon_{\pt}
(\overline{\cB}+\mathcal E_{\pt}),
\end{aligned}
$$
Applying both source bounds with exponent $D+1$ makes their
joint failure probability at most $n_{\pt}^{-D}$, after adjusting
$C_{D,K}$. On this source event and the target event in
\eqref{eq:bias-target-step},
$$
\left|\cB-\overline{\cB}\right|
\leq
C_{D,K}
\left[
(\varepsilon_{\pt}+\varepsilon_{\ft})\overline{\cB}
+
\varepsilon_{\pt}\mathcal E_{\pt}
\right],
$$
where $\varepsilon_{\ft}\leq1$ is used to absorb the product of the
target error with the source error.
Finally, the target bound holds conditionally on $\bX_{\pt}$ with
failure probability at most $n_{\ft}^{-D}$, so averaging over
$\bX_{\pt}$ gives the same unconditional failure probability. A union
bound with the source event gives probability at least
$1-n_{\pt}^{-D}-n_{\ft}^{-D}$.
\end{proof}

\begin{proof}[Proof of \Cref{thm:total-risk}]
\label{app:total-risk-proof}
By \Cref{lemm:exact-risk} and the definition of
$\overline{\cR}$,
\[
\cR=\cB+\cV_{\pt}+\cV_{\ft},
\qquad
\overline{\cR}
=
\overline{\cB}
+\overline{\cV}_{\pt}
+\overline{\cV}_{\ft}.
\]
Define
$$
\varepsilon_{\pt}
:=
\frac{\rho_{\pt}^{6}\log^{5/2}(n_{\pt})}{\sqrt{n_{\pt}}},
\qquad
\varepsilon_{\ft}
:=
\frac{\rho_{\ft}^{6}\log^{5/2}(n_{\ft})}{\sqrt{n_{\ft}}}.
$$
Since the above quantities dominate all error factors appearing in
\Cref{prop:bias,prop:pretraining-variance,prop:finetuning-variance}, by combining them and taking a union bound
we obtain, with probability at least
$1-C_D(n_{\pt}^{-D}+n_{\ft}^{-D})$,
\[
\begin{aligned}
|\cR-\overline{\cR}|
&\leq
|\cB-\overline{\cB}|
+
|\cV_{\pt}-\overline{\cV}_{\pt}|
+
|\cV_{\ft}-\overline{\cV}_{\ft}|
\leq
C_{D,K}
\left[
(\varepsilon_{\pt}+\varepsilon_{\ft})
\overline{\cR}
+
\varepsilon_{\pt}\mathcal E_{\pt}
\right].
\end{aligned}
\]
Since $\varepsilon_{\pt},\varepsilon_{\ft}\leq1$,
this also implies
\[
\cR
\leq
C_{D,K}
\left(
\overline{\cR}
+
\varepsilon_{\pt}\mathcal E_{\pt}
\right).
\]
\end{proof}

\section{Proof of Theorem \ref{thm:ridge-transfer-scaling}}
\label{app:scaling-ridge}
Throughout this section, we use \(f_n\asymp g_n\) for two-sided constant-order equivalence,
\(f_n=\widetilde\Theta(g_n)\) for equivalence up to fixed powers of \(\log n\),
and \(x_+:=\max\{x,0\}\). 
The proof has three main steps. First, we determine the scaling of the effective source and target regularization parameters. Next, we derive the scaling of the spectral quantities appearing in the bias and variance terms. Finally, we combine these results to obtain the scaling of the total risk. The required power-law sum estimates are given at the end of the section.

We first derive the scaling of the effective regularizers. For eigenvalues $\xi_k\asymp k^{-\alpha}$, define
\begin{equation}
 S_{1,\alpha}(\lambda):=\sum_{k\ge1}\frac{\xi_k}{\xi_k+\lambda},
 \qquad
 S_{2,\alpha}(\lambda):=\sum_{k\ge1}\frac{\xi_k^2}{(\xi_k+\lambda)^2}.
\end{equation}

\begin{lemma}[Effective-dimension traces]
\label{lemm:effective-dimension-traces}
Let $\alpha>1$. As $\lambda\to0^+$,
\begin{equation}
 S_{1,\alpha}(\lambda)
 \asymp
 S_{2,\alpha}(\lambda)
 \asymp
 S_{1,\alpha}(\lambda)-S_{2,\alpha}(\lambda)
 \asymp
 \lambda^{-1/\alpha}.
\label{eq:effective-dimension-traces}
\end{equation}
\end{lemma}

\begin{proof}
Since $\xi_k\asymp k^{-\alpha}$,
$$
\frac{\xi_k}{\xi_k+\lambda}
\asymp
\frac{1}{1+\lambda k^\alpha},
\qquad
\frac{\xi_k^2}{(\xi_k+\lambda)^2}
\asymp
\frac{1}{(1+\lambda k^\alpha)^2}.
$$
Therefore, \Cref{lem:one-scale-polynomial-sums} gives
$$
S_{1,\alpha}(\lambda)
\asymp
S_{2,\alpha}(\lambda)
\asymp
\lambda^{-1/\alpha}.
$$

It remains to control their difference. Let $K:=\lambda^{-1/\alpha}$, so that $\lambda K^\alpha=1$, and split the sum at $K$. Then,
$$
\begin{aligned}
S_{1,\alpha}(\lambda)-S_{2,\alpha}(\lambda)
&=
\sum_{k\ge1}
\frac{\lambda\xi_k}{(\xi_k+\lambda)^2}\\
&\asymp
\sum_{k\ge1}
\frac{\lambda k^\alpha}
{(1+\lambda k^\alpha)^2}\\
&=
\sum_{k\le K}
\frac{\lambda k^\alpha}
{(1+\lambda k^\alpha)^2}
+
\sum_{k>K}
\frac{\lambda k^\alpha}
{(1+\lambda k^\alpha)^2}.
\end{aligned}
$$
For $k\le K$,
we have $\lambda k^\alpha\le1$, and hence
$$
\frac{\lambda k^\alpha}
{(1+\lambda k^\alpha)^2}
\asymp
\lambda k^\alpha.
$$
Thus
$$
\sum_{k\le K}
\frac{\lambda k^\alpha}
{(1+\lambda k^\alpha)^2}
\asymp
\lambda\sum_{k\le K}k^\alpha
\asymp
\lambda K^{\alpha+1}
=
\lambda^{-1/\alpha}.
$$

For $k>K$, we have $\lambda k^\alpha>1$, so
$$
\frac{\lambda k^\alpha}
{(1+\lambda k^\alpha)^2}
\asymp
\frac{1}{\lambda k^\alpha}.
$$
Since $\alpha>1$,
$$
\sum_{k>K}
\frac{\lambda k^\alpha}
{(1+\lambda k^\alpha)^2}
\asymp
\lambda^{-1}\sum_{k>K}k^{-\alpha}
\asymp
\lambda^{-1}K^{1-\alpha}
=
\lambda^{-1/\alpha}.
$$
Both parts are therefore of order $\lambda^{-1/\alpha}$, which proves
\eqref{eq:effective-dimension-traces}.
\end{proof}

We next use these trace estimates to determine the scaling of the effective regularization parameter defined by the self-consistency equation.
\begin{proposition}[Generic effective-ridge scaling]
\label{prop:generic-effective-ridge}
Let \(\alpha>1\), \(N\to\infty\), \(q>-1\),
\(\lambda_N\asymp N^{-q}\), and let
\(\mu_N>0\) solve
\begin{equation}
 \mu_N=\frac{N}{1+\Tr\!\left(\bSigma(\mu_N\bSigma+\lambda_NI)^{-1}\right)},
\label{eq:generic-self-consistency-original}
\end{equation}
where \(\xi_k(\bSigma)\asymp k^{-\alpha}\). Define
\(\lamonmu_{N}:=\lambda_N/\mu_N\). Then the solution is unique and
\begin{equation}
 \mu_N+S_{1,\alpha}(\lamonmu_{N})=N.
\label{eq:generic-self-consistency-reduced}
\end{equation}
Furthermore,
\begin{equation}
 \lamonmu_{N}\asymp N^{-((1+q)\wedge\alpha)},
\label{eq:generic-lambda-star-unified}
\end{equation}
and
\begin{equation}
 \mu_N\asymp
 \begin{cases}
  N, & q\le\alpha-1,\\
  N^{\alpha-q}, & q>\alpha-1.
 \end{cases}
\label{eq:generic-mu-unified}
\end{equation}
More precisely,
$$
N-\mu_N\asymp
\begin{cases}
N^{(1+q)/\alpha}, & q<\alpha-1,\\
N, & q\ge\alpha-1.
\end{cases}
$$
\end{proposition}

\begin{proof}
Factoring $\mu_N$ in \eqref{eq:generic-self-consistency-original} gives \eqref{eq:generic-self-consistency-reduced}. The map $\mu\mapsto\mu+S_{1,\alpha}(\lambda_N/\mu)$ is continuous and strictly increasing from $0$ to $\infty$, so the solution is unique.

We first show $\lamonmu_N\to0$. Otherwise, along some subsequence, $\lamonmu_N$ is bounded away from zero, so $S_{1,\alpha}(\lamonmu_N)=O(1)$ and $\mu_N=N-O(1)$. Hence $\lamonmu_N=\lambda_N/\mu_N\asymp N^{-q-1}\to0$, since $q>-1$, a contradiction.

By \Cref{lemm:effective-dimension-traces},
$$
N-\mu_N
=
S_{1,\alpha}(\lamonmu_N)
\asymp
\lamonmu_N^{-1/\alpha}
=
\left(\frac{\mu_N}{\lambda_N}\right)^{1/\alpha}.
$$
Using $\lambda_N\asymp N^{-q}$ and dividing by $N$ gives
\begin{equation}
1-\frac{\mu_N}{N}
\asymp
N^{(1+q-\alpha)/\alpha}
\left(\frac{\mu_N}{N}\right)^{1/\alpha}.
\label{eq:generic-normalized-balance}
\end{equation}

If $q<\alpha-1$, then $N^{(1+q-\alpha)/\alpha}\to0$, so \eqref{eq:generic-normalized-balance} gives $\mu_N/N\to1$. Thus $\mu_N\asymp N$, and
$$
N-\mu_N
\asymp
\left(\frac{N}{N^{-q}}\right)^{1/\alpha}
=
N^{(1+q)/\alpha},
\qquad
\lamonmu_N
=
\frac{\lambda_N}{\mu_N}
\asymp
N^{-(1+q)}.
$$

If $q=\alpha-1$, then \eqref{eq:generic-normalized-balance} becomes
$$
1-\frac{\mu_N}{N}
\asymp
\left(\frac{\mu_N}{N}\right)^{1/\alpha},
$$
so $\mu_N/N\asymp1$ and $1-\mu_N/N\asymp1$. Hence $\mu_N\asymp N$, $N-\mu_N\asymp N$, and $\lamonmu_N\asymp N^{-\alpha}$.

If $q>\alpha-1$, then $N^{(1+q-\alpha)/\alpha}\to\infty$, so \eqref{eq:generic-normalized-balance} gives $\mu_N/N\to0$. Hence $1-\mu_N/N\asymp1$ and
$$
1
\asymp
N^{(1+q-\alpha)/\alpha}
\left(\frac{\mu_N}{N}\right)^{1/\alpha},
\qquad
\frac{\mu_N}{N}
\asymp
N^{-(1+q-\alpha)}.
$$
Therefore $\mu_N\asymp N^{\alpha-q}$ and $\lamonmu_N=\lambda_N/\mu_N\asymp N^{-\alpha}$.

Combining the three regimes gives
$$
\lamonmu_N\asymp N^{-((1+q)\wedge\alpha)},
\qquad
\mu_N\asymp
\begin{cases}
N, & q\le\alpha-1,\\
N^{\alpha-q}, & q>\alpha-1,
\end{cases}
$$
which proves \eqref{eq:generic-lambda-star-unified} and \eqref{eq:generic-mu-unified}.
\end{proof}

We now specialize \Cref{prop:generic-effective-ridge} to the source and target problems. This gives the effective regularization scales that will be used throughout the remainder of the proof.

We apply \Cref{prop:generic-effective-ridge} with
\((N,\alpha,q)=(n,b,\gamma_{\lambda_{\ft}})\) for the source object:
\begin{equation}
\lamonmu_{\ft}
\asymp
n^{-\lamonmupowerft},
\qquad
\lamonmupowerft
:=
(1+\gamma_{\lambda_{\ft}})\wedge b.
\label{eq:lambda-ft-star-corrected}
\end{equation}
Moreover,
\begin{equation}
\mu_{\ft}
\asymp
\begin{cases}
n, & \gamma_{\lambda_{\ft}}\le b-1,\\
n^{b-\gamma_{\lambda_{\ft}}}, & \gamma_{\lambda_{\ft}}>b-1.
\end{cases}
\label{eq:mu-ft-star-corrected}
\end{equation}

For the target objects, we use \(N=n_{\pt}\asymp n^{\gamma_{n}}\), \(\alpha=a\), and
\(q=\gamma_{\lambda_{\pt}}/\gamma_{n}\):
\begin{equation}
\lamonmu_{\pt}
\asymp
n^{-\lamonmupowerpt},
\qquad
\lamonmupowerpt
:=
(\gamma_{\lambda_{\pt}}+\gamma_n)\wedge(a\gamma_n).
\label{eq:lambda-pt-star-corrected}
\end{equation}
Moreover,
\begin{equation}
\mu_{\pt}
\asymp
\begin{cases}
n^{\gamma_n},
& \gamma_{\lambda_{\pt}}\le(a-1)\gamma_n,\\
n^{a\gamma_n-\gamma_{\lambda_{\pt}}},
& \gamma_{\lambda_{\pt}}>(a-1)\gamma_n.
\end{cases}
\label{eq:mu-pt-star-corrected}
\end{equation}

At either threshold, both \(\mu\) and the corresponding gap
\(N-\mu\) are of order \(N\).

Now for the next step, we define the corresponding effective cutoffs
\begin{equation}
K_{\pt}:=\lamonmu_{\pt}^{-1/a},
\qquad
K_{\ft}:=\lamonmu_{\ft}^{-1/b}.
\label{eq:effective-spectral-cutoffs}
\end{equation}

\begin{proposition}[Uniform denominator scaling]
\label{prop:uniform-denominator-scaling}
In the setting of \Cref{prop:generic-effective-ridge}, we have
\begin{equation}
 N-S_{2,\alpha}(\lamonmu_{N})\asymp N.
\label{eq:uniform-denominator}
\end{equation}
\end{proposition}

\begin{proof}
By \eqref{eq:generic-self-consistency-reduced},
$$
 N-S_{2,\alpha}(\lamonmu_{N})
 =\mu_N+S_{1,\alpha}(\lamonmu_{N})-S_{2,\alpha}(\lamonmu_{N}).
$$
If $q\le\alpha-1$, then $\mu_N\asymp N$. If $q>\alpha-1$, then
$S_{1,\alpha}=N-\mu_N\asymp N$, and
\Cref{lemm:effective-dimension-traces} gives
$S_{1,\alpha}-S_{2,\alpha}\asymp S_{1,\alpha}\asymp N$.
The upper bound is immediate because $S_{2,\alpha}\ge0$.
\end{proof}

Recall the prefactors $c_{\pt}$ and $c_{\ft}$ in
\eqref{eq:c-def-main}. By \Cref{prop:uniform-denominator-scaling},
$$
n_\pt-
\Tr\left[
\bSigma_\pt^2
(\bSigma_\pt+\lamonmu_\pt I)^{-2}
\right]
\asymp n_\pt,
\qquad n_\ft-
\Tr\left[
\bSigma_\ft^2
(\bSigma_\ft+\lamonmu_\ft I)^{-2}
\right]
\asymp n_\ft.
$$
Substituting this into the definitions of $c_{\pt}$ and $c_{\ft}$ gives
\begin{equation}
\label{eq:source-target-prefactors}
c_{\pt}\asymp n_{\pt}^{-1}\asymp n^{-\gamma_n},
\qquad
c_{\ft}\asymp\lamonmu_{\ft}^2
\asymp n^{-2\lamonmupowerft}.
\end{equation}

Now, with the effective regularization parameters and prefactors established, we determine the scaling of the two variance terms.

For $\cV_\ft,$ recall that \(\sigma_{\pt}^2\asymp\sigma_{\ft}^2\asymp1\). \Cref{lemm:effective-dimension-traces}, applied with
$\alpha=b$, gives
$$
 \Tr\!\left[
 \bSigma_{\ft}^2
 (\bSigma_{\ft}+\lamonmu_{\ft}I)^{-2}
 \right]
 =
 S_{2,b}(\lamonmu_{\ft})
 \asymp
 \lamonmu_{\ft}^{-1/b}
 \asymp
 n^{\lamonmupowerft/b}.
$$
\Cref{prop:uniform-denominator-scaling}, applied with
$N=n$ and $\alpha=b$, gives $n-S_{2,b}(\lamonmu_{\ft})\asymp n$. Hence the fine-tuning variance satisfies
$$
 \overline{\cV}_{\ft}
 \asymp
 \frac{S_{2,b}(\lamonmu_{\ft})}
 {n-S_{2,b}(\lamonmu_{\ft})}
 \asymp
 n^{-1+\lamonmupowerft/b}.
$$
Finally, using $\lamonmupowerft=(1+\gamma_{\lambda_{\ft}})\wedge b$ gives
\begin{equation}
 \Gamma_{\cV_{\ft}}
 =
 1-\frac{\lamonmupowerft}{b}
 =
 1-\frac{(1+\gamma_{\lambda_{\ft}})\wedge b}{b}.
\label{eq:gamma-Vft-corrected}
\end{equation}

The fine-tuning variance depends only on the target spectrum. The pretraining variance and the second bias term also involve the interaction between the source and target spectra, through the cross-trace $\Tr(\bW_{\pt}\bW_{\ft})$. We therefore determine its scaling next.

Set $u:=\frac{\lamonmupowerpt}{a}$ and $v:=\frac{\lamonmupowerft}{b}$. Recalling the matrices $\bW_{\pt}$ and $\bW_{\ft}$ from \eqref{eq:W-def-main} and using
$(\bSigma_{\pt})_{k,k}\asymp k^{-a}$, $(\bSigma_{\ft})_{k,k}\asymp k^{-b}$, we have
$$
\Tr(\bW_{\pt}\bW_{\ft})
\asymp
\sum_{k\ge1}
\frac{k^{-a-b}}
{(k^{-a}+\lamonmu_{\pt})^2
 (k^{-b}+\lamonmu_{\ft})^2}.
$$
Applying  \Cref{lem:trace_approx} with
$\mu=\lamonmu_{\pt}$, $\nu=\lamonmu_{\ft}$,
$\gamma_1=\gamma_2=2$,
$s=(a+b)/2$, and
$\delta_a=\delta_b=1/2$. Here note that
$s+a\delta_a+b\delta_b=a+b$, and the exponent in
\Cref{lem:trace_approx} is
$$
\eta=1-(a+b)+2a+2b=a+b+1.
$$

Since $\lamonmu_{\pt}\asymp n^{-au}$ and
$\lamonmu_{\ft}\asymp n^{-bv}$, we have
$\lamonmu_{\pt}^{1/a}\asymp n^{-u}$ and
$\lamonmu_{\ft}^{1/b}\asymp n^{-v}$. So, the relative size of the two effective scales is determined by the ordering of $u$ and $v$.

Suppose first that $u\le v$. If $\eta\le2a$, equivalently
$b-a+1\le0$, case~(1) of \Cref{lem:trace_approx} gives
$$
\Tr(\bW_{\pt}\bW_{\ft})
=
\widetilde\Theta\!\left(\lamonmu_{\pt}^{-\eta/a}\right)
=
\widetilde\Theta\!\left(n^{(a+b+1)u}\right).
$$
If instead $\eta>2a$, equivalently $b-a+1>0$, then
$$
\Tr(\bW_{\pt}\bW_{\ft})
=
\Theta\!\left(
\lamonmu_{\pt}^{-2}
\lamonmu_{\ft}^{-(\eta-2a)/b}
\right)
=
\Theta\!\left(
n^{(a+b+1)u+(b-a+1)(v-u)}
\right),
$$
where we used $\eta-2a=b-a+1$. Therefore, for $u\le v$,
$$
-\Gamma_T
=
(a+b+1)u+(b-a+1)_+(v-u).
$$

The case $v\le u$ follows symmetrically from case~(2) of
\Cref{lem:trace_approx}. Since $\eta-2b=a-b+1$,
$$
-\Gamma_T
=
(a+b+1)v+(a-b+1)_+(u-v).
$$
Combining the two cases gives
\begin{equation}
\Tr(\bW_{\pt}\bW_{\ft})
=
\widetilde\Theta(n^{-\Gamma_T}),
\label{eq:cross-trace-corrected-reference}
\end{equation}
where
\begin{equation}
  \Gamma_T=
  \begin{cases}
    -(a{+}b{+}1)u-(\delta{+}1)_+(v-u), & u\le v,\\[1mm]
    -(a{+}b{+}1)v-(1{-}\delta)_+(u-v), & v\le u,
  \end{cases}
\label{eq:GammaT-main}
\end{equation}

For $\cV_{\pt}$, combining \eqref{eq:source-target-prefactors} with
\eqref{eq:cross-trace-corrected-reference} gives

\begin{equation}
 \overline{\cV}_{\pt}
 =\sigma_{\pt}^2c_{\pt}c_{\ft}\Tr(\bW_{\pt}\bW_{\ft})
 =\widetilde\Theta\!\left(n^{-\gamma_{n}-2\lamonmupowerft-\crosstracescale}\right),
\label{eq:Vpt-corrected}
\end{equation}
so
\begin{equation}
\Gamma_{\cV_{\pt}}=\gamma_{n}+2\lamonmupowerft+\crosstracescale.
\label{eq:gamma-Vpt-corrected}
\end{equation}

 We now start with the bias terms. The cross-trace also appears in $\overline{\cB}_2$, while the remaining bias scaling is determined by the source quadratic form $\bbeta^\top\bW_{\pt}\bbeta$ and the shrinkage contribution in $\overline{\cB}_1$. Recall that
\begin{equation}
 (\bbeta_{\pt})_k=(\bbeta_{\ft})_k=:(\bbeta)_k,
 \qquad
 (\bbeta)_k^2\asymp k^{b-1-2br_{\ft}}.
\label{eq:aligned-beta-assumption}
\end{equation}
Under this signal scaling, we first determine the source quadratic form appearing in $\overline{\cB}_2$, and then the shrinkage contribution appearing in $\overline{\cB}_1$.
Recall that in this case
\begin{equation}
 a-1-2ar_{\pt}=b-1-2br_{\ft}.
\label{eq:smoothness-compatibility}
\end{equation}

We now establish the scaling law for the bias terms. Define \begin{equation}
  \shrinkagebias
  :=\sum_{k\ge1}
  \frac{(\bSigma_{\ft})_{k, k}}{((\bSigma_{\ft})_{k, k}+\lamonmu_{\ft})^2}
  \left(
    \frac{\lamonmu_{\pt}}{(\bSigma_{\pt})_{k, k}+\lamonmu_{\pt}}
  \right)^{2}
  (\bbeta)_k^2,
\label{eq:Bsh-def-main}
\end{equation} since under
\eqref{eq:power-law-spectra} and \eqref{eq:source-conditions} the two covariances
are both diagonal,  \eqref{eq:main-B1} factors as
$\overline{\cB}_1=c_{\ft}\shrinkagebias.$ Further, by \eqref{eq:main-B2}, we remain to characterize the scaling of $\bbeta^\top\bW_{\pt}\bbeta.$ We compute these scalings in the Proposition below.

\begin{proposition}\label{prop:three-spectral-sums}
 Under the same assumptions in Theorem \ref{thm:ridge-transfer-scaling},
\begin{equation}
\begin{gathered}
  \bbeta^\top\bW_{\pt}\bbeta=\widetilde\Theta(n^{-\Gamma_S}),
  \qquad
  \shrinkagebias=\widetilde\Theta(n^{-\Gamma_{\rm sh}}).
\end{gathered}
\label{eq:three-sums-exponents}
\end{equation}
where
\begin{equation}
  \Gamma_S=-(a+b-2br_{\ft})_+u,
\label{eq:GammaS-main}
\end{equation}
and
\begin{equation}
  \Gamma_{\rm sh}=
  \begin{cases}
    \begin{aligned}
      &2au-2bu\big(1{+}\tfrac ab{-}r_{\ft}\big)_+ -2b(v{-}u)(1{-}r_{\ft})_+,
    \end{aligned} & u\le v,\\[5mm]
    \begin{aligned}
      &2au-2bv\big(1{+}\tfrac ab{-}r_{\ft}\big)_+ -2b(u{-}v)\big(\tfrac ab{-}r_{\ft}\big)_+,
    \end{aligned} & v\le u.
  \end{cases}
\label{eq:Gammash-main}
\end{equation}
\end{proposition}

\begin{proof}
We already established cross-scaling. Now for the other two, we establish them separately.

Using $(\bSigma_{\pt})_{k, k}\asymp k^{-a}$ and \eqref{eq:aligned-beta-assumption},
$$
\frac{(\bSigma_{\pt})_{k, k}(\bbeta)_k^2}
{((\bSigma_{\pt})_{k, k}+\lamonmu_{\pt})^2}
\asymp
\frac{k^{a+b-2br_{\ft}-1}}
{(1+\lamonmu_{\pt}k^a)^2}.
$$
Hence
\begin{equation}
\bbeta^\top \bW_{\pt}\bbeta
\asymp
\sum_{k\ge1}
\frac{k^{a+b-2br_{\ft}-1}}
{(1+\lamonmu_{\pt}k^a)^2}.
\label{eq:Spt-sum-corrected}
\end{equation}

If $a+b-2br_{\ft}<2a$, \Cref{lem:one-scale-polynomial-sums} with
$(\alpha,\gamma,\eta)=(a,2,a+b-2br_{\ft})$ gives
\begin{equation}
\bbeta^\top \bW_{\pt}\bbeta
\asymp
\begin{cases}
1, & a+b-2br_{\ft}<0,\\[1mm]
\log(1/\lamonmu_{\pt}), & a+b-2br_{\ft}=0,\\[1mm]
\lamonmu_{\pt}^{-(a+b-2br_{\ft})/a},
& 0<a+b-2br_{\ft}<2a.
\end{cases}
\label{eq:Spt-scaling-corrected}
\end{equation}

If $a+b-2br_{\ft}\ge2a$, then for large $k$,
$$
\frac{k^{a+b-2br_{\ft}-1}}
{(1+\lamonmu_{\pt}k^a)^2}
\asymp
\lamonmu_{\pt}^{-2}k^{b-a-2br_{\ft}-1},
$$
The tail is summable if and only if $b-a-2br_{\ft}<0$. Under the exact alignment condition
\eqref{eq:smoothness-compatibility},
$$
a-b+2br_{\ft}=2ar_{\pt}>0,
$$
since $r_{\pt}>0$. Hence, the quadratic form is always finite under our assumptions.

Using $\lamonmu_{\pt}\asymp n^{-au}$, we therefore obtain
$$
\bbeta^\top \bW_{\pt}\bbeta
=
\widetilde{\Theta}\!\left(
n^{(a+b-2br_{\ft})_+u}
\right).
$$
Equivalently,
\begin{equation}
\bbeta^\top \bW_{\pt}\bbeta
=
\widetilde{\Theta}(n^{-\Gamma_S}),
\qquad
\Gamma_S:=-(a+b-2br_{\ft})_+u,
\label{eq:tau-S-corrected}
\end{equation}
where the logarithmic boundary case is absorbed by $\widetilde{\Theta}$.

Recall  $\shrinkagebias$ in \eqref{eq:Bsh-def-main}.
Under \eqref{eq:aligned-beta-assumption},
\begin{equation}
\shrinkagebias
\asymp
\lamonmu_{\pt}^2
\sum_{k\ge1}
\frac{k^{2a+2b-2br_{\ft}-1}}
{(1+\lamonmu_{\pt}k^a)^2
 (1+\lamonmu_{\ft}k^b)^2}.
\label{eq:Ysh-powerlaw}
\end{equation}
Rewriting
$(1+\lamonmu_{\pt}k^a)^2
=k^{2a}(k^{-a}+\lamonmu_{\pt})^2$ and
$(1+\lamonmu_{\ft}k^b)^2
=k^{2b}(k^{-b}+\lamonmu_{\ft})^2$, we obtain
\begin{equation}
\shrinkagebias
\asymp
\lamonmu_{\pt}^2
\sum_{k\ge1}
\frac{k^{-2br_{\ft}-1}}
{(k^{-a}+\lamonmu_{\pt})^2
 (k^{-b}+\lamonmu_{\ft})^2}.
\label{eq:Ysh-trace-lemma-form}
\end{equation}
\Cref{lem:trace_approx} applies with
$\mu=\lamonmu_{\pt}$, $\nu=\lamonmu_{\ft}$,
$\gamma_1=\gamma_2=2$, $s=1$,
$\delta_a=br_{\ft}/a$, and $\delta_b=r_{\ft}$, so its relevant
exponent is $2a+2b-2br_{\ft}$.

Recall the effective cutoffs $K_{\pt}$ and $K_{\ft}$ from
\eqref{eq:effective-spectral-cutoffs}.
Suppose first that $K_{\pt}\le K_{\ft}$, equivalently
$\lamonmu_{\pt}^{1/a}\ge\lamonmu_{\ft}^{1/b}$, so case~(1) of
\Cref{lem:trace_approx} applies.

If $2a+2b-2br_{\ft}\le0$, equivalently
$r_{\ft}\ge1+a/b$, the sum in
\eqref{eq:Ysh-trace-lemma-form} is
$\widetilde{\Theta}(1)$, hence
$\shrinkagebias=\widetilde{\Theta}(\lamonmu_{\pt}^2)$.

If $0<2a+2b-2br_{\ft}\le2a$, equivalently
$1\le r_{\ft}<1+a/b$, then
$$
\shrinkagebias
=
\widetilde{\Theta}\!\left(
\lamonmu_{\pt}^{
2-(2a+2b-2br_{\ft})/a}
\right)
=
\widetilde{\Theta}\!\left(
\lamonmu_{\pt}^{\frac{2b}{a}(r_{\ft}-1)}
\right).
$$

If $2a+2b-2br_{\ft}>2a$, equivalently $r_{\ft}<1$, then
$$
\shrinkagebias
=
\Theta\!\left(
\lamonmu_{\ft}^{-(2b-2br_{\ft})/b}
\right)
=
\Theta\!\left(
\lamonmu_{\ft}^{\,2(r_{\ft}-1)}
\right).
$$

Thus, if $K_{\pt}\le K_{\ft}$,
\begin{equation}
\shrinkagebias
=
\begin{cases}
\widetilde{\Theta}\!\left(\lamonmu_{\pt}^{2}\right),
& r_{\ft}\ge1+\frac{a}{b},\\[1mm]
\widetilde{\Theta}\!\left(
\lamonmu_{\pt}^{\frac{2b}{a}(r_{\ft}-1)}
\right),
& 1\le r_{\ft}<1+\frac{a}{b},\\[1mm]
\Theta\!\left(
\lamonmu_{\ft}^{\,2(r_{\ft}-1)}
\right),
& r_{\ft}<1.
\end{cases}
\label{eq:Ysh-Kpt-smaller-corrected}
\end{equation}

Now suppose $K_{\ft}\le K_{\pt}$, equivalently
$\lamonmu_{\pt}^{1/a}\le\lamonmu_{\ft}^{1/b}$, so case~(2) of
\Cref{lem:trace_approx} applies.

If $2a+2b-2br_{\ft}\le0$, equivalently
$r_{\ft}\ge1+a/b$, we again obtain
$\shrinkagebias=\widetilde{\Theta}(\lamonmu_{\pt}^2)$.

If $0<2a+2b-2br_{\ft}\le2b$, equivalently
$a/b\le r_{\ft}<1+a/b$, then
$$
\shrinkagebias
=
\widetilde{\Theta}\!\left(
\lamonmu_{\pt}^2
\lamonmu_{\ft}^{-(2a+2b-2br_{\ft})/b}
\right)
=
\widetilde{\Theta}\!\left(
\lamonmu_{\pt}^2
\lamonmu_{\ft}^{-2(1+a/b-r_{\ft})}
\right).
$$

Finally, if $2a+2b-2br_{\ft}>2b$, equivalently
$r_{\ft}<a/b$, then
$$
\shrinkagebias
=
\Theta\!\left(
\lamonmu_{\pt}^{
2-(2a-2br_{\ft})/a}
\lamonmu_{\ft}^{-2}
\right)
=
\Theta\!\left(
\lamonmu_{\pt}^{2br_{\ft}/a}
\lamonmu_{\ft}^{-2}
\right).
$$

Thus, if $K_{\ft}\le K_{\pt}$,
\begin{equation}
\shrinkagebias
=
\begin{cases}
\widetilde{\Theta}\!\left(\lamonmu_{\pt}^{2}\right),
& r_{\ft}\ge1+\frac{a}{b},\\[1mm]
\widetilde{\Theta}\!\left(
\lamonmu_{\pt}^{2}
\lamonmu_{\ft}^{-2(1+a/b-r_{\ft})}
\right),
& \frac{a}{b}\le r_{\ft}<1+\frac{a}{b},\\[1mm]
\Theta\!\left(
\lamonmu_{\pt}^{2br_{\ft}/a}
\lamonmu_{\ft}^{-2}
\right),
& r_{\ft}<\frac{a}{b}.
\end{cases}
\label{eq:Ysh-Kft-smaller-corrected}
\end{equation}
At $K_{\pt}=K_{\ft}$, either representation gives the same order.

Now suppose
$$
\lamonmu_{\pt}\asymp n^{-au},
\qquad
\lamonmu_{\ft}\asymp n^{-bv}.
$$
Writing
$$
\shrinkagebias
=
\widetilde{\Theta}(n^{-\Gamma_{\mathrm{sh}}}),
$$
we can directly read off the shrinkage exponent.

If $u\le v$, equivalently $K_{\pt}\le K_{\ft}$, then
\begin{equation}
\Gamma_{\mathrm{sh}}
=
\begin{cases}
2au,
& r_{\ft}\ge 1+\frac{a}{b},\\[1mm]
2bu(r_{\ft}-1),
& 1\le r_{\ft}<1+\frac{a}{b},\\[1mm]
2bv(r_{\ft}-1),
& r_{\ft}<1.
\end{cases}
\label{eq:Gamma-sh-u-le-v}
\end{equation}
Equivalently,
\begin{equation}
\Gamma_{\mathrm{sh}}
=
2au
-2bu\left(1+\frac{a}{b}-r_{\ft}\right)_+
-2b(v-u)(1-r_{\ft})_+.
\label{eq:Gamma-sh-u-le-v-compact}
\end{equation}

If $v\le u$, equivalently $K_{\ft}\le K_{\pt}$, then
\begin{equation}
\Gamma_{\mathrm{sh}}
=
\begin{cases}
2au,
& r_{\ft}\ge 1+\frac{a}{b},\\[1mm]
2au-2bv\left(1+\frac{a}{b}-r_{\ft}\right),
& \frac{a}{b}\le r_{\ft}<1+\frac{a}{b},\\[1mm]
2br_{\ft}u-2bv,
& r_{\ft}<\frac{a}{b}.
\end{cases}
\label{eq:Gamma-sh-v-le-u}
\end{equation}
Equivalently,
\begin{equation}
\Gamma_{\mathrm{sh}}
=
2au
-2bv\left(1+\frac{a}{b}-r_{\ft}\right)_+
-2b(u-v)\left(\frac{a}{b}-r_{\ft}\right)_+.
\label{eq:Gamma-sh-v-le-u-compact}
\end{equation}

When $u=v$, both expressions reduce to
$$
\Gamma_{\mathrm{sh}}
=
2au
-2bu\left(1+\frac{a}{b}-r_{\ft}\right)_+,
$$
so the two cutoff orderings agree.
\end{proof}

\begin{proof}[Proof of \Cref{thm:ridge-transfer-scaling}]
For the first bias term, the resolvent identity
$$
\bSigma_{\pt}(\bSigma_{\pt}+\lamonmu_{\pt}I)^{-1}-I
=
-\lamonmu_{\pt}(\bSigma_{\pt}+\lamonmu_{\pt}I)^{-1}
$$
gives
$\overline{\cB}_1=c_{\ft}\shrinkagebias$. Since
$c_{\ft}\asymp n^{-2\lamonmupowerft}$ and
$\shrinkagebias
=\widetilde{\Theta}(n^{-\Gamma_{\mathrm{sh}}})$,
\begin{equation}
\overline{\cB}_1
=
\widetilde{\Theta}\!\left(
n^{-\Gamma_{\cB_1}}
\right),
\qquad
\Gamma_{\cB_1}
=
2\lamonmupowerft+\Gamma_{\mathrm{sh}}.
\label{eq:B1-exponent-corrected}
\end{equation}

For the second bias term,
$$
\overline{\cB}_2
=
c_{\ft}c_{\pt}\lamonmu_{\pt}^2
\Tr(\bW_{\ft}\bW_{\pt})
\bbeta^\top\bW_{\pt}\bbeta.
$$
Using
$c_{\ft}\asymp n^{-2\lamonmupowerft}$,
$c_{\pt}\asymp n^{-\gamma_n}$,
$\lamonmu_{\pt}^2\asymp n^{-2au}$,
$\Tr(\bW_{\ft}\bW_{\pt})
=\widetilde{\Theta}(n^{-\Gamma_T})$, and
$\bbeta^\top\bW_{\pt}\bbeta
=\widetilde{\Theta}(n^{-\Gamma_S})$ gives
\begin{equation}
\overline{\cB}_2
=
\widetilde{\Theta}\!\left(
n^{-\Gamma_{\cB_2}}
\right),
\qquad
\Gamma_{\cB_2}
=
2\lamonmupowerft+\gamma_n+2au+\Gamma_T+\Gamma_S.
\label{eq:B2-exponent-corrected}
\end{equation}
Hence
\begin{equation}
\overline{\cB}
=
\widetilde{\Theta}(n^{-\Gamma_{\cB}}),
\qquad
\Gamma_{\cB}
=
\min\left\{
\Gamma_{\cB_1},
\Gamma_{\cB_2}
\right\}.
\label{eq:gamma-bias-final-corrected}
\end{equation}

For the variance terms, \eqref{eq:gamma-Vft-corrected} gives
$$
\overline{\cV}_{\ft}
=
\widetilde{\Theta}\!\left(
n^{-\Gamma_{\cV_{\ft}}}
\right),
\qquad
\Gamma_{\overline{\cV}_{\ft}}
=
1-v,
$$
while \eqref{eq:gamma-Vpt-corrected} gives
$$
\overline{\cV}_{\pt}
=
\widetilde{\Theta}\!\left(
n^{-\Gamma_{\cV_{\pt}}}
\right),
\qquad
\Gamma_{\cV_{\pt}}
=
\gamma_n+2bv+\Gamma_T.
$$
\end{proof}

\subsection{Technical lemmas}
\label{app:summation-lemmas}
\begin{lemma}
\label{lem:one-scale-polynomial-sums}
Let $\alpha,\gamma>0$, $\eta<\alpha\gamma$, and
$$
 S_{\alpha,\gamma,\eta}(\lambda)
 :=\sum_{k\ge1}\frac{k^{\eta-1}}{(1+\lambda k^\alpha)^\gamma}.
$$
As $\lambda\to0^+$,
\begin{equation}
 S_{\alpha,\gamma,\eta}(\lambda)\asymp
 \begin{cases}
  1, & \eta<0,\\
  \log(1/\lambda), & \eta=0,\\
  \lambda^{-\eta/\alpha}, & 0<\eta<\alpha\gamma.
 \end{cases}
\label{eq:one-scale-polynomial-sum}
\end{equation}
\end{lemma}

\begin{proof}
Set $K=\lambda^{-1/\alpha}$. For $k\le K$, the denominator is
$\asymp1$, so the contribution is $\asymp\sum_{k\le K}k^{\eta-1}$,
which is respectively $1$, $\log K$, or $K^\eta$. For $k>K$,
$$
 \sum_{k>K}\frac{k^{\eta-1}}{(1+\lambda k^\alpha)^\gamma}
 \asymp \lambda^{-\gamma}\sum_{k>K}k^{\eta-\alpha\gamma-1}
 \asymp K^\eta,
$$
where $\eta<\alpha\gamma$ ensures summability. Substituting
$K=\lambda^{-1/\alpha}$ proves the claim.
\end{proof}

\begin{lemma}
\label{lem:trace_approx}
\citep[Appendix F.1]{wu2026improved}
     For any $a,b > 1$ and $s, \delta_a, \delta_b, \gamma_1, \gamma_2 > 0,$ such that $s + a \delta_a + b \delta_b > 1$, let \begin{equation*}
        \begin{split}
            T^s_{(\delta_a, \delta_b ), (\gamma_1, \gamma_2)}(\mu, \nu; a,b) = \sum_{k=1}^\infty \frac{k^{-s- \delta_a a - \delta_b b }}{ (k^{-a} + \mu)^{\gamma_1} (k^{-b} + \nu)^{\gamma_2}  }.
        \end{split}
    \end{equation*}
    We define $\eta = 1-s - a \delta_a - b \delta_b + a \gamma_1 + b \gamma_2.$ Then, the order of the summation above is as follows: \begin{itemize}
        \item[(1)] If $\mu^{\frac{1}{a}} \geq \nu^{\frac{1}{b}},$ then \begin{equation*}
        T^s_{(\delta_a, \delta_b ), (\gamma_1, \gamma_2)}(\mu, \nu; a,b) = \begin{cases}
            \widetilde{\Theta}(1), &\quad \eta \leq 0 \\
            \widetilde{\Theta}(\mu^{- \eta/a}), &\quad 0 < \eta \leq a \gamma_1\\
            \Theta(\mu^{- \gamma_1} \nu^{-(\eta - a \gamma_1)/b} ),& \quad \eta > a \gamma_1
        \end{cases}
    \end{equation*}

    \item[(2)] If $\mu^{\frac{1}{a}} < \nu^{\frac{1}{b}},$ then \begin{equation*}
        T^s_{(\delta_a, \delta_b ), (\gamma_1, \gamma_2)}(\mu, \nu; a,b) = \begin{cases}
            \widetilde{\Theta}(1), &\quad \eta \leq 0 \\
            \widetilde{\Theta}(\nu^{- \eta/b}), &\quad 0 < \eta \leq b \gamma_2\\
            \Theta(\nu^{- \gamma_2} \mu^{-(\eta - b \gamma_2)/a} ),& \quad \eta > b \gamma_2
        \end{cases}
    \end{equation*}
    \end{itemize}
    
\end{lemma}


\section{Proof of Corollary \ref{cor:minimax-optimal-ridge-transfer}}
\label{app:optimal-ridge-rate}

\begin{corollary}[Optimal scaling of the deterministic equivalent]
\label{cor:minimax-optimal-ridge-transfer-full}
Under the assumptions of
\Cref{thm:ridge-transfer-scaling}, 
define
\begin{align}
 \Gamma_{\le}
 &:=
 1-
 \frac{
 a+b+1-(\delta+1)_+
 +2((s-b)_+\wedge a)(1-\gamma_n)
 }{
 [a+b+1-(\delta+1)_+]
 [1+2(s\wedge(a+b))]
 +2\delta((s-b)_+\wedge a)
 },
 \label{eq:ridge-candidate-le}\\
 \Gamma_{\ge}
 &:=
 1-
 \frac{
 (1-\delta)_+
 +2(s\wedge a)(1-\gamma_n)
 }{
 (1-\delta)_+[1+2(s\wedge(a+b))]
 +2\delta(s\wedge a)
 }.
 \label{eq:ridge-candidate-ge}
\end{align}
Then, the optimal exponent of the deterministic
equivalent 
is given by
\begin{equation}
 \Gamma_{\rm ridge}^{\star}
 =
 \max\left\{
 \frac{2(s\wedge b)}{1+2(s\wedge b)},
 \frac{2(s\wedge a)\gamma_n}
      {2(s\wedge a)+(1-\delta)_+},
 \Gamma_{\le},
 \Gamma_{\ge}
 \right\}.
 \label{eq:global-optimal-ridge-exponent}
\end{equation}
The following choices attain
\begin{equation}
 \overline{\cR}
 =
 \widetilde\Theta
 \bigl(n^{-\Gamma_{\rm ridge}^{\star}}\bigr).
 \label{eq:global-optimal-ridge-rate}
\end{equation}

\begin{enumerate}
 \item If $\frac{2(s\wedge b)}{1+2(s\wedge b)}$ in 
 \eqref{eq:global-optimal-ridge-exponent} attains the maximum, using only target data is optimal, and the corresponding optimal ridge is
 \[
  \lambda_{\pt}=\infty,
  \qquad
  \lambda_{\ft}
  \asymp n^{1-b/[1+2(s\wedge b)]}.
 \]

 \item If $\frac{2(s\wedge a)\gamma_n}
      {2(s\wedge a)+(1-\delta)_+}$ in 
 \eqref{eq:global-optimal-ridge-exponent} attains the maximum, using only source data is optimal, and the corresponding optimal ridge is  \[
  \lambda_{\ft}=\infty,
  \qquad
  \lambda_{\pt}
  \asymp
  n^{\gamma_n-
  a\gamma_n/[2(s\wedge a)+(1-\delta)_+]}.
 \]

 \item Otherwise, the optimal ridge is chosen by
 \[
  v^\star=1-\Gamma_{\rm ridge}^{\star},
 \]
 and 
 \[
  u^\star=
  \begin{cases}
   \displaystyle
   \frac{
   \Gamma_{\rm ridge}^{\star}-2b v^\star
   }{2((s-b)\wedge a)},
   &\Gamma_{\rm ridge}^{\star}=\Gamma_{\le},\\[3mm]
   \displaystyle
   \frac{
   \Gamma_{\rm ridge}^{\star}
   -2((s-a)_+\wedge b)v^\star
   }{2(s\wedge a)},
   &\Gamma_{\rm ridge}^{\star}=\Gamma_{\ge},
  \end{cases}
 \]
 which implies
 \[
  \lambda_{\pt}\asymp n^{\gamma_n-au^\star},
  \qquad
  \lambda_{\ft}\asymp n^{1-bv^\star}.
 \]
\end{enumerate}

\end{corollary}


\begin{proof}[Proof of \Cref{cor:minimax-optimal-ridge-transfer-full}]
Note that $\lambda_\pt = \infty, \lambda_\ft = \infty$ exactly correspond to the target-only and auxiliary-only rate in \citep[Theorem 4]{wu2026data}. Thus, it remains to consider the case where non-trivial transfer attains the maximum, i.e., $\Gamma^*_{\rm ridge} = \Gamma_{\geq}$ or $\Gamma^*_{\rm ridge} = \Gamma_{\leq}.$

We first show that the exponents
in \Cref{thm:ridge-transfer-scaling} satisfy
$\Gamma_{\cB_2}>\Gamma_{\cV_{\pt}}$ for every admissible $u>0$,
so this term does not determine the leading order
of the deterministic equivalent.
Indeed, by \eqref{eq:four-exponents} and \eqref{eq:GammaS-main},
\begin{align*}
 \Gamma_{\cB_2}-\Gamma_{\cV_{\pt}}
 &=2au+\Gamma_S=\bigl[2a-(a+b-2s)_+\bigr]u=\bigl[2a\wedge(2s-\delta)\bigr]u>0.
\end{align*}
Thus, the exponent of the deterministic
equivalent reduces to
\[
 \min\left\{
 \Gamma_{\cB_1},
 \Gamma_{\cV_{\pt}},
 \Gamma_{\cV_{\ft}}
 \right\}.
\]

Next, we express the scaling as a function of $u,v.$  Substituting $s=br_{\ft}$ into
\Cref{thm:ridge-transfer-scaling} gives
\begin{equation}
 \Gamma_{\cB_1}
 =2bv+\Gamma_{\rm sh}
 =
 \begin{cases}
  2\bigl((s-b)_+\wedge a\bigr)u
  +2(s\wedge b)v,
  &u\le v,\\
  2(s\wedge a)u
  +2\bigl((s-a)_+\wedge b\bigr)v,
  &v\le u,
 \end{cases}
 \label{eq:ridge-proof-bias-exponent}
\end{equation}
and
\begin{equation}
 \Gamma_{\cV_{\pt}}
 =
 \begin{cases}
  \gamma_n-\bigl[a+b+1-(\delta+1)_+\bigr]u
  +\bigl[2b-(\delta+1)_+\bigr]v,
  &u\le v,\\
  \gamma_n-(1-\delta)_+u+(\delta-1)_+v,
  &v\le u,
 \end{cases}
 \qquad
 \Gamma_{\cV_{\ft}}=1-v.
 \label{eq:ridge-proof-variance-exponent}
\end{equation}
Combining the above bounds yields
\begin{equation}
 \overline{\cR}
 =
 \widetilde\Theta\!\left(
 n^{-\min\{\Gamma_{\cB_1},
           \Gamma_{\cV_{\pt}},
           \Gamma_{\cV_{\ft}}\}}
 \right).
 \label{eq:ridge-proof-objective}
\end{equation}
For simplicity, define $ F(u,v) = \min\left\{
 \Gamma_{\cB_1},
 \Gamma_{\cV_{\pt}},
 \Gamma_{\cV_{\ft}}
 \right\}$. Thus, optimizing the polynomial decay exponent over the
two-stage ridge scalings reduces to maximizing $F(u,v)$
over $0<u\le\gamma_n$ and $0<v\le1.$ Next, we show that $\max_{u,v} F(u,v) = \Gamma_{\rm ridge}^*.$ To do so, we first show that $F(u,v) \leq \Gamma_{\rm ridge}^*$ for any $u,v,$ and then that $F(u_*,v_*) = \Gamma_{\rm ridge}^*,$ with $u_*, v_*$ in Corollary \ref{cor:minimax-optimal-ridge-transfer-full}.

\paragraph{Upper bound.} 
Recall that
\[
 F(u,v)\le 1-v,\qquad
 F(u,v)\le\Gamma_{\cV_{\pt}},\qquad
 F(u,v)\le\Gamma_{\cB_1}.
\]
Thus, $F(u,v)$ is bounded above by any nonnegative weighted
average of these three quantities.

If $u\le v$, take the weighted average with weights
\[
 \begin{gathered}
  2(s\wedge b)\bigl[a+b+1-(\delta+1)_+\bigr]
  +2\bigl((s-b)_+\wedge a\bigr)
       \bigl[2b-(\delta+1)_+\bigr],\\
  2\bigl((s-b)_+\wedge a\bigr),
  \qquad a+b+1-(\delta+1)_+,
 \end{gathered}
\]
respectively. These weights are nonnegative. Substituting from
\eqref{eq:ridge-proof-bias-exponent} and
\eqref{eq:ridge-proof-variance-exponent}, the coefficients
of both $u$ and $v$ cancel. Dividing by the sum of the
weights and simplifying gives
\begin{equation}
 F(u,v)\le
 1-
 \frac{
 a+b+1-(\delta+1)_+
 +2\bigl((s-b)_+\wedge a\bigr)(1-\gamma_n)
 }{
 \bigl[a+b+1-(\delta+1)_+\bigr]
 \bigl[1+2(s\wedge(a+b))\bigr]
 +2\delta\bigl((s-b)_+\wedge a\bigr)
 }
 =\Gamma_{\le}.
 \label{eq:ridge-proof-le-bound}
\end{equation}

If $v\le u$ and $\delta<1$, use instead the nonnegative
weights
\[
 2(1-\delta)\bigl((s-a)_+\wedge b\bigr),
 \qquad 2(s\wedge a),
 \qquad 1-\delta,
\]
respectively. Again, the coefficients of $u$ and $v$
cancel, yielding
\begin{equation}
 F(u,v)\le
 \frac{
 2(s\wedge a)\gamma_n
 +2(1-\delta)\bigl((s-a)_+\wedge b\bigr)
 }{
 2(s\wedge a)
 +(1-\delta)\bigl[1+2((s-a)_+\wedge b)\bigr]
 }
 =\Gamma_{\ge}.
 \label{eq:ridge-proof-ge-bound}
\end{equation}

Finally, if $v\le u$ and $\delta\ge1$, then
$\Gamma_{\cV_{\pt}}=\gamma_n+(\delta-1)v$.
The two variance constraints alone give
\[
 F(u,v)
 \le
 \frac{(\delta-1)(1-v)+\Gamma_{\cV_{\pt}}}{\delta}
 =
 \frac{\delta-1+\gamma_n}{\delta}
 =\Gamma_{\ge}.
\]
Therefore, every admissible two-stage choice satisfies
\[
 F(u,v)\le\max\{\Gamma_{\le},\Gamma_{\ge}\}
 \le\Gamma_{\rm ridge}^{\star}.
\]
Together with the previously established single-dataset
exponents, this proves that no admissible choice has a
polynomial decay exponent exceeding
\eqref{eq:global-optimal-ridge-exponent}.

\paragraph{Lower bound.}

We show that $u_*, v_*$ are admissible and in fact
\[
 F(u^\star,v^\star)=\Gamma_{\rm ridge}^{\star}.
\]
We first discuss properties of $\Gamma_{\geq}, \Gamma_{\leq}.$ As a function of $\gamma_n,$
$\Gamma_{\le}$ and $\Gamma_{\ge}$ are affine and intersect at
\[
 \gamma_n
 =
 1-\frac{\delta}{1+2(s\wedge(a+b))},
\]
where their common value is
$2(s\wedge(a+b))/(1+2(s\wedge(a+b)))$.
Moreover, $\Gamma_{\ge}$ has strictly larger slope as
\[
\frac{2(s\wedge a)}
 {(1-\delta)_+\bigl[1+2(s\wedge(a+b))\bigr]
   +2\delta(s\wedge a)}
 >
 \frac{2\bigl((s-b)_+\wedge a\bigr)}
 {\bigl[a+b+1-(\delta+1)_+\bigr]
   \bigl[1+2(s\wedge(a+b))\bigr]
   +2\delta\bigl((s-b)_+\wedge a\bigr)}
\]
where both denominators are positive, and we use
\[
 0\le ((s-b)_+\wedge a)\le s\wedge a,
 \qquad s\wedge a>0,
\]
together with
\[
 0<a+b+1-(\delta+1)_+-(1-\delta)_+
 =
 \begin{cases}
  2b, & \delta<-1,\\
  a+b-1, & -1\le\delta\le1,\\
  2a, & \delta>1,
 \end{cases}
\]
which follows from $\delta=b-a$ and $a,b>1$.
Consequently,
\begin{equation}
\begin{aligned}
 \Gamma_{\le}\ge\Gamma_{\ge}
 &\ \Longrightarrow\
 \Gamma_{\le}
 \le
 \frac{2(s\wedge(a+b))}{1+2(s\wedge(a+b))},
 \\
 \Gamma_{\ge}\ge\Gamma_{\le}
 &\ \Longrightarrow\
 \Gamma_{\ge}
 \ge
 \frac{2(s\wedge(a+b))}{1+2(s\wedge(a+b))}.
\end{aligned}
\label{eq:ridge-proof-candidate-crossing}
\end{equation}
Next, we show that by choosing
\[
 v^\star=1-\Gamma_{\rm ridge}^{\star},
 \qquad
 u^\star=
  \begin{cases}
   \displaystyle
   \frac{
   \Gamma_{\rm ridge}^{\star}-2b v^\star
   }{2((s-b)\wedge a)},
   &\Gamma_{\rm ridge}^{\star}=\Gamma_{\le},\\[3mm]
   \displaystyle
   \frac{
   \Gamma_{\rm ridge}^{\star}
   -2((s-a)_+\wedge b)v^\star
   }{2(s\wedge a)},
   &\Gamma_{\rm ridge}^{\star}=\Gamma_{\ge},
  \end{cases}
\] we have 
\[
 \Gamma_{\cB_1}(u^\star,v^\star)
 ,
 \Gamma_{\cV_{\pt}}(u^\star,v^\star)
 ,
 \Gamma_{\cV_{\ft}}(u^\star,v^\star)
 \geq
 \Gamma_{\rm ridge}^{\star}.
\]
First note that, by definition,   $ \Gamma_{\cV_{\ft}}(u^\star,v^\star) = 1- v^\star = \Gamma_{\rm ridge}^\star.$ We then show that
\[
 \Gamma_{\cV_{\pt}}(u^\star,v^\star)
 \ge \Gamma_{\rm ridge}^{\star},
 \qquad
 \Gamma_{\cB_1}(u^\star,v^\star)
 \ge \Gamma_{\rm ridge}^{\star},
\]
and that $u^\star, v^\star$ are admissible, namely
\[
 0<u^\star\le\gamma_n,
 \qquad
 0<v^\star\le1.
\]
Recall that $\Gamma_{\rm ridge}^{\star}$ strictly exceeds
both single-dataset exponents in the case under consideration.

\smallskip
\noindent
\emph{Case $\Gamma_{\rm ridge}^{\star}=\Gamma_{\ge}$.}
We first verify $(u^\star, v^\star)$ are admissible. Using \eqref{eq:ridge-candidate-ge} gives
\[
\begin{aligned}
 0
 &<
 \Gamma_{\rm ridge}^{\star}
 -\frac{2(s\wedge a)\gamma_n}
       {2(s\wedge a)+(1-\delta)_+}
 =
 \frac{
 2(1-\delta)_+((s-a)_+\wedge b)
 +2(s\wedge a)(\delta-1)_+
 }{2(s\wedge a)+(1-\delta)_+}\,v^\star.
\end{aligned}
\]
Since the coefficient of $v^\star$ is nonnegative,
this implies $v^\star>0$.
Also $v^\star<1$, since
$\Gamma_{\rm ridge}^{\star}>0$.
By \eqref{eq:ridge-proof-candidate-crossing},
\[
 u^\star-v^\star
 =
 \frac{
 [1+2(s\wedge(a+b))]\Gamma_{\rm ridge}^{\star}
 -2(s\wedge(a+b))
 }{2(s\wedge a)}
 \ge0.
\]
Thus $0<v^\star\le u^\star$.
The definitions and \eqref{eq:ridge-candidate-ge} also give
\[
\begin{aligned}
 \gamma_n-u^\star
 &=
 \bigl[2(s\wedge a)+(1-\delta)_+-1\bigr]
 (u^\star-v^\star)+
 \bigl[2(s\wedge(a+b))-\delta\bigr]v^\star
 >0.
\end{aligned}
\]
Here both coefficients are positive under
$a,b>1$, $s>0$, and $2s>\delta$.
In this case $(u^\star, v^\star)$ are admissible and satisfy
$v^\star\le u^\star$.

For $\cB_1$, we therefore have
\[
\begin{aligned}
 \Gamma_{\cB_1}(u^\star,v^\star)
 &=
 2(s\wedge a)u^\star
 +2((s-a)_+\wedge b)v^\star=
 \Gamma_{\rm ridge}^{\star}
 -2((s-a)_+\wedge b)v^\star
 +2((s-a)_+\wedge b)v^\star=
 \Gamma_{\rm ridge}^{\star}.
\end{aligned}
\]
For the source variance, using
\[
 s\wedge(a+b)
 =(s\wedge a)+((s-a)_+\wedge b),
 \qquad
 (1-\delta)_++\delta=1+(\delta-1)_+,
\]
we obtain
\[
\begin{aligned}
 \Gamma_{\cV_{\pt}}(u^\star,v^\star)
 &=
 \gamma_n-(1-\delta)_+u^\star
 +(\delta-1)_+v^\star\\
 &=
 \gamma_n-\frac{(1-\delta)_+}{2(s\wedge a)}+
 \left[
 \frac{
 (1-\delta)_+\bigl[1+2(s\wedge(a+b))\bigr]
 +2\delta(s\wedge a)}
 {2(s\wedge a)}
 -1
 \right]v^\star\\
 &=
 \gamma_n-\frac{(1-\delta)_+}{2(s\wedge a)}
 +\frac{(1-\delta)_++2(s\wedge a)(1-\gamma_n)}
       {2(s\wedge a)}
 -v^\star\\
 &=
 1-v^\star
 =\Gamma_{\rm ridge}^{\star}.
\end{aligned}
\]

\smallskip
\noindent
\emph{Case $\Gamma_{\rm ridge}^{\star}=\Gamma_{\le}$.}
In this case, we must have $s>b$;
otherwise $\Gamma_{\le}=2s/(1+2s)$ would equal the
target-only exponent, contrary to the strict comparison.
Thus the denominator in the definition of $u^\star$
is positive.
Moreover,
$\Gamma_{\rm ridge}^{\star}>2b/(1+2b)$ implies
\[
 u^\star
 =
 \frac{(1+2b)\Gamma_{\rm ridge}^{\star}-2b}
      {2((s-b)\wedge a)}
 >0.
\]
By \eqref{eq:ridge-proof-candidate-crossing},
\[
 u^\star-v^\star
 =
 \frac{
 [1+2(s\wedge(a+b))]\Gamma_{\rm ridge}^{\star}
 -2(s\wedge(a+b))
 }{2((s-b)\wedge a)}
 \le0.
\]
Hence $0<u^\star\le v^\star$.
Also $v^\star<1$, since
$\Gamma_{\rm ridge}^{\star}>0$.
The definitions and \eqref{eq:ridge-candidate-le} give
\[
\begin{aligned}
 \gamma_n-u^\star
 &=
 \bigl[a+b-(\delta+1)_+
       +2((s-b)\wedge a)\bigr]u^\star
 +(\delta+1)_+v^\star
 >0.
\end{aligned}
\]
Indeed, the coefficient of $u^\star$ is positive because
\[
 a+b-(\delta+1)_+
 =\min\{a+b,2a-1\}>0.
\]
Thus $(u^\star, v^\star)$ are admissible and satisfy
$u^\star\le v^\star$.

For the leading bias, using $s>b$, we obtain
\[
\begin{aligned}
 \Gamma_{\cB_1}(u^\star,v^\star)
 &=
 2((s-b)\wedge a)u^\star+2bv^\star=
 \Gamma_{\rm ridge}^{\star}-2bv^\star+2bv^\star
 =
 \Gamma_{\rm ridge}^{\star}.
\end{aligned}
\]
For the source variance, using
\[
 s\wedge(a+b)=b+((s-b)\wedge a),
 \qquad \delta=b-a,
\]
we have
\[
\begin{aligned}
 \Gamma_{\cV_{\pt}}(u^\star,v^\star)
 &=
 \gamma_n-\bigl[a+b+1-(\delta+1)_+\bigr]u^\star
 +\bigl[2b-(\delta+1)_+\bigr]v^\star\\
 &=
 \gamma_n
 -\frac{a+b+1-(\delta+1)_+}{2((s-b)\wedge a)}+
 \left[
 \frac{
 \bigl[a+b+1-(\delta+1)_+\bigr]
 \bigl[1+2(s\wedge(a+b))\bigr]
 +2\delta((s-b)\wedge a)}
 {2((s-b)\wedge a)}
 -1
 \right]v^\star\\
 &=
 \gamma_n
 -\frac{a+b+1-(\delta+1)_+}{2((s-b)\wedge a)}+
 \frac{
 a+b+1-(\delta+1)_+
 +2((s-b)\wedge a)(1-\gamma_n)}
 {2((s-b)\wedge a)}
 -v^\star\\
 &=
 1-v^\star
 =\Gamma_{\rm ridge}^{\star}.
\end{aligned}
\]
The last substitution uses
$v^\star=1-\Gamma_{\le}$ and
\eqref{eq:ridge-candidate-le}.

In both cases, the proposed cutoffs are admissible and
satisfy
\[
 \Gamma_{\cB_1}(u^\star,v^\star)
 =
 \Gamma_{\cV_{\pt}}(u^\star,v^\star)
 =
 \Gamma_{\cV_{\ft}}(u^\star,v^\star)
 =
 \Gamma_{\rm ridge}^{\star}.
\]
Consequently,
\[
 F(u^\star,v^\star)=\Gamma_{\rm ridge}^{\star},
\]
which establishes that picking $(u^\star, v^\star)$ achieves the optimal scaling exponent.
\end{proof}

\subsection{Proof of
\Cref{cor:minimax-optimal-ridge-transfer}}
Suppose $\delta>1$ and $\gamma_c<\gamma_n<1$.
Since
\[
 \gamma_n>\gamma_c
 =1-\frac{\delta}{1+2s}
 \ge
 1-\frac{\delta}{1+2(s\wedge(a+b))},
\]
we have $\Gamma_{\ge}>\Gamma_{\le}$.
Indeed, $\delta>1$ and $\delta=b-a$ imply
\[
 (1-\delta)_+=0,
 \qquad
 a+b+1-(\delta+1)_+=2a.
\]
Substituting these identities into
\eqref{eq:ridge-candidate-le} and
\eqref{eq:ridge-candidate-ge} gives
\[
\begin{aligned}
 \Gamma_{\ge}-\Gamma_{\le}
 &=
 \frac{
 a+((s-b)_+\wedge a)(1-\gamma_n)}
 {a[1+2(s\wedge(a+b))]
  +\delta((s-b)_+\wedge a)}
 -\frac{1-\gamma_n}{\delta}\\
 &=
 \frac{
 a\bigl\{\delta-[1+2(s\wedge(a+b))](1-\gamma_n)\bigr\}}
 {\delta\bigl\{
 a[1+2(s\wedge(a+b))]
 +\delta((s-b)_+\wedge a)
 \bigr\}}
 >0.
\end{aligned}
\]
Here the denominator is positive, and the numerator is
positive because
\[
 \delta-[1+2(s\wedge(a+b))](1-\gamma_n)
 =
 [1+2(s\wedge(a+b))]
 \left(
 \gamma_n-1+\frac{\delta}{1+2(s\wedge(a+b))}
 \right)
 >0.
\]
Moreover,
\[
 \Gamma_{\ge}
 =\frac{\delta-1+\gamma_n}{\delta}
 >
 \max\left\{\gamma_n,\frac{2s}{1+2s}\right\},
\]
so $\Gamma_{\ge}$ also strictly exceeds both
single-dataset exponents.
Thus \Cref{cor:minimax-optimal-ridge-transfer-full} yields
\[
 \Gamma_{\rm ridge}^{\star}
 =\frac{\delta-1+\gamma_n}{\delta},
 \qquad
 v^\star=\frac{1-\gamma_n}{\delta}.
\]

We finally argue that $u = \gamma_n$ is optimal. Note that the optimal $(u^\star, v^\star)$ by
\Cref{cor:minimax-optimal-ridge-transfer-full}
satisfies $v^\star\le u^\star\le\gamma_n$.
For fixed $v=v^\star$ and $v^\star\le u\le\gamma_n$,
the two variance exponents are
\[
\Gamma_{\cV_{\ft}}=1-v^\star,
\qquad
\Gamma_{\cV_{\pt}}=\gamma_n+(\delta-1)v^\star,
\]
and hence do not depend on $u$.
Meanwhile, the leading-bias exponent
\[
\Gamma_{\cB_1}
=
2(s\wedge a)u
+
2((s-a)_+\wedge b)v^\star
\]
is nondecreasing in $u$.
Since $\Gamma_{\cB_2}>\Gamma_{\cV_{\pt}}$, the second bias
term does not determine the overall exponent.
Consequently, replacing $u^\star$ in Corollary \ref{cor:minimax-optimal-ridge-transfer-full} by $\gamma_n$ preserves
the optimal exponent, as $\Gamma^*_{\rm ridge}$ is the minimal over the three exponents. 

Thus, the choice $u=\gamma_n$,
$v=\frac{1-\gamma_n}{\delta}$ is also optimal. By \eqref{eq:u-v-def}, this choice is realized by
\[
\lambda_{\pt}\asymp n_{\pt}^{1-a},
\qquad
\lambda_{\ft}\asymp n^{1-b(1-\gamma_n)/\delta},
\]
which proves the Corollary \ref{cor:minimax-optimal-ridge-transfer}.


\section{Proof of Lemma \ref{lem:sgd-scaling-main}}
\label{sec:gd-sgd-transfer}
To prove \Cref{lem:sgd-scaling-main}, we first 
condition on the pretrained source estimator and apply the
last-iterate SGD upper bound in \citep[Corollary 3.4]{wu2022power}. This reduces the upper bound to two
terms that depend on the source ridge estimator. We then compute the deterministic equivalent of these two terms, and evaluate the scaling law. Finally, we derive a matching lower bound.

Before going into the proof, we define some shorthand notations. Recall the power-law model in
\eqref{eq:power-law-spectra} and \eqref{eq:source-conditions}, with
\[
 \delta=b-a>1,\qquad
 \gamma_{\rm c}=1-\frac{\delta}{1+2s}<\gamma_n<1,
 \qquad 2s>b.
\]
Also by the condition that $2s>b$ we have
\begin{equation}
 \|\bbeta_*\|_2^2
 \asymp\sum_{k\ge1}k^{b-1-2s}<\infty.
 \label{eq:sgd-square-summable-signal}
\end{equation}
The signal follows a power law, with
$s=br_{\ft}=ar_{\pt}+\delta/2$, and both noise variances are  $O(1).$ 
For simplicity define
\begin{equation}
 m:=n_{\pt},\qquad 
 N:=\left\lfloor\frac{n}{\log_2 n}\right\rfloor,\qquad
 K:=\left(\frac Nm\right)^{1/\delta},\qquad
 v:=\frac{1-\gamma_n}{\delta}.
 \label{eq:sgd-proof-scales}
\end{equation}

We first derive the order of certain quantities that will be frequently used in the proof. A direct computation gives
\begin{equation}
 \frac KN=\frac1mK^{1-\delta}
 =\widetilde\Theta(n^{-\Gamma_{\rm TL}}),
 \qquad
 \Gamma_{\rm TL}:=\frac{\delta-1+\gamma_n}{\delta}.
 \label{eq:sgd-balanced-variances}
\end{equation}
Moreover, since $\delta>1$ and $2s>\delta$,
\[
 \gamma_{\rm c}-\frac1{1+\delta}
 =\frac{\delta(2s-\delta)}{(1+2s)(1+\delta)}>0.
\]
Hence $\gamma_n>\gamma_{\rm c}>1/(1+\delta)$, which implies
\[
 1-\gamma_n<\delta\gamma_n,
 \qquad
 v=\frac{1-\gamma_n}{\delta}<\gamma_n.
\]
Using $N\asymp n/\log n$ and $m\asymp n^{\gamma_n}$, we obtain
\[
 \frac Km
 =\frac{(N/m)^{1/\delta}}m
 \asymp
 n^{v-\gamma_n}(\log n)^{-1/\delta}
 \longrightarrow0.
\]

\subsection{Upper bound}
\label{app:sgd-rd-upper}

\paragraph{Step 1: Upper bound conditioned on source dataset.}

Conditional on the source estimator $\widehat \bbeta_\pt$, we write
$\be_0:=\widehat\bbeta_{\pt}-\bbeta_*,$ which is independent of the target dataset. Recall the geometric schedule
in \eqref{eq:sgd-optimal-tuning-main}--\eqref{eq:sgd-optimal-step-main}, use the shorthand $\xi_{z,k}:=(\bSigma_{z})_{k, k}$ ($z\in\{\pt, \ft\}$),
and define the following quantities
\begin{equation}
 \bP:=\prod_{t=1}^n(I-\eta_t\bSigma_{\ft}),\qquad
 p_k:=\prod_{t=1}^n(1-\eta_t\xi_{\ft,k}),\qquad
 \bC_P:=\bP\bSigma_{\ft}\bP,
 \label{eq:sgd-population-filter}
\end{equation}
as well as
\begin{equation}
 \begin{aligned}
 J&:=\{k:\xi_{\ft,k}\ge(N\eta_0)^{-1}=K^{-b}\},\\
 \bQ&:=K^{-b}I_J+(\bSigma_{\ft})_{J^c},\\
 D_{\rm eff}&:=|J|+K^{2b}\sum_{k\notin J}\xi_{\ft,k}^2,
 \end{aligned}
 \label{eq:sgd-effective-dimension}
\end{equation}
where $I_J$ and $(\bSigma_{\ft})_{J^c}$ are given by
\[
 [I_J]_{kk}:=\mathbf{1}_{\{k\in J\}},
 \qquad
 [(\bSigma_{\ft})_{J^c}]_{kk}
 :=\xi_{\ft,k}\mathbf{1}_{\{k\notin J\}},
 \qquad k\ge1.
\]

Next, we apply the last-iterate upper bound in
\cite[Corollary~3.4]{wu2022power}. To do so, we check that the relevant assumptions hold. \citep[Assumption 1A]{ wu2022power} and \citep[Assumption 2]{ wu2022power} trivially follow from  Assumption \ref{ass:sgd-fourth-moment-main}, and the way the target labels are generated in our model. It remains to verify $\eta_0 < \frac{1}{4 C \Tr(\bSigma_\ft)}.$ Note that
\[
 \gamma_{\rm c}-\frac ab
 =\frac{\delta(1+2s-b)}{b(1+2s)}>0,
\]
and hence
\[
 \eta_0=\frac{K^b}{N}
 =N^{a/\delta}m^{-b/\delta}
 \asymp n^{(a-b\gamma_n)/\delta}(\log n)^{-a/\delta}
 \longrightarrow0.
\]
Since $\Tr(\bSigma_{\ft})=O(1)$, there exists $N_0$ that only depends on $C, a, b, \gamma_n$ such that for large enough $n > N_0,$
$\eta_0<\min(1/2, [4C\Tr(\bSigma_{\ft})]^{-1})$, where $C$ is the constant
in Assumption \ref{ass:sgd-fourth-moment-main}.

Applying \cite[Corollary~3.4, upper bound]{wu2022power} to the initialization $\widehat \bbeta_\pt$ yields
\begin{equation}
 \begin{aligned}
 &\E\!\left[
 \|\widehat\bbeta_n-\bbeta_*\|_{\bSigma_{\ft}}^2
 \mid\mathcal D_{\pt}\right] \lesssim
 \|\bP\be_0\|_{\bSigma_{\ft}}^2
 +\frac{D_{\rm eff}}{N}
 \left(\sigma_{\ft}^2+\|\be_0\|_{\bQ}^2\right),
 \end{aligned}
 \label{eq:sgd-conditional-target-upper}
\end{equation} where we denote $\mathcal D_{\pt} = \{\bX_\pt, \by_\pt\}$, and $\lesssim$ absorbs universal constants.

The power-law spectrum gives
\[
\begin{aligned}
 D_{\rm eff}
 &=
 \sum_{k\ge1}\min\{1,K^{2b}\xi_{\ft,k}^2\}\asymp
 \sum_{k\le K}1+K^{2b}\sum_{k>K}k^{-2b}
 \asymp K+K^{2b}K^{1-2b}
 \asymp K.
\end{aligned}
\]
Moreover, for every $k$,
\[
 0\le[\bQ]_{kk}
 =\min\{K^{-b},\xi_{\ft,k}\}
 \le\xi_{\ft,k}.
\]
Since $0\le\eta_t\xi_{\ft,k}\le
\eta_0\|\bSigma_{\ft}\|_{\op}<1/2$, we also have
\[
 0\le[\bC_P]_{kk}
 =\xi_{\ft,k}p_k^2
 =\xi_{\ft,k}\prod_{t=1}^n(1-\eta_t\xi_{\ft,k})^2
 \le\xi_{\ft,k}.
\]
Thus,
\begin{equation}
 D_{\rm eff}\asymp K,
 \qquad 0\preceq\bQ\preceq\bSigma_{\ft},
 \qquad 0\preceq\bC_P\preceq\bSigma_{\ft}.
 \label{eq:sgd-target-weight-bounds}
\end{equation}

We now average over the source label noise while keeping the source
design fixed. Let
\[
 \bH_{\pt}:=(\bX_{\pt}^{\top}\bX_{\pt}+\lambda_{\pt}I)^{-1}.
\]
We compute
\[
\begin{aligned}
 \be_0
 &=
 \bH_{\pt}\bX_{\pt}^\top\by_{\pt}-\bbeta_*=
 (\bH_{\pt}\bX_{\pt}^\top\bX_{\pt}-\bI)\bbeta_*
 +\bH_{\pt}\bX_{\pt}^\top\boldsymbol{\varepsilon}_{\pt}=
 -\lambda_{\pt}\bH_{\pt}\bbeta_*
 +\bH_{\pt}\bX_{\pt}^\top\boldsymbol{\varepsilon}_{\pt}.
\end{aligned}
\]
Averaging over the source noise gives
\begin{equation}
\begin{aligned}
 \bS_0
 &:=
 \E_{\eps_\pt}[\be_0\be_0^\top\mid\bX_{\pt}]\\
 &=
 \lambda_{\pt}^2\bH_{\pt}\bbeta_*\bbeta_*^\top\bH_{\pt}
 +\bH_{\pt}\bX_{\pt}^\top
 \E[\boldsymbol{\varepsilon}_{\pt}
       \boldsymbol{\varepsilon}_{\pt}^\top\mid\bX_{\pt}]
 \bX_{\pt}\bH_{\pt}\\
 &=
 \lambda_{\pt}^2\bH_{\pt}\bbeta_*\bbeta_*^\top\bH_{\pt}
 +\sigma_{\pt}^2\bH_{\pt}\bX_{\pt}^\top
 \bX_{\pt}\bH_{\pt}.
\end{aligned}
\label{eq:sgd-source-second-moment}
\end{equation}
Thus, 
\begin{equation}
\begin{aligned}
 \E_{\eps_\pt}\!\left[
 \|\widehat\bbeta_n-\bbeta_*\|_{\bSigma_{\ft}}^2
 \mid\bX_{\pt}\right]
 &\lesssim
 \E_{\eps_\pt}[\|\bP\be_0\|_{\bSigma_{\ft}}^2\mid\bX_{\pt}]
 +\frac{D_{\rm eff}}{N}
 \left(
 \sigma_{\ft}^2+
 \E_{\eps_\pt}[\|\be_0\|_{\bQ}^2\mid\bX_{\pt}]
 \right)\\
 &=
 \Tr(\bC_P\bS_0)
 +\frac{D_{\rm eff}}{N}
 \left(\sigma_{\ft}^2+\Tr(\bQ\bS_0)\right)\\
 &\lesssim
 \Tr(\bC_P\bS_0)
 +\frac KN
 \left(\sigma_{\ft}^2+\Tr(\bSigma_{\ft}\bS_0)\right).
\end{aligned}
\label{eq:sgd-conditional-design-upper}
\end{equation}
The last inequality uses $D_{\rm eff}\asymp K$ and
\[
 \Tr(\bSigma_{\ft}\bS_0)-\Tr(\bQ\bS_0)
 =
 \Tr\!\left(
 \bS_0^{1/2}(\bSigma_{\ft}-\bQ)\bS_0^{1/2}
 \right)\ge0,
\]
since $\bS_0\succeq0$ and $\bQ\preceq\bSigma_{\ft}$.
The remaining task for the upper bound is therefore to compute the two terms:
$\Tr(\bC_P\bS_0)$ and $\Tr(\bSigma_{\ft}\bS_0)$.

\paragraph{Step 2: Deterministic equivalents for the source estimator.}
Choose $\lambda_{\pt}=\kappa m^{1-a}$ and let
$\tau:=\lamonmu_{\pt}$ be its effective regularization.
By \Cref{prop:generic-effective-ridge,prop:uniform-denominator-scaling},
\begin{equation}
 \tau\asymp m^{-a},\qquad
 c_{\pt}:=
 \left\{m-\Tr\!\left[
 \bSigma_{\pt}^2(\bSigma_{\pt}+\tau I)^{-2}
 \right]\right\}^{-1}\asymp m^{-1}.
 \label{eq:sgd-source-cutoff-fixed}
\end{equation}
For a deterministic diagonal positive semidefinite $\bC$,
define
\begin{align}
 \bW_{\pt}
 &:=\bSigma_{\pt}(\bSigma_{\pt}+\tau I)^{-2},\nonumber\\
 A(\bC)
 &:=\tau^2\bbeta_*^\top
 (\bSigma_{\pt}+\tau I)^{-1}\bC
 (\bSigma_{\pt}+\tau I)^{-1}\bbeta_*,\nonumber\\
 L(\bC)
 &:=A(\bC)+c_{\pt}
 \left(\sigma_{\pt}^2+\tau^2\bbeta_*^\top\bW_{\pt}\bbeta_*\right)
 \Tr(\bC\bW_{\pt}).
 \label{eq:sgd-weighted-source-equivalent}
\end{align}
We further define the following quantities that will appear in the approximation error
\begin{equation}
 \chi_m:=\tau^2\bbeta_*^\top
 (\bSigma_{\pt}+\tau I)^{-1}\bbeta_*,\qquad
 \varepsilon_m:=
 \frac{\rho_{\lambda_{\pt}}(m)^6\log^{5/2}m}{\sqrt m}.
 \label{eq:sgd-source-error-quantities}
\end{equation}

Next we apply \cite[Theorem~6, equation~(40), and Theorem~5, equation~(37)]{wu2026improved} to obtain the deterministic equivalence. We verify assumptions in \cite[Theorem~6, and Theorem~5]{wu2026improved} at the chosen source regularization. By definition in \eqref{eq:rho_lambda},
\[
 \rho_{\lambda_{\pt}}(m)
 =1+O\!\left(\frac{m^{1-a}\log m}{\lambda_{\pt}}\right)
 =O(\log m).
\]
To see this, take $j=\lfloor\vartheta m\rfloor$, with the fixed
$\vartheta\in(0,1/2)$ appearing in the definition of $\rho$, and use
\[
 \sum_{k\ge j}\xi_{\pt,k}\asymp m^{1-a},\qquad
 \frac{\sum_{k\ge j}\xi_{\pt,k}}{\xi_{\pt,j}}\asymp m.
\]
Consequently,
\begin{equation}
 \varepsilon_m=O(m^{-1/2}\log^{17/2}m)=o(1),
 \qquad
 \chi_m\le\tau\|\bbeta_*\|_2^2=O(m^{-a}).
 \label{eq:sgd-source-epsilon}
\end{equation}
Thus, for large enough $m$, we have that $\lambda_{\pt}\rho_{\lambda_{\pt}}(m)\gtrsim m^{1-a}$
dominates $\|\bSigma_{\pt}\|_{\op}m^{-K_0}$ for fixed $K_0>a-1$,
and $\rho_{\lambda_{\pt}}(m)^{5/2}\log^{3/2}m=o(\sqrt m),$ which implies that the assumptions holds.

Applying 
\cite[Theorem~6, equation~(40), and Theorem~5, equation~(37)]{wu2026improved} gives
\begin{equation}
 \left|\Tr(\bC\bS_0)-L(\bC)\right|
 \le C_D\varepsilon_m\left[
 L(\bC)+\chi_m
 \|\bC(\bSigma_{\pt}+\tau I)^{-1}\|_{\op}\right],
 \label{eq:sgd-weighted-source-error}
\end{equation}
with probability at least $1-C_Dm^{-D}$ over the source design,
for every fixed $D>0$. 

\paragraph{Step 3: Evaluation of the deterministic equivalents and approximation errors.}

We first control $\Tr(\bC_P\bS_0)$ through its deterministic
equivalent
\[
 L(\bC_P)
 =
 A(\bC_P)
 +c_{\pt}
 \left(
 \sigma_{\pt}^2+\tau^2\bbeta_*^\top\bW_{\pt}\bbeta_*
 \right)
 \Tr(\bC_P\bW_{\pt}).
\]
Since $c_{\pt}\asymp m^{-1}$ and
$K/N=m^{-1}K^{1-\delta}$, we will show that
\[
\begin{aligned}
 &\Tr(\bC_P\bW_{\pt})
 \asymp K^{1-\delta}, \quad
 A(\bC_P)
 =o(K/N), \quad 
\tau^2\bbeta_*^\top\bW_{\pt}\bbeta_*
 =O(m^{-a}), \\
 &\|\bC_P(\bSigma_{\pt}+\tau I)^{-1}\|_{\op}
 \lesssim K^{-\delta},\quad \chi_m \lesssim m^{-a}.
\end{aligned}
\]
For the term $\Tr(\bC_P \bW_\pt)$,
we first estimate the order of  $p_k$.
The geometric schedule satisfies
\[
 K^b=N\eta_0
 \le\sum_{t=1}^n\eta_t
 \le2N\eta_0=2K^b.
\]
Since $0\le\eta_t\xi_{\ft,k}<1/2$, the inequalities
$-2x\le\log(1-x)\le-x$ give
\[
 -2\xi_{\ft,k}\sum_{t=1}^n\eta_t
 \le\log p_k
 \le-\xi_{\ft,k}\sum_{t=1}^n\eta_t.
\]
Using $\xi_{\ft,k}\asymp k^{-b}$, we obtain
\begin{equation}
 e^{-C_1(K/k)^b}
 \le p_k
 \le e^{-c_1(K/k)^b}
 \label{eq:sgd-fixed-filter}
\end{equation}
for fixed positive constants $c_1,C_1$. Combining this estimate with
$\xi_{\pt,k}\asymp k^{-a}$ and $\tau\asymp m^{-a}$,
we have
\[
\begin{aligned}
 \Tr(\bC_P\bW_{\pt})
 &=
 \sum_{k\ge1}
 \frac{p_k^2\xi_{\ft,k}\xi_{\pt,k}}
      {(\xi_{\pt,k}+\tau)^2}\\
 &\lesssim
 \sum_{k\le m}e^{-c(K/k)^b}k^{-\delta}
 +m^{2a}\sum_{k>m}k^{-a-b}\\
 &\lesssim
 K^{1-\delta}+m^{1-\delta}
 \lesssim K^{1-\delta}.
\end{aligned}
\]
Indeed, boundedness of $x^\delta e^{-cx^b}$ and $\delta>1$
give
\[
 \sum_{k\le m}e^{-c(K/k)^b}k^{-\delta}
 \lesssim
 K^{-\delta}\sum_{k\le K}1
 +\sum_{k>K}k^{-\delta}
 \lesssim K^{1-\delta}.
\]
Also,
$m^{2a}\sum_{k>m}k^{-a-b}\asymp m^{1-\delta}$,
and
$m^{1-\delta}=o(K^{1-\delta})$ because $K=o(m)$
and $\delta>1$.

For the matching lower bound, restrict to
$\lceil K\rceil\le k\le\lfloor2K\rfloor$.
On this range, $p_k\gtrsim1$ and
$\tau/\xi_{\pt,k}\asymp(K/m)^a=o(1)$.
Consequently,
\[
 \Tr(\bC_P\bW_{\pt})
 \gtrsim
 \sum_{k=\lceil K\rceil}^{\lfloor2K\rfloor}
 \frac{\xi_{\ft,k}}{\xi_{\pt,k}}
 \asymp
 \sum_{k=\lceil K\rceil}^{\lfloor2K\rfloor}k^{-\delta}
 \asymp K^{1-\delta}.
\]
For the term $\|\bC_P(\bSigma_{\pt}+\tau I)^{-1}\|_{\op}$, a direct computation gives
\[
\begin{aligned}
 \|\bC_P(\bSigma_{\pt}+\tau I)^{-1}\|_{\op}
 =
 \sup_k
 \frac{p_k^2\xi_{\ft,k}}{\xi_{\pt,k}+\tau}\le
 \sup_k\frac{p_k^2\xi_{\ft,k}}{\xi_{\pt,k}}\lesssim
 \sup_k e^{-c(K/k)^b}k^{-\delta}=
 K^{-\delta}
 \sup_k (K/k)^\delta e^{-c(K/k)^b}
 \lesssim K^{-\delta}.
\end{aligned}
\]
Together with the trace estimate, this gives
\begin{equation}
 \frac1m\Tr(\bC_P\bW_{\pt})
 \asymp\frac1mK^{1-\delta},
 \qquad
 \|\bC_P(\bSigma_{\pt}+\tau I)^{-1}\|_{\op}
 \lesssim K^{-\delta}.
 \label{eq:sgd-source-spectral-bounds}
\end{equation}
For the remaining terms that depend on $\bbeta_*$,
by \eqref{eq:sgd-square-summable-signal},
\[
 \chi_m
 =
 \tau^2\bbeta_*^\top
 (\bSigma_{\pt}+\tau I)^{-1}\bbeta_*
 \le\tau\|\bbeta_*\|_2^2
 =O(m^{-a}).
\]
For any diagonal positive semidefinite $\bC$,
\begin{equation}
\begin{aligned}
 A(\bC)
 &=
 \tau^2\sum_{k\ge1}
 \frac{\bbeta_*[k]^2[\bC]_{kk}}
      {(\xi_{\pt,k}+\tau)^2}
 \le
 \left(
 \tau^2\sum_{k\ge1}
 \frac{\bbeta_*[k]^2}{\xi_{\pt,k}+\tau}
 \right)
 \sup_k\frac{[\bC]_{kk}}{\xi_{\pt,k}+\tau}
 =
 \chi_m
 \|\bC(\bSigma_{\pt}+\tau I)^{-1}\|_{\op}.
\end{aligned}
\label{eq:sgd-weighted-bias-bound}
\end{equation}
Applying this inequality to $\bC_P$ yields
\[
 A(\bC_P)\lesssim m^{-a}K^{-\delta},
 \qquad
 \frac{A(\bC_P)}{K/N}
 \lesssim
 \frac{m^{-a}K^{-\delta}}
      {m^{-1}K^{1-\delta}}
 =
 \frac{m^{1-a}}K
 \longrightarrow0.
\]
Finally, the inequality
$(x+\tau)^2\ge4x\tau$ gives
\[
\begin{aligned}
 0\le
 \tau^2\bbeta_*^\top\bW_{\pt}\bbeta_*
 &=
 \sum_{k\ge1}
 \frac{\tau^2\xi_{\pt,k}}
      {(\xi_{\pt,k}+\tau)^2}\bbeta_*[k]^2
 \le
 \frac{\tau}{4}\|\bbeta_*\|_2^2
 =O(m^{-a}).
\end{aligned}
\]
Since $\sigma_{\pt}^2$ is fixed and strictly positive,
\[
 \sigma_{\pt}^2+
 \tau^2\bbeta_*^\top\bW_{\pt}\bbeta_*
 =
 \sigma_{\pt}^2+O(m^{-a})
 \asymp1.
\]

\smallskip
\noindent
Substituting the above estimates into $L(\bC_P)$ gives
\[
\begin{aligned}
 L(\bC_P)
 &=
 A(\bC_P)
 +c_{\pt}
 \left(
 \sigma_{\pt}^2+\tau^2\bbeta_*^\top\bW_{\pt}\bbeta_*
 \right)
 \Tr(\bC_P\bW_{\pt})=
 o(K/N)+\Theta(m^{-1}K^{1-\delta})
 \asymp K/N.
\end{aligned}
\]
For the approximation error, we  have
\begin{equation}
\begin{aligned}
 \frac{|\Tr(\bC_P\bS_0)-L(\bC_P)|}{K/N}
 &\lesssim
 \varepsilon_m
 \left[
 \frac{L(\bC_P)}{K/N}
 +
 \frac{\chi_m K^{-\delta}}
      {m^{-1}K^{1-\delta}}
 \right]\lesssim
 \varepsilon_m
 \left(1+\frac{m\chi_m}{K}\right)\lesssim
 \varepsilon_m
 \left(1+\frac{m^{1-a}}K\right)
 =o(1).
\end{aligned}
\label{eq:sgd-source-accuracy}
\end{equation}
Here we used $\varepsilon_m=o(1)$ and $m^{1-a}/K\to0$.
Since $L(\bC_P)\asymp K/N$, it follows that
\begin{equation}
 \Tr(\bC_P\bS_0)
 =
 L(\bC_P)(1+o(1))
 \asymp\frac KN.
 \label{eq:sgd-propagated-source-bound}
\end{equation}

Next, we control $\Tr(\bSigma_{\ft}\bS_0)$.
Recall that
\[
 L(\bSigma_{\ft})
 =
 A(\bSigma_{\ft})
 +c_{\pt}
 \left(
 \sigma_{\pt}^2+\tau^2\bbeta_*^\top\bW_{\pt}\bbeta_*
 \right)
 \Tr(\bSigma_{\ft}\bW_{\pt}).
\]
We already know that $c_{\pt}\asymp m^{-1}$,
$\tau^2\bbeta_*^\top\bW_{\pt}\bbeta_*=O(m^{-a})$,
and $\chi_m=O(m^{-a})$.
It therefore suffices to establish
\[
 \Tr(\bSigma_{\ft}\bW_{\pt})=O(1),
 \qquad
 \|\bSigma_{\ft}(\bSigma_{\pt}+\tau I)^{-1}\|_{\op}=O(1),
 \qquad
 A(\bSigma_{\ft})=O(1),
\]
and then apply \eqref{eq:sgd-weighted-source-error}.

A direct computation gives
\[
\begin{aligned}
 \Tr(\bSigma_{\ft}\bW_{\pt})
 =
 \sum_{k\ge1}
 \frac{\xi_{\ft,k}\xi_{\pt,k}}
      {(\xi_{\pt,k}+\tau)^2}\le
 \sum_{k\ge1}\frac{\xi_{\ft,k}}{\xi_{\pt,k}}
\lesssim
 \sum_{k\ge1}k^{-\delta}
 =O(1),
\end{aligned}
\]
and
\[
\begin{aligned}
 A(\bSigma_{\ft})
 =
 \sum_{k\ge1}
 \left(\frac{\tau}{\xi_{\pt,k}+\tau}\right)^2
 \xi_{\ft,k}\bbeta_*[k]^2\le
 \sum_{k\ge1}\xi_{\ft,k}\bbeta_*[k]^2=\|\bbeta_*\|_{\bSigma_{\ft}}^2\le\|\bSigma_{\ft}\|_{\op}\|\bbeta_*\|_2^2
 =O(1).
\end{aligned}
\]

Substituting these bounds into the deterministic equivalent,
we obtain
\[
\begin{aligned}
 L(\bSigma_{\ft})
 &=
 A(\bSigma_{\ft})
 +c_{\pt}
 \left(
 \sigma_{\pt}^2+\tau^2\bbeta_*^\top\bW_{\pt}\bbeta_*
 \right)
 \Tr(\bSigma_{\ft}\bW_{\pt})=
 A(\bSigma_{\ft})+O(m^{-1})
 =O(1).
\end{aligned}
\]
Moreover,
\[
 \chi_m
 \|\bSigma_{\ft}(\bSigma_{\pt}+\tau I)^{-1}\|_{\op}
 =O(m^{-a}).
\]
Hence, on the same event,
\eqref{eq:sgd-weighted-source-error} yields
\[
\begin{aligned}
 &\left|
 \Tr(\bSigma_{\ft}\bS_0)-L(\bSigma_{\ft})
 \right|\le
 C_D\varepsilon_m
 \left[
 L(\bSigma_{\ft})
 +\chi_m
 \|\bSigma_{\ft}(\bSigma_{\pt}+\tau I)^{-1}\|_{\op}
 \right]=O(\varepsilon_m)=o(1).
\end{aligned}
\]
Consequently,
\begin{equation}
 \Tr(\bSigma_{\ft}\bS_0)
 =
 L(\bSigma_{\ft})+O(\varepsilon_m)
 =O(1).
 \label{eq:sgd-initial-energy-bound}
\end{equation}

\paragraph{Step 4: Compute the expectation.}
For every source design $\bX_\pt$,
\[
 \|\lambda_{\pt}\bH_{\pt}\|_{\op}\le1,\qquad
 \bH_{\pt}\bX_{\pt}^{\top}\bX_{\pt}\bH_{\pt}
 \preceq\frac1{4\lambda_{\pt}}I.
\]
Therefore, for either $\bC=\bC_P$ or $\bC=\bSigma_{\ft}$,
\[
 \begin{aligned}
 0\le\Tr(\bC\bS_0)
 \le\|\bSigma_{\ft}\|_{\op}\|\bbeta_*\|_2^2
 +\frac{\sigma_{\pt}^2\Tr(\bSigma_{\ft})}{4\lambda_{\pt}}=O(1+m^{a-1}).
 \end{aligned}
\]
Choose the 
exponent
$D>a-1+\Gamma_{\rm TL}/\gamma_n$ in
\eqref{eq:sgd-weighted-source-error}, and recall that $\frac{K}{N} = \tilde \bTheta(n^{-\Gamma_{\rm TL}})$. Let $\mathcal E$ denote the common source-design event on which
\eqref{eq:sgd-propagated-source-bound} and
\eqref{eq:sgd-initial-energy-bound} hold, so that
$\P(\mathcal E^c)\le C_Dm^{-D}$.
For either $\bC=\bC_P$ or $\bC=\bSigma_{\ft}$,
the uniform bound above gives
\[
\begin{aligned}
 \E_{\bX_{\pt}}\!\left[
 \Tr(\bC\bS_0)\mathbf{1}_{\mathcal E^c}
 \right]
 \lesssim
 (1+m^{a-1})\P(\mathcal E^c)\lesssim
 m^{a-1-D}
 =o(K/N).
\end{aligned}
\]
Indeed, by the choice $D>a-1+\Gamma_{\rm TL}/\gamma_n$,
\[
 \frac{m^{a-1-D}}{K/N}
 =
 \widetilde O\!\left(
 n^{\gamma_n(a-1-D)+\Gamma_{\rm TL}}
 \right)
 \longrightarrow0.
\]

We now split the expectation over $\mathcal E$ and $\mathcal E^c$.
The upper bound on $\mathcal E$ yields
\[
\begin{aligned}
 \E_{\bX_{\pt}}\Tr(\bC_P\bS_0)
 &=
 \E_{\bX_{\pt}}\!\left[
 \Tr(\bC_P\bS_0)\mathbf{1}_{\mathcal E}
 \right]
 +
 \E_{\bX_{\pt}}\!\left[
 \Tr(\bC_P\bS_0)\mathbf{1}_{\mathcal E^c}
 \right]\lesssim
 \frac KN+o(K/N)
 \lesssim\frac KN.
\end{aligned}
\]
Conversely, nonnegativity and the lower bound on $\mathcal E$ give
\[
\begin{aligned}
 \E_{\bX_{\pt}}\Tr(\bC_P\bS_0)
 \ge
 \E_{\bX_{\pt}}\!\left[
 \Tr(\bC_P\bS_0)\mathbf{1}_{\mathcal E}
 \right]\gtrsim
 \frac KN\,\P(\mathcal E)
 \asymp\frac KN.
\end{aligned}
\]
Thus
$\E_{\bX_{\pt}}\Tr(\bC_P\bS_0)\asymp K/N$.
Similarly,
\[
\begin{aligned}
 \E_{\bX_{\pt}}\Tr(\bSigma_{\ft}\bS_0)
 &=
 \E_{\bX_{\pt}}\!\left[
 \Tr(\bSigma_{\ft}\bS_0)\mathbf{1}_{\mathcal E}
 \right]
 +
 \E_{\bX_{\pt}}\!\left[
 \Tr(\bSigma_{\ft}\bS_0)\mathbf{1}_{\mathcal E^c}
 \right]\le O(1)+o(K/N)
 =O(1).
\end{aligned}
\]

Finally, by the definition
$\bS_0=\E[\be_0\be_0^\top\mid\bX_{\pt}]$,
\begin{equation}
\begin{aligned}
 \E\|\bP\be_0\|_{\bSigma_{\ft}}^2
 &=
 \E_{\bX_{\pt}}\!
 \left[\E[\be_0^\top\bC_P\be_0\mid\bX_{\pt}]\right]=
 \E_{\bX_{\pt}}\Tr(\bC_P\bS_0)
 \asymp K/N,\\
 \E\|\be_0\|_{\bSigma_{\ft}}^2
 &=
 \E_{\bX_{\pt}}\!
 \left[\E[\be_0^\top\bSigma_{\ft}\be_0\mid\bX_{\pt}]\right]=
 \E_{\bX_{\pt}}\Tr(\bSigma_{\ft}\bS_0)
 =O(1).
\end{aligned}
\label{eq:sgd-unconditional-source-bounds}
\end{equation}

Taking the expectation of
\eqref{eq:sgd-conditional-design-upper} over the source design
therefore yields
\begin{equation}
\begin{aligned}
 \E\|\widehat\bbeta_n-\bbeta_*\|_{\bSigma_{\ft}}^2
 &=
 \E_{\bX_{\pt}}\!\left[
 \E\!\left[
 \|\widehat\bbeta_n-\bbeta_*\|_{\bSigma_{\ft}}^2
 \mid\bX_{\pt}
 \right]\right]\\
 &\lesssim
 \E_{\bX_{\pt}}\Tr(\bC_P\bS_0)
 +\frac KN
 \left(
 \sigma_{\ft}^2+
 \E_{\bX_{\pt}}\Tr(\bSigma_{\ft}\bS_0)
 \right)\\
 &\lesssim
 \frac KN+\frac KN\bigl(\sigma_{\ft}^2+O(1)\bigr)
 \lesssim\frac KN.
\end{aligned}
\label{eq:sgd-final-upper-bound}
\end{equation}
This completes the proof of  the upper bound.

\subsection{Lower bound}
\label{app:sgd-rd-lower}

We derive a matching lower bound from the target label noise alone.
Let us condition on the source sample $\mathcal D_{\pt}$,
so that the initialization error
$\be_0=\widehat\bbeta_{\pt}-\bbeta_*$ is fixed.
Writing $\be_t:=\widehat\bbeta_t-\bbeta_*$ and using
$y_t=\bx_t^\top\bbeta_*+\epsilon_t$, the SGD recursion gives
\[
 \be_t
 =
 (\bI-\eta_t\bx_t\bx_t^\top)\be_{t-1}
 +\eta_t\epsilon_t\bx_t.
\]
We decompose this error as
\[
 \be_t=\boldsymbol r_t+\bz_t,
\]
where $\boldsymbol r_t$ is the initialization error that propagates during training, which satisfies
\[
 \boldsymbol r_0=\be_0,\qquad
 \boldsymbol r_t
 =
 (\bI-\eta_t\bx_t\bx_t^\top)\boldsymbol r_{t-1},
\]
and $\bz_t$ is the error due to the noisy labels,
\[
 \bz_0=0,\qquad
 \bz_t
 =
 (\bI-\eta_t\bx_t\bx_t^\top)\bz_{t-1}
 +\eta_t\epsilon_t\bx_t.
\]
The identity $\be_t=\boldsymbol r_t+\bz_t$ follows by
induction from these recursions.

Conditional on $\mathcal D_{\pt}$ and all target covariates
$\bx_1,\ldots,\bx_n$, the vector
\[
 \boldsymbol r_n
 =
 (\bI-\eta_n\bx_n\bx_n^\top)\cdots
 (\bI-\eta_1\bx_1\bx_1^\top)\be_0
\]
is deterministic. Furthermore, expanding the noise recursion gives
\[
 \bz_n
 =
 \sum_{t=1}^n\eta_t\epsilon_t
 (\bI-\eta_n\bx_n\bx_n^\top)\cdots
 (\bI-\eta_{t+1}\bx_{t+1}\bx_{t+1}^\top)\bx_t,
\]
where the matrix product is interpreted as $I$ when $t=n$.
Thus, $\bz_n$ is a linear combination of the target noises
with coefficients fixed under this conditioning.

Averaging over the conditioned variables  yields the lower bounds
\[
 \E\|\be_n\|_{\bSigma_{\ft}}^2
 =
 \E\|\boldsymbol r_n\|_{\bSigma_{\ft}}^2
 +
 \E\|\bz_n\|_{\bSigma_{\ft}}^2
 \ge
 \E\|\bz_n\|_{\bSigma_{\ft}}^2,
\]
where we use the fact that the target noises are centered and independent of the
source sample and all target covariates, so that the cross term is zero.

It remains to control $\E\|\bz_n\|_{\bSigma_{\ft}}^2.$ Defining $\bV_t:=\E[\bz_t\bz_t^\top]$, we have $\E\|\widehat\bbeta_n-\bbeta_*\|_{\bSigma_{\ft}}^2
 \ge\Tr(\bSigma_{\ft}\bV_n),$ where
$\bV_t$ satisfies the following recursion
\[
 \bV_t
 =\E\big[(I-\eta_t\bx_t\bx_t^\top)\bV_{t-1}
                (I-\eta_t\bx_t\bx_t^\top)\big]
 +\eta_t^2\sigma_{\ft}^2\bSigma_{\ft}.
\]
For every positive semidefinite $\bA$, we compute
\[
\begin{aligned}
 &\E\big[(I-\eta_t\bx_t\bx_t^\top)\bA (I-\eta_t\bx_t\bx_t^\top)\big]-(I-\eta_t\bSigma_{\ft})\bA(I-\eta_t\bSigma_{\ft})]=\eta_t^2\E\big[
  (\bx_t\bx_t^\top-\bSigma_{\ft})\bA
  (\bx_t\bx_t^\top-\bSigma_{\ft})\big]\succeq0.
\end{aligned}
\]
Iterating this positive-semidefinite comparison and taking the target
trace yields
\begin{equation}
 \Tr(\bSigma_{\ft}\bV_n)
 \ge\sigma_{\ft}^2\sum_k\sum_{t=1}^n
 \eta_t^2\xi_{\ft,k}^2
 \prod_{j=t+1}^n(1-\eta_j\xi_{\ft,k})^2.
 \label{eq:sgd-noise-population-lower}
\end{equation}
We now restrict \eqref{eq:sgd-noise-population-lower} to
$1\le t\le N$ and $\lceil K\rceil\le k\le\lfloor2K\rfloor$.
On this range,
\[
 \eta_t=\eta_0=\frac{K^b}{N},\qquad
 \xi_{\ft,k}\asymp K^{-b},\qquad
 \sum_{j=1}^n\eta_j\le2K^b.
\]
For sufficiently large $n$, we have
$0\le\eta_j\xi_{\ft,k}\le1/2$.
Using $\log(1-x)\ge-2x$ on this interval, we obtain uniformly
over these $t,k$
\[
 \prod_{j=t+1}^n(1-\eta_j\xi_{\ft,k})^2
 \ge\exp\!\left(-4\xi_{\ft,k}\sum_{j=t+1}^n\eta_j\right)
 \ge c>0.
\]
Computing the summation gives 
\begin{equation}
 \begin{aligned}
 \E\|\widehat\bbeta_n-\bbeta_*\|_{\bSigma_{\ft}}^2
 &\ge c\sigma_{\ft}^2
 \sum_{k=\lceil K\rceil}^{\lfloor2K\rfloor}
 \sum_{t=1}^N\eta_0^2\xi_{\ft,k}^2 \gtrsim K N\eta_0^2K^{-2b}
 =\frac KN.
 \end{aligned}
 \label{eq:sgd-fixed-schedule-lower}
\end{equation}

Combining with the upper bound gives
\[
 \E\|\widehat\bbeta_n-\bbeta_*\|_{\bSigma_{\ft}}^2
 \asymp\frac KN
 =\widetilde\Theta(n^{-\Gamma_{\rm TL}}),
\]
which proves \Cref{lem:sgd-scaling-main}.

\section{Additional numerical results and experimental details}
\label{app:num}

\subsection{Deterministic equivalent predictions for empirical spectra}
\label{app:de-empirical}

We test the accuracy of the deterministic equivalent  predictions when we replace the Gaussian inputs with PCA features from MNIST \citep{LeCun1998GradientbasedLA} and CIFAR-10 \citep{Krizhevsky09learningml}, while keeping the outputs generated by the same linear model as in the theory. 
\Cref{fig:mnist-pca-de,fig:cifar10-pca-de} compare predictions from  the deterministic equivalent with the empirical values of the transfer learning risk as the feature dimension, regularization, and sample size vary. We also include a Gaussian control with the same covariance spectrum. The Gaussian control follows the deterministic equivalent, while the image features show slightly larger deviations, especially under weaker regularization. In both datasets, the deterministic prediction captures the main shape of the risk curve, including the peak near $p=n$ and its shift with $n$.

\begin{figure*}[t]
    \centering
    \includegraphics[width=\textwidth]{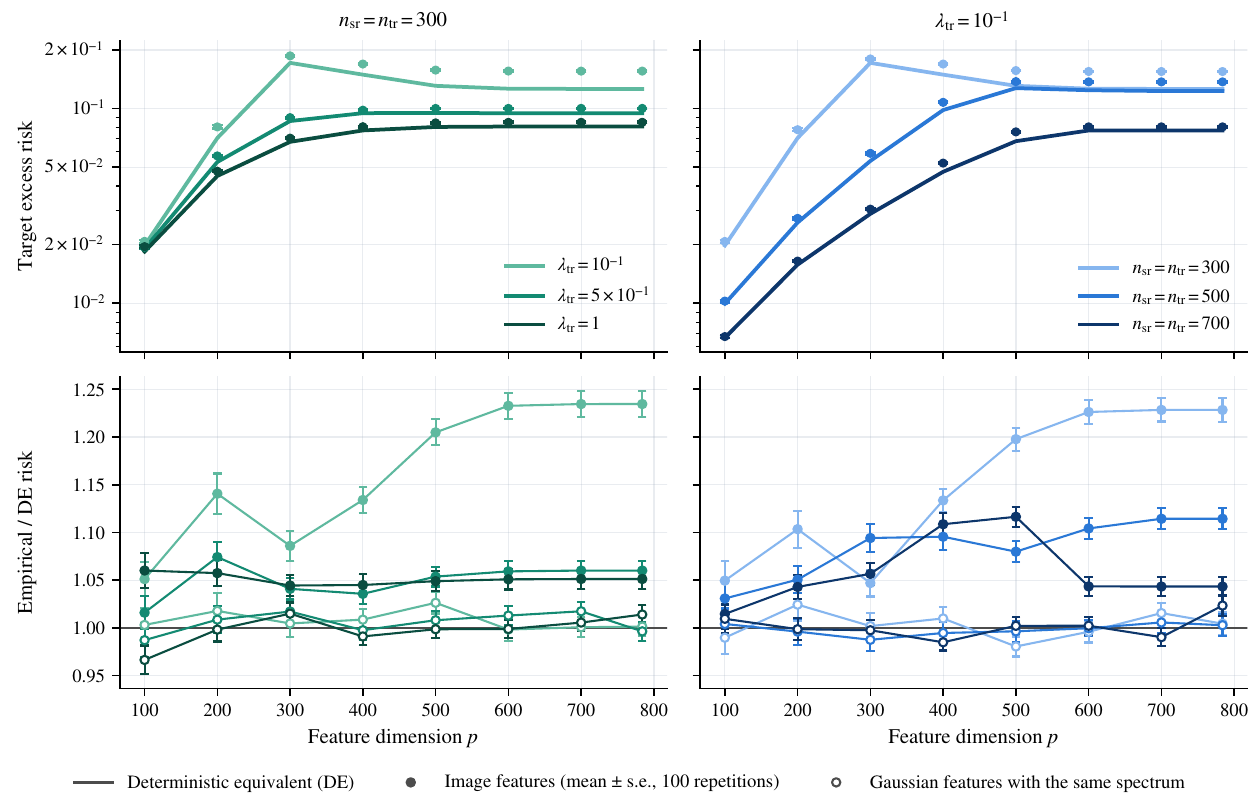}



    \caption{
    Deterministic equivalent on MNIST PCA features. Inputs are the top $p$ principal components of the centered training images, and outputs follow a linear model with noise standard deviation $0.2$ and normalized correlation between $\bbeta_{\pt}$ and $\bbeta_{\ft}$ equal to $0.8$. Source and target inputs are drawn with replacement from the training images. We use $\lambda_{\rm sr}=0.5\lambda_{\rm tr}$. The left column fixes $n_{\rm sr}=n_{\rm tr}=300$ and varies $\lambda_{\rm tr}\in\{10^{-1},5\times10^{-1},1\}$; the right column fixes $\lambda_{\rm tr}=10^{-1}$ and varies $n_{\rm sr}=n_{\rm tr}\in\{300,500,700\}$. \textbf{Top:} solid curves show the deterministic equivalent and markers show the mean target excess risk over $100$ training draws with $\pm1$ s.e.\ bars. The risk increases up to $p=n$ and then flattens. For $n=300$ and $\lambda_{\rm tr}=10^{-1}$, there is a peak at $p=n$. Discrepancies are larger at weaker regularization. \textbf{Bottom:} ratios of empirical risk to the deterministic equivalent for image features and Gaussian features with the same covariance spectrum. The Gaussian control stays within $3.4\%$ of the prediction, while the image features deviate more.}
    \label{fig:mnist-pca-de}
\end{figure*}

\begin{figure*}[t]
    \centering
    \includegraphics[width=\textwidth]{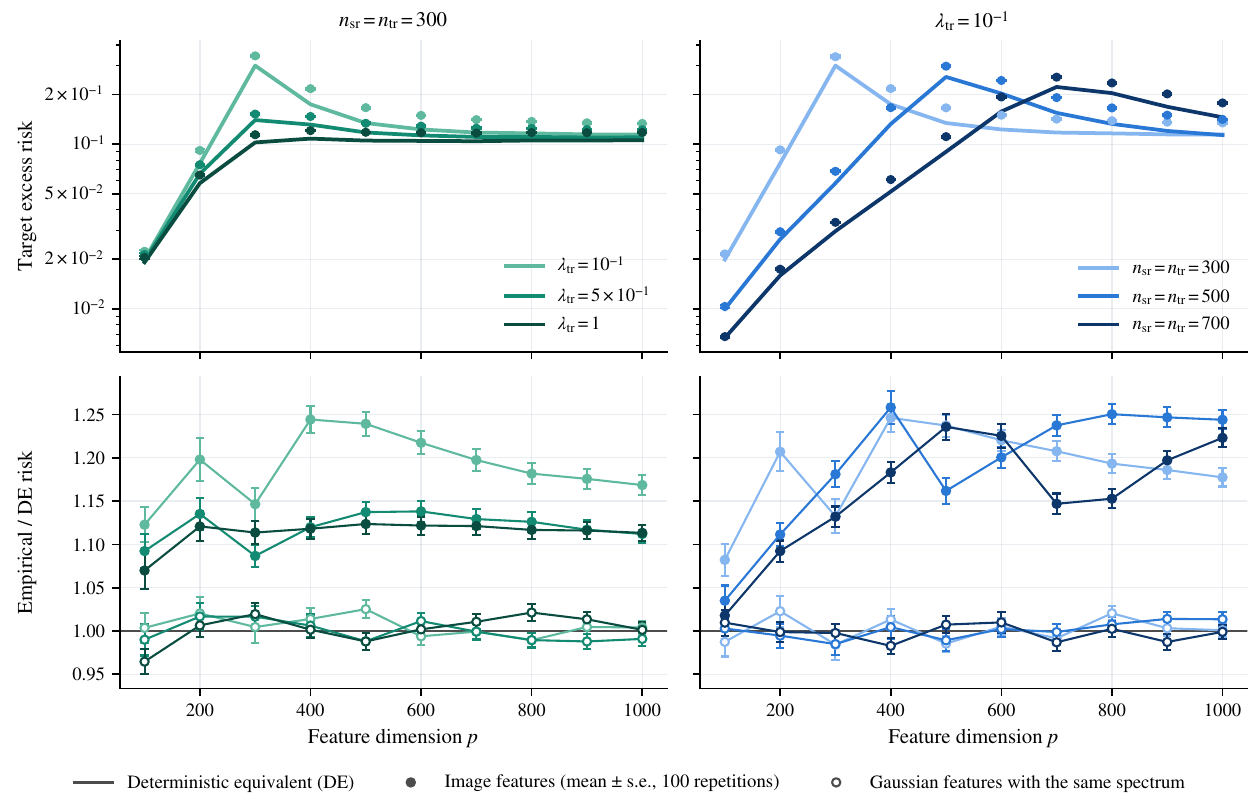}



    \caption{Deterministic equivalent on CIFAR-10 PCA features. We use the same setup as in \Cref{fig:mnist-pca-de}, using the same image pool for PCA and for drawing $20{,}000$ training samples. Source and target inputs are drawn with replacement from this pool. The left column fixes $n_{\rm sr}=n_{\rm tr}=300$ and varies $\lambda_{\rm tr}\in\{10^{-1},5\times10^{-1},1\}$; the right column fixes $\lambda_{\rm tr}=10^{-1}$ and varies $n_{\rm sr}=n_{\rm tr}\in\{300,500,700\}$. \textbf{Top:} solid curves show the deterministic equivalent and markers show the mean target excess risk over $100$ training draws with $\pm1$ s.e.\ bars. The curves show a peak near $p=n$, which shifts with $n$, with larger errors under weaker regularization. \textbf{Bottom:} ratios of empirical risk to the deterministic equivalent for image features and Gaussian features with the same covariance spectrum. The Gaussian control stays within $3.5\%$ of the prediction, while the image features deviate more.}
    \label{fig:cifar10-pca-de}
\end{figure*}




\subsection{Details for experiments on language models}
\label{app:llm-exp}

\paragraph{Model and data.}
We use GPT-2 small \citep{radford2019language} in the nanoGPT implementation \citep{karpathy2022nanogpt}, with 12 layers, 12 attention heads, width 768, context length 1024, and 123.6M parameters. FineWeb-Edu \citep{penedo2024fineweb} is the source distribution, and Python \citep{codeparrot_clean} is the target distribution. We tokenize all data with the GPT-2 tokenizer. We ensure that the target validation and test sets do not overlap with the source training pool; for Python, we additionally split the train and held-out data by repository and remove near-duplicate training files.

\paragraph{Setup.}
We refer to the number of unique target and source tokens as the
\emph{target budget} and \emph{source budget}, respectively. The target
budget is $n=2^k$ for $k\in\{21,\ldots,26\}$. Within each seed, the
budgets are nested, so different methods see the same target examples in
the same order.
The source budget follows
\begin{equation}
  m_\gamma(n)
  =
  m^\ast
  \left(\frac{n}{n^\ast}\right)^\gamma,
  \qquad
  n^\ast=2^{21},
  \qquad
  m^\ast=7\cdot2^{21},
\end{equation}
with $\gamma\in\{0,0.75,1,1.25\}$. Thus, $\gamma=0$ corresponds to a fixed source budget, while $\gamma>0$ makes the source budget grow with the target budget. 

\paragraph{Training.} We train all models with AdamW \citep{loshchilov2018decoupled}, using $\beta=(0.9,0.95)$, gradient clipping at $1$, and batches of 65,536 tokens. Each source model makes one pass over FineWeb-Edu. The main source learning rate, $6\times10^{-4}$, is selected on seed $0$ from $\{3,6,12\}\times10^{-4}$ using source validation loss at $m^\ast$.

We compare three procedures for training with target data. \emph{Target only} (TO) trains from random initialization. \emph{Fine-tuning} (FT) starts from the corresponding source checkpoint with a fresh optimizer state. L2-SP \citep{li2018explicit} replaces ordinary weight decay with a pull toward the source parameters. In our implementation (FTSP), we apply the decoupled update
\begin{equation}
  \theta
  \leftarrow
  \theta
  -
  \min(\eta_t\lambda,1)
  (\theta-\theta_{\pt}),
\end{equation}
after every AdamW step, so it is decoupled from the gradient in the same way AdamW's weight decay is. The original L2-SP adds a penalty to the loss instead. Also, here $\eta_t$ is the scheduled learning rate, which can be during the warm-up phase, constant phase, or cool-down phase.

\paragraph{Target-training length.} We train for $E\in\{1,2,4,\dots\}$ epochs using cool-down branches from the same run. Hyperparameters are tuned with at most 16 epochs. We choose the number of epochs based on validation, and we increase $E$ until doubling the number of epochs improves validation loss by less than $0.01$, and then we report the branch with the lowest validation loss. The $E = 1$ branch of the same run gives the one-pass result. 

\paragraph{Hyperparameters and seeds.} Learning rate, regularization strength, and dropout are tuned on seed $0$, separately for each target and training method, at $n\in\{2^{22},2^{26}\}$. The experiments use seeds $0,4,5$. Each seed fixes the initialization, dropout noise, and data order, allowing paired comparisons across methods.

\paragraph{Evaluation.}
The test loss is measured as the average cross-entropy loss per token. Since the irreducible loss floor is unknown, we summarize scaling directly from the observed test loss. Between neighboring target budgets $n_i$ and $n_{i+1}$, we define the local exponent
\begin{equation}
\label{eq:gamma-hat-llm}
  \widehat{\Gamma}_{i,i+1}
  :=
  -\frac{
    \log L_{i+1}-\log L_i
  }{
    \log n_{i+1}-\log n_i
  }.
\end{equation}
For windows containing more than two target budgets $n_i$, we fit a single exponent by least squares.
We compute ratios and differences within each seed and then average them, reporting standard errors across seeds. 

\end{document}